\documentclass[11pt]{article}

\usepackage{xcolor}
\usepackage{float}
\usepackage{graphicx}
\usepackage{amsmath}
\usepackage{longtable}
\usepackage[top=1in, bottom=1in,left=1in, right=1in]{geometry}
\usepackage{caption}
\usepackage{booktabs}
\usepackage{float}
\usepackage{subfig}
\usepackage{lscape}
\usepackage{dsfont}
\usepackage{amssymb}
\usepackage[most]{tcolorbox}
\usepackage[colorlinks=true, allcolors=blue]{hyperref}
\usepackage{algpseudocode}
\usepackage{algorithm}
\usepackage{titlesec}
\titleformat{\paragraph}
{\normalfont\normalsize\bfseries}{\theparagraph}{1em}{}
\titlespacing*{\paragraph}
{0pt}{3.25ex plus 1ex minus .2ex}{1.5ex plus .2ex}
\usepackage{amsmath,amsthm,amssymb, enumitem, mathtools, mathrsfs,bbm}
\usepackage{todonotes}

\theoremstyle{plain}
\newtheorem{theorem}{Theorem}[section]
\newtheorem*{theorem*}{Theorem}
\newtheorem{proposition}[theorem]{Proposition}
\newtheorem{lemma}[theorem]{Lemma}
\newtheorem{corollary}[theorem]{Corollary}

\newtheorem{assumption}{Assumption}

\theoremstyle{definition}
\newtheorem{definition}[theorem]{Definition}
\newtheorem{example}[theorem]{Example}

\theoremstyle{remark}
\newtheorem{remark}[theorem]{Remark}

\numberwithin{equation}{section}

\newcommand{\R}{\mathbb{R}}
\newcommand{\Z}{\mathbb{Z}}

\renewcommand{\L}{\mathscr{L}}
\newcommand{\M}{\mathcal{M}}

\renewcommand{\d}{\mathrm{d}}

\DeclareMathOperator{\diag}{diag}

\newcommand{\lb}{\left}
\newcommand{\rb}{\right}

\newcommand{\pp}{\partial}

\newcommand{\wt}{\widetilde}

\newcommand{\veps}{\varepsilon}
\newcommand{\bP}{\mathbb{P}}

\newcommand{\bL}{\mathbb{L}}
\newcommand{\E}{\mathbb{E}}

\newcommand{\Cov}{\operatorname{Cov}}
\newcommand{\Var}{\operatorname{Var}}
\newcommand{\calN}{\mathcal{N}}

\usepackage[giveninits=true,backend=biber, maxnames=3,maxalphanames=3,style=authoryear,doi=false,isbn=false,url=false,eprint=true,natbib]{biblatex}
\renewbibmacro{in:}{}
\AtEveryBibitem{%
  \clearlist{language}%
}

\title{Data-Driven Dynamic Factor Modeling via Manifold Learning\footnote{
We thank Max Reppen, Alex Belloni, Igor Cialenco,  Nizar Touzi, and Renyuan Xu, for helpful discussions. We also acknowledge participants of the 2025 SIAM-FM biennial meeting, the 2026 JMM, the 2025 AMaMeF conference, Duke University, and the NYU-Columbia Financial Mathematics series.}}
\author{Graeme Baker\thanks{Department of Statistics, Columbia University, NY, USA \href{mailto:g.baker@columbia.edu}{g.baker@columbia.edu}.}\and 
    Agostino Capponi\thanks{Department of Industrial Engineering and Operations Research and Columbia Business School, Columbia University, NY, USA \href{mailto:ac3827@columbia.edu}{ac3827@columbia.edu}.}\and 
     J. Antonio Sidaoui\thanks{Department of Industrial Engineering and Operations Research, Columbia University, NY, USA \href{mailto:j.sidaoui@columbia.edu}{j.sidaoui@columbia.edu}.}}
\date{}

\begin{document}

\maketitle

\begin{abstract}
We introduce a data-driven dynamic factor framework for modeling the joint evolution of high-dimensional covariates and responses without parametric assumptions. Standard factor models applied to covariates alone often lose explanatory power for responses. Our approach uses anisotropic diffusion maps, a manifold learning technique, to learn low-dimensional embeddings that preserve both the intrinsic geometry of the covariates and the predictive relationship with responses. For time series arising from Langevin diffusions in Euclidean space, we show that the associated graph Laplacian converges to the generator of the underlying diffusion. We further establish a bound on the approximation error between the diffusion map coordinates and linear diffusion processes, and we show that ergodic averages in the embedding space converge under standard spectral assumptions. These results justify using Kalman filtering in diffusion-map coordinates for predicting joint covariate-response evolution. We apply this methodology to equity-portfolio stress testing using macroeconomic and financial variables from Federal Reserve supervisory scenarios, achieving mean absolute error improvements of up to 55\% over classical scenario analysis and 39\% over principal component analysis  benchmarks.
\end{abstract}

\section{Introduction}

The emerging field of data-driven dynamical systems aims to learn the potentially high-dimensional and nonlinear dynamical description of a data set without parametric assumptions. Some of the primary goals in this area include finding stochastic differential equation (SDE) descriptions of discretely sampled time series data, extracting interpretable stochastic models from experimental data, and finding simple linear descriptions for high-dimensional and nonlinear dynamical systems. 
Herein, we propose a data-driven approach to the problem of \emph{dynamic factor modeling} in the \emph{supervised learning} setting where a response variable $y(t) \in \mathbb{R}^n$ depends on a high-dimensional set of covariates $x(t) \in \mathbb{R}^m$. A key challenge in this setting is that dimensionality reduction applied to covariates alone, such as principal component analysis, may discard information critical for predicting responses, even when the intrinsic geometry of $x(t)$ is preserved. Our framework addresses this by learning embeddings jointly from $(x(t), y(t))$ that capture both the dynamics of the covariates and their relationship to the responses.

We treat the observation vector $z(t)={(x(t),y(t)) \in \mathbb{R}^{m+n}}$ as the state of a dynamical system at time $t$. We assume the covariates and response variables are nonlinear functions of a true, possibly lower-dimensional, underlying state of the system $\theta_t \in \mathbb{R}^d$, which is assumed to satisfy a Langevin diffusion. These nonlinear functions restrict the trajectories of the observations to lie on some (unknown) manifold $\mathcal{M}$.  

To fix the mathematical setting, consider an $\R^d$-valued stochastic process $\theta=(\theta_t)_{t\ge 0}$ and an $\R^{m+n}$-valued process $Z=(Z_t)_{t\ge 0}$ satisfying the equations
\begin{align}
\d \theta_t&=-\nabla U(\theta_t)\d t+ \sqrt{2} \d W_t,\\
Z_t&=H(\theta_t),
\end{align}
where $W$ is a $d$-dimensional standard Brownian motion. The potential $U:\R^d\to \R$ is assumed to grow fast enough so that $\theta$ admits the invariant measure $\mu(\d w)\propto e^{-U(w)}\d w$, and the infinitesimal generator of $\theta$, given by $\L:=\Delta-\nabla U\cdot \nabla$, possesses a discrete spectrum. (We restate our assumptions formally in Section \ref{sec:assumps} below.) The function ${H:\R^d\to\R^{m+n}}$ is assumed to be $C^2$-smooth and $\bL^2(\mu)$-integrable. We view these equations as a degenerate
\emph{hidden Markov model} (HMM): $Z$ is an observed process which depends (deterministically) on the evolution of the latent state $\theta$, which in turn evolves stochastically according to a Langevin diffusion. Given $N+1$ discrete observations of a single $Z$-trajectory, $\{Z(t_i)\}_{i=0}^N$, and without any knowledge of $H,U,$ or even the latent dimension $d$, we wish to \emph{learn} the dynamics of $Z$. That is, we wish to build an accurate data-driven approximation to carry out tasks such as prediction, conditional sampling, interpolation, and generation of new sample trajectories. 

In broad terms, we adopt a spectral approach, whereby $H$ is approximated using $\ell\ge1$ principal eigenfunctions $(\varphi_1,\dots, \varphi_\ell)$ of the infinitesimal generator $\L$. As $H\in \bL^2(\mu)$, we have that
\begin{align}
Z_t=H(\theta_t)\approx \mathbf{H} (\varphi_1(\theta_t),\dots,\varphi_\ell(\theta_t))^\top,
\end{align}
for some matrix $\mathbf{H}\in\R^{(m+n)\times\ell}$. If $\varphi_i$ is associated with eigenvalue $-\lambda_i<0$, then It\^{o}'s formula yields
\begin{align}
\d \varphi_i(\theta_t)=-\lambda_i \varphi_i(\theta_t) \d t+\sqrt{2} \nabla \varphi_i(\theta_t)\cdot \d W_t.
\end{align}
Approximating $(\varphi_1(\theta_t),\dots,\varphi_\ell(\theta_t))^\top$ by a linear diffusion $\xi=(\xi_t)_{t\ge 0}$ in $\R^\ell$, we arrive at a \emph{linear} HMM:
\begin{align}
\d \xi_t&=-\Lambda \xi_t\d t+\Gamma \d B_t,\\
Z_t&=\mathbf{H} \xi_t+\eta_t,
\end{align}
where $\Lambda=\diag(\lambda_1,\dots,\lambda_\ell)$, $\Gamma\in \R^{\ell\times\ell}$, and $B$ is a standard $\ell$-dimensional Brownian motion. Here, $\eta=(\eta_t)_{t\ge 0}$ is a noise term incorporating the approximation error between $\xi$ and the eigenfunctions, as well as the truncation error from the expansion of $H$.

The learning problem in this setting amounts to estimating the eigenvalues which comprise $\Lambda$ and the projection coefficients $\mathbf{H}$. 
Assuming a Gaussian model for $(\eta_t)_{t\ge 0}$, we arrive at a linear-Gaussian state-space model for $Z$. Using standard Kalman filtering techniques, one may fit $\Gamma$ and carry out prediction, conditional sampling, and so on. We view this non-parametric approach, introduced as the \emph{diffusion Kalman filter} (DKF) and applied to object tracking in \citet{shnitzer2020diffusion}, as an alternative to the \emph{extended Kalman filter} (see, for instance, \citet{Kailath99}). In Subsection \ref{sec:jdkf_model}, we generalize this methodology to incorporate both the covariates $x(t)$ and response variables $y(t)$, yielding the \emph{joint diffusion Kalman filter} (JDKF). 

As a special case, assuming a linear relationship between $x(t)$ and $y(t)$ yields a data-driven approach to linear multi-factor modeling, which can be compared and interpreted alongside standard linear models such as principal component analysis. By contrast, standard dynamic factor models typically uncover latent factors from a single set of measurements $x(t)$. In supervised settings where $y(t)$ is a noisy function of $x(t)$, factors learned from $x(t)$ alone may fail to retain explanatory power for $y(t)$. 
Motivated by this limitation, our approach learns diffusion-coordinate embeddings jointly from $(x(t),y(t))$, and models their evolution through a linear state-space approximation. This ensures that the learned factors retain explanatory power for $y(t)$ while still capturing the intrinsic dynamics of $x(t)$. The key insight is to treat both covariates and responses as noisy observations of a common latent state $\theta$ evolving according to unknown dynamics, and to learn this latent evolution in a data-driven way via diffusion maps. This yields a supervised dynamic factor framework that supports prediction and conditional sampling, and is formalized in Section~\ref{sec:jdkf_model}.

Our aims in this work are threefold: (1) we show that $\Lambda$ and $\mathbf{H}$ can be learned using \emph{anisotropic diffusion maps} and provide a quantitative error bound, (2) we justify the reduction to a linear HMM by bounding the approximation error between the eigenfunctions and linear diffusions, and (3) we demonstrate the effectiveness of our approach for dynamic factor modeling with an application to conditional sampling in the context of the stress testing of equity portfolios. 
The mathematical contributions of our work are as follows:
applying concentration inequalities for Markov processes, in Theorem \ref{thm:concentration} we generalize a result on the convergence of graph Laplacians from \citet{singer_graph_2006} to the time series case, showing that the eigenfunctions of $\L$ can be learned. 
With Proposition \ref{prop:clt}, we provide rigorous justification for reconstructing points in the original data space via the lifting operator by establishing its convergence to an $\mathbb{L}^2$-projection with respect to the invariant measure of $\theta$. The rationale for utilizing diffusion maps and defining a state space framework completely in terms of these embeddings is the observation, first mentioned in \citet{coifman2008diffusion}, that the dynamics of the diffusion coordinate embeddings are \textit{approximately linear}. Under standard spectral assumptions on the generator of $\theta$, we show in Subsection \ref{sec:linearapprox} that the approximation error using linear dynamics remains small when averaging over the data. 

\subsection{Manifold Learning with Diffusion Maps}

We carry out the learning of $\Lambda$ and $\mathbf{H}$ using the \emph{diffusion maps} algorithm, which we now outline following \citet{scholkopf_convergence_2007, von_luxburg_consistency_2008}. Consider first the case where $H=I$ so that $Z=\theta$, and we wish to determine the eigenvalues of the infinitesimal generator $\L$ from samples $\{\theta(t_i)\}_{i=0}^N$. For $\veps>0$, let $P_\veps=\exp(\veps\L)$ denote the Markov transition operator associated with $\theta$, and let $p_\veps$ denote the associated transition kernel with respect to the invariant measure $\mu$. As noticed in \citet{scholkopf_convergence_2007}, $P_\veps$ and $\L$ share eigenfunctions, and the eigenvalues of $P_\veps$ can be mapped to those of $-\L$ by $\lambda\mapsto -\frac{1}{\veps}\log(\lambda)$. One may approximate $P_\veps$ by the random operator $P_\veps^N$ which acts on continuous bounded functions $f:\R^d\to\R$ by 
\begin{align}
P_\veps^Nf(w):=\frac{1}{N+1}\sum_{i=0}^N p_\veps(w,\theta(t_i))f(\theta(t_i)).
\end{align}
This can be rewritten as 
\begin{align}
P_\veps^Nf(w)=\int p_\veps(w,y)f(y) \mu_N(\d y)
\end{align}
where $\mu_N=\frac{1}{N+1}\sum_{i=0}^N \delta_{\theta(t_i)}$ denotes the empirical measure defined by the data. As $N\to\infty$, one hopes that $\mu_N$ approximates the invariant measure $\mu$, that $P_\veps^N\to P_\veps$, and the eigenvalues/functions/spaces of $P_\veps^N$ converge to those of $P_\veps$.

The construction of $P_\veps^N$ requires \emph{a priori} knowledge of the transition kernel $p_\veps$. To overcome this obstruction, we approximate $p_\veps$ using a Gaussian kernel
\begin{align}\label{eq:kkernel}
k_\veps(w,y)=\exp(-\|w-y\|^2/4\veps). 
\end{align}
This is justified since the short-time asymptotic for the heat kernel of a Langevin diffusion satisfies
\begin{align}
\lim_{t\downarrow 0}4\veps\log(p_\veps(w,y))=-\|w-y\|^2.
\end{align}
(this follows from applying a Girsanov transformation; and see also \citet{hsu2002stochastic}, for background).
Consider the operator $T_\veps^N$ which acts on bounded continuous $f$ by
\begin{align}
T_\veps^Nf(w):=\frac{\sum_{i=0}^N k_\veps(w,\theta(t_i))f(\theta(t_i))}{\sum_{i=0}^N k_\veps(w,\theta(t_i))}.
\end{align}
Setting 
\begin{align}
d_\veps^N(w):=\sum_{i=0}^N k_\veps(w,\theta(t_i))=\int k_\veps(w,y) \mu_N(\d y),
\end{align}
we may rewrite $T_\veps^Nf(w)$ as
\begin{align}
T_\veps^Nf(w)=\frac{1}{d_\veps^N(w)}\int k_\veps(w,y)f(y)\mu_N(\d y).
\end{align}
Supposing that $T_\veps^N$ converges to some limiting $T_\veps$, and $T_\veps$ is close to $P_\veps$ for small $\veps$, then the eigenvalues/functions/spaces of $P_\veps$ can be approximated by those of $T_\veps^N$ for large $N$. Furthermore, the spectrum and eigenfunctions of $T_\veps^N$ can be computed explicitly as they are in a one-to-one correspondence with the matrix $\mathbf{P}$ which has entries
\begin{equation*}
\mathbf{P}_{ij}=\frac{k_\veps(\theta(t_i),\theta(t_j))}{d_\veps^N(\theta(t_i))},\quad i,j=1,\dots,N,
\end{equation*}
see \citet[Proposition 7]{von_luxburg_consistency_2008}.

When $H\neq I$, and we observe $\{Z(t_i)\}_{i=0}^N$ rather than the latent process $\{\theta(t_i)\}_{i=0}^N$, we must approximate the squared Euclidean distance in the kernel \eqref{eq:kkernel} and construct an analog of $\mathbf{P}$ using only samples of $Z$; this leads to \emph{anisotropic diffusion maps} (ADM), introduced in \citet{singer2008nonlinear}. To illustrate, let $w,\wt w\in \R^d$ and set $z=H(w)$ and $\wt z=H(\wt w)$. Denoting the Jacobian matrix for $H$ by $J_H$, a Taylor approximation gives
\begin{align}\label{eq:mahalanobis}
\|w-\wt w\|^2&=\frac{1}{2}(z-\wt z)^T\lb((J_HJ_H^T)^{-1}(z)+(J_HJ_H^T)^{-1}(\wt z)\rb)(z-\wt z)+O(\|z-\wt z\|^4)
\end{align}
Although $H$ is not observed, the Jacobian $J_H$ can be deduced from the quadratic covariation process of $Z$. To wit, applying It\^{o}'s formula to a component $Z^i=H^i(\theta)$ gives
\begin{equation*}
\d Z_t^i= \L H_i(\theta_t) \d t+\sqrt{2}\nabla H_i(\theta_t) \cdot \d W_t
\end{equation*}
and hence
\begin{equation*}
\d \langle Z^i,Z^j\rangle_t=2\nabla H_i(\theta_t) \cdot \nabla H_j(\theta_t)\d t,\quad i,j=1,\dots, m+n.
\end{equation*}
Therefore, $\|\theta(t_i)-\theta(t_j)\|^2$ can be approximated by the \emph{squared modified Mahalanobis distance} between samples of $Z$:
\begin{align*}
d(Z(t_i),Z(t_j))^2:=\frac{1}{2}(Z(t_i)-Z(t_i))^T(C^{-1}(t_i)+C^{-1}(t_j))(Z(t_i)-Z(t_j)),
\end{align*}
where $C$ is an estimator for the instantaneous quadratic covariation process of $Z$.

\subsection{Applications}

We demonstrate the applicability of our framework in problems from mathematical finance, where stochastic dynamical systems play a central role. Classical continuous and discrete-time models, including interest rate, credit risk, volatility, and return models, are typically specified parametrically and tailored to low-dimensional settings. Our data-driven approach is designed to operate under substantially weaker modeling assumptions and to scale naturally to high-dimensional environments.

% We demonstrate the power of our framework for applications, focusing on problems from mathematical finance, where dynamical systems are ubiquitous. Continuous time processes, typically represented by SDEs, include interest rate models such as the Vasicek process, credit risk models such as Cox--Ingersoll--Ross, and derivatives pricing models such as Black--Scholes--Merton. Discrete time models are also widely present in financial applications; these include models for stochastic volatility such as ARCH and GARCH, and models for returns such as the standard autoregressive process. Strong parametric and modeling assumptions are often applied when fitting real-world data with stochastic processes. 
%Under comparably weak assumptions, our approach is applicable to a wide variety of problems. One of 
%A third key application, with which 
We validate our framework in the context of financial stress testing and scenario analysis through a detailed empirical investigation. These settings involve predicting the evolution of an asset or portfolio conditional on shocks to the risk factors that drive its value (see, for instance, \citet{haugh2020scenario}).
Within the context of stress testing of equity portfolios, 
we demonstrate the superiority of our framework over several comparison benchmarks including Principal Component Analysis (PCA) via a series of historical rolling backtests. Our historical backtests span three major financial crisis periods and we demonstrate reductions in mean absolute error (MAE) of up to 55\% and superior accuracy up to 77.78\% of the time when utilizing our framework over the benchmarks. 

The rest of the paper is organized as follows. Section \ref{sec:lit} relates our work to the literature. Section \ref{sec:model} introduces some preliminary concepts and tools necessary for our framework, including the dynamical system setting and the diffusion map construction method. Section \ref{sec:results} presents key theoretical results, including convergence analysis of the graph Laplacian and the robustness of approximating the diffusion coordinates by linear diffusions. Section \ref{sec:jdkf_model} introduces our dynamic factor modeling framework and conditional sampling procedure. Section \ref{sec:validation} describes the application of our method to financial stress testing: in \ref{sec:stress_benchmarks} we describe the comparison benchmarks and \ref{sec:backtest} defines the historical backtesting procedure. Section \ref{sec:empirical} describes the data and the empirical results. We conclude in Section \ref{sec:conclusion}. The appendices provide further details on the data, empirical results, implementation of benchmarks, and a lengthy calculation used for the analysis of the graph Laplacian.

\section{Related Literature}
\label{sec:lit}
The current paper contributes to the growing bodies of literature on data-driven stochastic systems, manifold learning, and dynamic factor modeling.
Diffusion maps were introduced by \citet{coifman2006diffusion} as a geometrically faithful dimensionality reduction technique.
\citet{coifman2008diffusion} present diffusion maps for feature extraction in high-dimensional stochastic systems, in particular, the first few eigenfunctions of the backward Fokker--Planck diffusion operator are proposed as a coarse-grained low-dimensional representation for the long-term evolution of a system. Our framework leverages the ADM algorithm first introduced in \citet{singer2008nonlinear}, and falls within the larger category of models that find linear reduced-order representations of high-dimensional nonlinear stochastic systems. 

In its original setting, ADM recovers intrinsic parametrizations of independent and identically distributed (i.i.d.)~data that have undergone a nonlinear transformation. The error analysis of ADM is based on a result from \citet{singer_graph_2006} on the convergence of graph Laplacians, which we generalize from the case of uniform data on a compact manifold to data arising from a time series in Euclidean space. The heart of this analysis relies on a concentration inequality (called ``the inequality of Chernoff'' in \citet[paragraph before Equation (3.14)]{singer_graph_2006}) with the second moment of the summands as variance proxy. In \citet[Appendix B.1]{berry_variable_2016}, it is noted that Bernstein's Inequality (see, for instance, \citet[Section 2.8]{boucheron_concentration_2013}) is an applicable concentration inequality for the case of i.i.d.~samples. In recent years, many concentration inequalities have been generalized for data arising from Markov chains (see, for instance \citet{adamczak_tail_2008,douc_consistency_2011,miasojedow_hoeffdings_2014,paulin_concentration_2015,fan_hoeffdings_2021}), and we will draw on this literature in Section \ref{sec:results}.

As noted in the introduction, applying diffusion maps directly to time series data may be problematic since the data is not sampled uniformly and one should ideally take into account the temporal relationships in the data. Other works have also directly tackled the issue of explicitly incorporating the time dependencies into the diffusion coordinates or devised diffusion map-based models for multivariate time series. \citet{shnitzer2020diffusion} propose a nonparametric framework for state-estimation of stochastic dynamical systems that are high-dimensional and evolve according to gradient flows with isotropic diffusion. Their framework combines diffusion maps with Kalman filtering and concepts from Koopman operator theory, and relies on the existence of a linear lifting operator. In Section \ref{sec:jdkf_model}, we extend and generalize this framework to admit two datasets, one of responses and one of covariates, to perform supervised learning tasks and ensure the estimated hidden states contain explanatory power for the responses. \citet{shnitzer2022manifold} present a data-driven approach for time series analysis that combines diffusion maps with contracting observers from control theory, and their main goal is to construct a low-dimensional representation of high-dimensional signals generated by dynamical systems without requiring strong modeling assumptions. \citet{lian_multivariate_2015} introduce a method to analyze multivariate time series using diffusion maps on a parametric statistical manifold; the framework handles high-dimensional nonstationary signals that are subject to random interference and ambient noise, which could degrade the performance of more conventional methodologies.

Some of the aforementioned works, for instance, \citet{coifman2008diffusion} and \citet{shnitzer2020diffusion}, make the assumption that the dynamics of the intrinsic latent state evolve according to a Langevin equation, a key assumption we carry forward into our work. In those works, it is suggested that the dynamics of the diffusion coordinates should approximately follow linear SDEs. Indeed, in Section \ref{sec:results}, we establish a new robustness result showing that the approximation error using linear diffusions remains small when averaging over the data, and provided that some spectral assumptions are satisfied. When it comes to financial applications, it is worth noting that the class of Langevin diffusions encompasses a wide variety of stochastic processes ubiquitous in the finance literature (see 
Table \ref{table:stochproc}), exemplified by the Ornstein-Uhlenbeck process. For the use of Langevin dynamics to model stock prices see, for instance, the early works \citet{bouchaud1998langevin,Canessa2001}.

An important financial application where dynamic factor models can yield powerful methodologies is financial stress testing and scenario analysis, which consists in predicting a portfolio's profit \& loss (P\&L) when subsets of the risk factors that drive the portfolio's value experience shocks. In \citet{haugh2020scenario}, scenario analysis is embedded within a dynamic factor framework by defining a linear Gaussian state space model. There, derivatives portfolios are considered and the risk factors are low-dimensional, as is usually the case in stress testing. Our state space model, outlined in Section \ref{sec:jdkf_model}, allows one to define arbitrarily high-dimensional scenarios, and unlike the framework in \citet{haugh2020scenario} is purely data-driven and nonparametric, overcoming restrictive assumptions, like a VAR(1) process for the latent factors. Furthermore, our state space model is explicitly designed to admit two sets of observations (covariates and responses) and considers their joint dependence on the latent state as well as their joint dynamics, thus resulting in a supervised learning dynamic factor framework. 

\section{Preliminaries}
\label{sec:model}
In this section, we precisely state our mathematical setting and assumptions (Subsection \ref{sec:assumps}), and make explicit the data-driven embedding approach using ADM (Subsection \ref{sec:embedding}).
Throughout the remainder of the paper, we adopt the notations found in Table \ref{tab:notations}. In addition, we adopt the following conventions for stochastic processes: $X_t$ (upper case Latin letter with subscript) denotes a continuous time process at time $t$, $x(t)$ (lower case Latin letter with bracketed argument) is a discrete time sample of $X_t$, and $x_t$ (lower case Latin letter with subscript) is a linear state space approximation for $x(t)$. The convention for Greek letters will depend on context.

\begin{table}[h]
\centering
\begin{tabular}{|c|l|}
\hline 
\textbf{Notation} & \textbf{Description} \\\hline
$\nabla$ & Gradient operator\\
$\Delta$ & Laplacian operator\\ 
$C^2$ & Twice continuously differentiable functions\\
$\mathbb{L}^2(\mu)$ & Square-integrable functions w.r.t. measure $\mu$\\ 
$\langle \cdot,\cdot\rangle_\mu$ & $\mathbb{L}^2(\mu)$ inner product\\
$\|\cdot\|_{\mathbb{L}^2(\mu)}$ & $\mathbb{L}^2(\mu)$ norm\\
$\L$ & Fokker--Planck generator $\Delta - \nabla U \cdot \nabla$\\
$\theta_t$ & Latent state process\\
$X_t, Y_t$ & Covariate and response processes\\
$\mu$ & Invariant measure of $\theta$\\
$\lambda_k$ & $k$-th eigenvalue of $-\L$\\
$\varphi_k$ & $k$-th eigenfunction of $\L$\\
$\psi_k$ & $k$-th diffusion coordinate (empirical eigenvector)\\
$\mathbf{H}^x, \mathbf{H}^y$ & Lifting operators\\
$\ell$ & Dimension of diffusion coordinate space\\
\hline
\end{tabular}
\caption{Frequently used notations.}
\label{tab:notations}
\end{table}

% \begin{table}
% \centering
% \begin{tabular}{|c|l|}
% \hline Notation & Description \\\hline
% $\nabla$ & gradient\\
% $\Delta$ & Laplacian\\ 
% $I$ & identity matrix\\
% $C^2$ & space of twice differentiable functions\\
% $\bL^2(\mu)$ & space of square integrable functions with respect to a measure $\mu$\\ 
% $\langle \cdot,\cdot\rangle_\mu$ & $\bL^2(\mu)$ inner product\\
% $\| \cdot\|_{\bL^2(\mu)}$ & $\bL^2(\mu)$ norm\\
% \hline
% \end{tabular}
% \caption{Frequently-used notations.}
% \label{tab:notations}
% \end{table}

\subsection{Dynamical System Setting}\label{sec:assumps}

Given integers $m,n>0$, we model the \emph{covariates} $X=(X_t)_{t\ge0}$ and \emph{response variables} $Y=(Y_t)_{t\ge0}$ as continuous-time stochastic processes taking values in $\R^m$ and $\R^n$, respectively. We model $X$ and $Y$ as nonlinear mappings of a latent stochastic process $\theta=(\theta_t)_{t\ge 0}$, which takes values in $\R^d$ for some unknown dimension $d$. We introduce the key assumptions next.

\begin{assumption}[Langevin Diffusion]\label{ass1}
The latent variables $\theta=(\theta_t)_{t\ge 0}$ satisfy a Langevin diffusion:
\begin{align}\label{eq:thetaSDE}
\d \theta_t = -\nabla U(\theta_t)\d t+\sqrt{2}\d W_t
\end{align}
where $W=(W^1,\dots, W^d)$ is a Brownian motion in $\R^d$ and $U$ is a $C^2$ potential function satisfying the following growth condition at infinity:
\begin{align}\label{eq:growth}
\lim_{\|w\| \to\infty} \lb[-\Delta U(w)+\frac{1}{2}|\nabla U(w)|^2\rb]=\infty.
\end{align}
\end{assumption}
The invariant measure for the dynamics of $\theta$ is given by
\begin{align*}
\mu(\d w)=\frac{1}{Z}\exp(-U(w))\d w,
\end{align*}
where $Z$ is a normalizing constant (see, for instance, \citet[Section 1.15.7]{bakry_analysis_2014}). We let $p(w)=\frac{1}{Z}\exp(-U(w))$ denote the density of $\mu$.
Condition \eqref{eq:growth} ensures that the \emph{backward Fokker--Planck operator} or \emph{generator} associated with $\theta$ given by
\begin{align*}
\L :=\Delta -\nabla U\cdot \nabla
\end{align*}
has a discrete spectrum, see \citet[Corollary 4.10.9]{bakry_analysis_2014}, which allows for expansions of $\bL^2(\mu)$ functions using series of eigenfunctions.

\begin{assumption}[Manifold Hypothesis]\label{ass2}
There exist $C^2\cap \bL^2(\mu)$ functions $F:\R^d\to \R^m$ and $G:\R^d\to \R^n$ such that ${X_t=F(\theta_t)}$ and  ${Y_t=G(\theta_t)}$ for all $t\ge 0$. 
\end{assumption}

\begin{remark}
Often, a Langevin diffusion is given with the coefficient $\sqrt{2/\beta}$ in front of $\d W_t$, where $\beta>0$ is a fixed constant. Here, we choose to absorb this additional degree of freedom into the unknown functions $F$, $G$, and $U$. 
\end{remark}

Some vector calculus shows that the \emph{carr\'{e} du champ} operator for $\L$, denoted by $\Gamma$, and defined by $\Gamma(f,g)=\frac{1}{2}\lb(\L (fg)-f\L g-g\L f\rb)$ satisfies $\Gamma(f,g)=\nabla f\cdot \nabla g$.
Integration by parts gives
\begin{align}\label{eq:dirichlet}
-\int_{\R^d} f \L g \d\mu = -\int_{\R^d} g \L f \d\mu=\int_{\R^d} \Gamma(f,g) \d\mu=\int_{\R^d} \nabla f\cdot \nabla g \d\mu.
\end{align}
Therefore, the spectrum of $-\L$, denoted $0=\lambda_0\le \lambda_1\le\dots$,  is contained in $[0,\infty)$.
The diffusion maps algorithm uncovers the relationship between the latent variables and the observations by relying on the following fact: any ${f\in \bL^2(\mu)}$ can be decomposed as
\begin{align}\label{eq:fexpansion}
f(w)=\sum_{k=0}^\infty a_k \varphi_k(w)
\end{align}
where $(\varphi_k)_{k=0}^\infty$ is a complete orthonormal sequence of eigenfunctions for $\L$ and for each $k$
\begin{align}\label{eq:muprojection}
a_k=\langle f,\varphi_k\rangle_{\mu}=\int_{\R^d}f(w)\varphi_k(w)\mu(\d w).
\end{align}

Our next and final assumption ensures that $\lambda_0=0$ is a simple eigenvalue of $-\L$, and all of the other eigenvalues of $-\L$ are bounded away from zero.

\begin{assumption}[Poincar\'{e} Inequality]\label{ass3}
We assume there exists some $C>0$ so that for any ${f\in \bL^2(\mu)}$
\begin{align*}
\Var_{\mu}(f)\le C \E_{\mu}[\|\nabla f\|_2^2].
\end{align*}
\end{assumption}
If $\varphi_k$ is an eigenfunction of $\L$ with eigenvalue $\lambda_k\neq 0$, then equation \eqref{eq:dirichlet} along with a Poincar\'{e} Inequality implies that $\lambda_k\ge 1/C$. In fact, the above Poincar\'{e} inequality is equivalent to the existence of a \emph{spectral gap} for $-\L$, that is, $\lambda_1-\lambda_0=\lambda_1\ge 1/C$. We remark that the term spectral gap is often employed in the diffusion maps literature to mean a gap in the spectrum other than between $\lambda_0$ and $\lambda_1$ (see, for instance, \citet[Section 3.2]{coifman2008diffusion}).

\begin{remark}
If $\nabla^2 U\ge \rho I$ for some $\rho>0$, then $\mu$ satisfies a Poincar\'{e} inequality with $C= 1/\rho$. For instance, if $U(x)=\frac{1}{2}x^T Ax$ with $A$ positive definite, then $\mu$ is a non-degenerate Gaussian, the smallest eigenvalue of $A$ gives the spectral gap $\rho$, and $1/\rho$ is in fact the best Poincar\'{e} constant. However, this condition is not necessary, for instance, the two-sided exponential distribution $\frac{1}{2} e^{-|x|}$ on $\R$ satisfies a Poincar\'{e} inequality with $C=4$.
\end{remark}

It\^{o}'s formula yields the dynamics for a $C^2$ function of $\theta$, such as $X$ or $Y$. For instance, ${X=(X^1,\dots,X^m)}$ satisfies (in components):
\begin{equation}\label{eq:XSDE}
\begin{split}
\d X_t^j&=\lb(-\nabla F_j(\theta_t)\cdot \nabla U(\theta_t^i)+\Delta F_{j}(\theta_t)\rb)\d t+\sqrt{2}\nabla F_j(\theta_t) \cdot \d W_t \\
&= \L F(\theta_t) \d t+\sqrt{2}\nabla F_j(\theta_t) \cdot \d W_t\\
&=\L F(\theta_t) \d t+\sqrt{2}\|\nabla F_j(\theta_t)\|_2 \cdot \d B_t^j,\quad j=1,\dots, m,
\end{split}
\end{equation}
where $B^j$ is a Brownian motion by L\'{e}vy's Theorem (see, for instance, a similar derivation for the case of Bessel processes in \citet[Proposition 3.3.21]{karatzas_brownian_2014}).
And similarly, writing $Y=(Y^1,\dots,Y^n)$ we have $\d Y_t^j=\L G(\theta_t) \d t+\sqrt{2}\nabla G_j(\theta_t) \cdot \d W_t$, and so on. 
These SDEs admit the flexibility to capture many of the standard models of mathematical finance, as shown in Table \ref{table:stochproc}. 
For illustrative purposes, we will work with a low-dimensional example which gives rise to Ornstein--Uhlenbeck (OU) and Cox--Ingersoll--Ross (CIR) processes:
\begin{example}[OU and CIR Processes]
Suppose that $U$ is given by $U(w)=\frac{1}{2}(w_1^2+w_2^2)$ for $w\in\R^2$. In this case, $\theta$ is a 2-dimensional Ornstein--Uhlenbeck process satisfying $\d\theta_t=-\theta_t \d t+\sqrt{2}\d W_t$, the generator satisfies $\L f(w)=\Delta f(w) -w\cdot \nabla f(w)$, and the invariant distribution $\mu$ is $N(0,I)$. The Hermite polynomials $(h_k)_{k\ge 0}$ are defined by
\begin{align*}
h_k(x)=\frac{(-1)^k}{\sqrt{k!}}\exp\lb(\frac{x^2}{2}\rb)\frac{\d^k}{\d x^k}\exp\lb(-\frac{x^2}{2}\rb), \quad x\in \R.
\end{align*}
The first several Hermite polynomials are given by $h_0=1$, $h_1(x)=x$, and $h_2(x)=\frac{1}{\sqrt{2}}(x^2-1)$.
Note that $\L h_k(w_1)=-kh_k(w_1)$ and $h_k'(w_1)=\sqrt{k}h_{k-1}(w_1)$ for $w_1\in \R$.
An orthonormal basis for $\bL^2(\mu)$ is given by the family $(h_{i,j})_{i,j\ge 0}$ where $h_{i,j}(w):=h_i(w_1)h_j(w_2)$ for any $w=(w_1,w_2)\in \R^2$. Consider the case where $X_t=(\theta_t^1)^2+(\theta_t^2)^2$ so that
\begin{align*}
\d X_t &= -2\lb[(\theta_t^1)^2+(\theta_t^2)^2-2\rb]\d t+2\sqrt{2}\theta_t^1 \d W_t^1+2\sqrt{2}\theta_t^2 \d W_t^2=-2(X_t-2)\d t+2\sqrt{2X_t}\d B_t,
\end{align*}
where $B$ is a standard Brownian motion, by L\'{e}vy's Theorem. $X$ is a CIR process, a popular model for non-negative data such as interest rates, typically. We may take $Y_t=\theta_t^1$, so that $Y$ is a 1-dimensional OU process. We note for later use that $\d \langle X,Y\rangle_t=4Y_t \d t$.
\label{ex:example}
\end{example}

\begin{table}
\begin{center}
\begin{tabular}{| c | c | c | c |}
\hline
$U$ & $F$ & $X$  \\ 
\hline
linear & exponential & geometric Brownian motion\\
linear & $\|\cdot\|_2$ & Bessel\\
quadratic & linear & Ornstein--Uhlenbeck/Vasicek\\
quadratic & quadratic & Cox--Ingersoll-Ross (CIR)\\
quadratic & power & constant elasticity of variance (CEV)\\
\hline
\end{tabular}
\end{center}
\caption{Common stochastic processes from mathematical finance which can be written in the form $X_t=F(\theta_t)$ with $\d \theta_t = -\nabla U(\theta_t)\d t+\sqrt{2}\d W_t$. These may all be checked using It\^{o}'s formula, see for instance, \citet[Sections 2.6 and 6.4]{jeanblanc_mathematical_2009} for the necessary computations in the case of CIR and CEV processes. Note also that a linear $U$ does not satisfy the growth condition \eqref{eq:growth}; however, these processes can be approximated locally.}
\label{table:stochproc}
\end{table}

\begin{remark}\label{rmk:homogenous}
In \citet{singer2008nonlinear}, diffusion maps were used to uncover the dynamics of $\theta$ for the case where the components $\theta^1,\dots \theta^d$ are \emph{independent} homogeneous diffusions of the form $\d \theta_t^i = b^i(\theta_t^i)\d t+\sigma^i(\theta^i)\d W_t^i$ with drifts $b^1,\dots, b^d$ and non-degenerate diffusion coefficients $\sigma^1,\dots,\sigma^d$. In this case, one may apply a Lamperti Transformation (see \citet[Chapter 5, Exercise 2.20]{karatzas_brownian_2014}) to map the diffusion coefficients to constants, provided some integrability and boundedness assumptions hold on $b^i$ and $\sigma^i$. Specifically, for each $i$, let $g^i$ be an antiderivative of $1/\sigma^i$ and define $\wt\theta^i_t=g^i(\theta_t^i)$. Then It\^{o}'s lemma gives
\begin{align*}
\d\wt\theta^i_t=\lb(\frac{b^i(\theta_t^i)}{\sigma^i(\theta_t^i)}-\frac{(\sigma^i)'(\theta_t^i)}{2}\rb)\d t+\d W_t^i
\end{align*}
\end{remark}

While the functions $U,F,G$ and the dimension $d$ are unobserved, the quadratic covariation processes of $X$ and $Y$ yield information about the Jacobians of $F$ and $G$:
\begin{align*}
\d \langle X^j,X^k\rangle_t&=2\nabla F_j(\theta_t) \cdot \nabla F_k(\theta_t)\d t,\quad j,k=1,\dots, m\\
\d \langle Y^j,Y^k\rangle_t&=2\nabla G_j(\theta_t) \cdot \nabla G_k(\theta_t)\d t,\quad j,k=1,\dots, n.
\end{align*}
Let $w,\wt w\in \R^d$ and set $x=F(w)$ and $\wt x=F(\wt w)$. Denoting the Jacobian matrix for $F$ by $J_F$, a Taylor approximation as in \citet{singer2008nonlinear} gives
\begin{align}\label{eq:mahalanobis}
\|w-\wt w\|^2&=\frac{1}{2}(x-\wt x)^T\lb((J_FJ_F^T)^{-1}(x)+(J_FJ_F^T)^{-1}(\wt x)\rb)(x-\wt x)+O(\|x-\wt x\|^4)
\end{align}
The first term on the right is a \emph{squared modified Mahalanobis distance}. Hence, distances in the latent space can be approximated using observations of $X$ and/or $Y$ along with estimates of the quadratic covariation processes.
In the next subsection, we follow this insight to uncover the eigenvalues and eigenfunctions of the Fokker--Planck operator associated with the dynamics of $\theta$.

\subsection{Data-Driven Embedding}\label{sec:embedding}

We outline here the anisotropic diffusion maps algorithm for learning the eigenvalues and eigenfunctions of the generator $\L$ from discrete observations. The algorithm constructs a graph Laplacian from data that approximates $I+\veps \L$ for small $\veps > 0$, enabling recovery of the generator's spectral properties without parametric assumptions.

Let $Z=H(\theta)$ be a $C^2\cap \bL^2(\mu)$ mapping of $\theta$, and suppose we are given a data set of $N+1$ discrete-time observations of $Z$, denoted $(z(t_i))_{i=0}^N$.
Depending on the application at hand, and the size and quality of the data sets, we may for instance take $Z=X$, $Z=Y$, or $Z=(X,Y)$. To fix notation, we consider the situation here where $Z=(X,Y)$ so that $H=(F,G):\R^{d}\to\R^{m+n}$. We assume that the data has been mean-centered.

Although we do not observe $\theta$ directly, distances in the latent space can be approximated using observations of $Z$. Motivated by the Taylor approximation \eqref{eq:mahalanobis} in the previous subsection, we consider the squared modified Mahalanobis distance between $z_i$ and $z_j$:
\begin{align*}
d(z(t_i),z(t_j)):=\frac{1}{2}(z(t_i)-z(t_i))^T(C^{-1}(t_i)+C^{-1}(t_j))(z(t_i)-z(t_j)),
\end{align*}
where $C(t_i)$ denotes the instantaneous quadratic covariation matrix of $z(t_i)$. This approximates $\|\theta(t_i)-\theta(t_j)\|^2$ up to $O(\|z(t_i)-z(t_j)\|^4)$ by equation \eqref{eq:mahalanobis}. In Section \ref{sec:validation}, we approximate $C(t_i)$ by its sample average in a trailing window. For strictly positive financial data, one can also use an estimator such as the one presented in \citet{chesney_estimating_1995}.

We construct a squared distance matrix for the data $\mathbf{D} \in \R^{N+1}\times\R^{N+1}$ with entries \newline ${\mathbf{D}_{ij}=d(z(t_i),z(t_j))}$. We pass the entries of $\mathbf{D}$ through the Radial Basis Function (RBF) kernel $\exp\lb(-d(\cdot,\cdot)/2 \veps\rb)$ with parameter $\veps>0$ to obtain an adjacency matrix $\mathbf{W}$:
\begin{align*}
\mathbf{W}_{ij}=
\exp\lb(-\frac{d(z(t_i),z(t_j))}{2 \veps}\rb).
\end{align*}
A popular choice for practitioners is to choose $\veps$ as the median of all distances between data points (see \citet{shnitzer2020diffusion} or Subsection \ref{sec:experiments}). Another methodology (see \citet[Figure 3]{singer_detecting_2009}) is to choose $\veps$ in a region where the Log-Log plot of the sum of all entries of $\mathbf{W}$ versus $\veps$ is linear.
In any case, $\mathbf{W}$ is then normalized to create a stochastic matrix $\mathbf{P}$ (which is in fact doubly-stochastic as it is symmetric):
\begin{align*}
\mathbf{P}_{ij}=\frac{\mathbf{W}_{ij}}{\sum_{j=0}^{N} \mathbf{W}_{ij}}.
\end{align*}

In Theorem \ref{thm:concentration}, we show that the \emph{graph Laplacian} $\mathbf{P}-I$ approximates the Fokker--Planck operator $\L$ in the following sense: if $(z(t_i))_{i=0}^{N}=(H(\theta(t_i)))_{i=0}^{N}$ where $\theta$ satisfies the dynamics in Assumption \ref{ass1}, and if $f:\R^{m+n}\to \R$ is such that $f\circ H\in \bL^2(\mu)$ then the $i$th row of $(\mathbf{P}-I)$ applied to the vector $[f(z(t_0)),\dots,f(z(t_N))]$ satisfies
\begin{align*}
\sum_{j=0}^N\mathbf{P}_{ij} f(z(t_j))-f(z(t_i))=\veps \L (f\circ H)(\theta(t_i))+O(\veps^2)
\end{align*}
with high probability. This is the key insight of diffusion maps: if $f$ is an eigenvector of $\mathbf{P}$ with eigenvalue $\kappa$ (note that $\mathbf{P}-I$ and $\mathbf{P}$ share eigenvectors), then $f\circ H$ is approximately an eigenfunction of $\L$. Using the approximation $\mathbf{P}\approx I+\veps \L\approx \exp(\veps \L)$, an estimate for the associated eigenvalue of $-\L$ is given by the spectral mapping $\lambda=-\veps^{-1}\log \kappa>0$. Since the spectrum of $\L$ is discrete, one can hope to approximate functions of the latent state $\theta$, using truncations of the eigenfunction expansion \eqref{eq:fexpansion}. This motivates using the eigenvectors of $\mathbf{P}$ as \emph{diffusion coordinates}, a geometrically faithful representation of the underlying system states.
We denote the ordered eigenvalues and an associated basis of $\R^{N+1}$-orthonormal eigenvectors of $\mathbf{P}$ by
\begin{align*}
1=\kappa_0\geq ... \geq \kappa_{N} \geq 0,\quad \text{and}\quad \psi_0,...,\psi_{N}.
\end{align*}
As our data is mean-centered, we discard the first eigenvector (a constant vector corresponding to eigenvalue $1$) and keep eigenvectors 1 through $\ell$, for some $\ell>0$. Practitioners commonly choose $\ell$ to correspond to the eigenvalue preceding the largest drop-off $\kappa_{i}-\kappa_{i+1}$ (see, for instance, Subsection \ref{sec:experiments} or \citet[Subsection 3.2]{coifman2008diffusion}). In the context of Example \ref{ex:example}, a higher value of $\ell$ corresponds to a set of polynomials with higher maximum degree.

\begin{definition}
The \emph{diffusion coordinates} at each time $t_i$ are given by $\psi(t_i):=[\psi_1(t_i), ..., \psi_\ell(t_i)]$.
\end{definition}

Following, \citet{shnitzer2020diffusion}, we approximate the observed data by a \emph{lifting operator} $z(t_i)=\mathbf{H} \psi(t_i)$, which is given in coordinates by 
\begin{align*}
z_j(t_i)\approx \sum_{k=1}^\ell \mathbf{H}_{j,k} \psi_k(t_i), \quad j=1,2,...,m
\end{align*}
where $\mathbf{H}_{j,k}=\langle z_j,\psi_k \rangle_{\text{data}}:=(N+1)^{-1}\sum_{i=0}^N z_j(t_i)\psi_k(t_i)$. We note the analogy to \eqref{eq:muprojection}, the $\bL^2(\mu)$-projection of $Z_t^j=H(\theta_t)$ onto an $\L$-eigenfunction $\varphi_k$. In Proposition \ref{prop:clt}, we consider $N^{-1/2}\mathbf{H}_{j,k}$ for large $N$, which is asymptotically normal.

\begin{remark}\label{rmk:scaling}
Notice that $\langle \cdot,\cdot \rangle_{\text{data}}$ is defined in \citet{shnitzer2020diffusion} without the normalization by $N+1$. Note also that the diffusion coordinates themselves are of order $O(N^{-1/2})$, as they are the coordinates of $\R^{N+1}$-orthonormal vectors. These scaling factors are of no consequence for the state-space model presented in Section \ref{sec:jdkf_model}, since the diffusion coordinates are treated as hidden variables and may be rescaled as one wishes.
\end{remark}

Next, we consider the dynamics of the eigenfunctions of $\L$, which we will approximate using diffusion coordinates. Suppose that $\varphi_1,\dots,\varphi_\ell$ satisfy $\L \varphi_i=-\lambda_i \varphi_i$ for $i= 1,\dots,\ell$. Just as with \eqref{eq:XSDE}, Ito's lemma gives
\begin{equation}\label{eq:phiSDE}
\begin{aligned}
\d\varphi_i(\theta_t)&=-\nabla U(\theta_t)\cdot \nabla \varphi_i(\theta_t)\d t+\Delta\varphi_i(\theta_t) \d t+\sqrt{2}\nabla \varphi_i(\theta_t) \cdot \d W_t\\
&=\L \varphi_i (\theta_t)\d t+\sqrt{2}\|\nabla \varphi_i(\theta_t)\|_2 \d B^i_t\\
&=-\lambda_i \varphi_i(\theta_t)\d t+\sqrt{2}\|\nabla \varphi_i(\theta_t)\|_2 \d B^i_t,
\end{aligned}
\end{equation}
where $B^i$ is a Brownian motion and we have applied L\'{e}vy's Theorem for the second equality.
Note also that for $i\neq j$ we have
\begin{align*}
\d\langle\varphi_i(\theta_\cdot),\varphi_j(\theta_\cdot)\rangle _t&=2\nabla \varphi_i(\theta_t)\cdot\nabla\varphi_j(\theta_t) \d t
\end{align*}
The right hand side is twice the \emph{carr\'{e} du champ} operator for $\L$ applied to $\varphi_i$ and $\varphi_j$. The identity \eqref{eq:dirichlet} gives
\begin{align*}
\int \nabla \varphi_i(w)\cdot\nabla\varphi_j(w) \mu(\d w)=- \int \varphi_i(w) \L \varphi_j(w) \mu(\d w)=0,
\end{align*}
by orthogonality of eigenfunctions.
The above analysis suggests that the dynamics of the eigenvectors $\psi(t)$ of $\mathbf{P}$ can be approximated using uncorrelated processes with linear drifts determined by the associated eigenvalues. Indeed, this is justified rigorously in Subsection \ref{sec:linearapprox}. Along with the lifting operator $\mathbf{H}$, this motivates the use of a linear state space model and Kalman filter as described in the next subsection.

Before turning to our theoretical results and data-driven methodology, we return to Example \ref{ex:example} and illustrate the use of diffusion maps on synthetic data. 
We perform a simulation of the observation vector $z(t)=(x(t),y(t))$ composed of 500 discrete-time samples of the processes $X$ and $Y$, respectively. The data has been mean-centered and the covariation process is approximated using a windowed covariance matrix, with a trailing window of size 50. Since we have access to the true $\theta$ and the ground-truth Hermite polynomial eigenfunctions, we can assess how well the diffusion coordinates recover the eigenfunctions of $\L$, as well as the reconstruction quality of the $X$ and $Y$ processes via the lifting operator $\mathbf{H}$, using the dominant $10$ diffusion map eigenvectors. Figures \ref{fig:reconstructed_time_series_1} and \ref{fig:reconstructed_time_series_2} display the reconstructed $X$ and $Y$ processes both in terms of time series and hypersurfaces.

\begin{figure}
    \centering
    \subfloat[True surfaces $X_t = (\theta_t^1)^2 + (\theta_t^2)^2$ and $Y_t = \theta_t^1$ 
(colored by value) overlaid with lifted reconstructions $\mathbf{H}\psi(t)$ (stem plots).]{
    \includegraphics[width=0.9\linewidth, trim={0 0 .3cm, 1cm},clip]{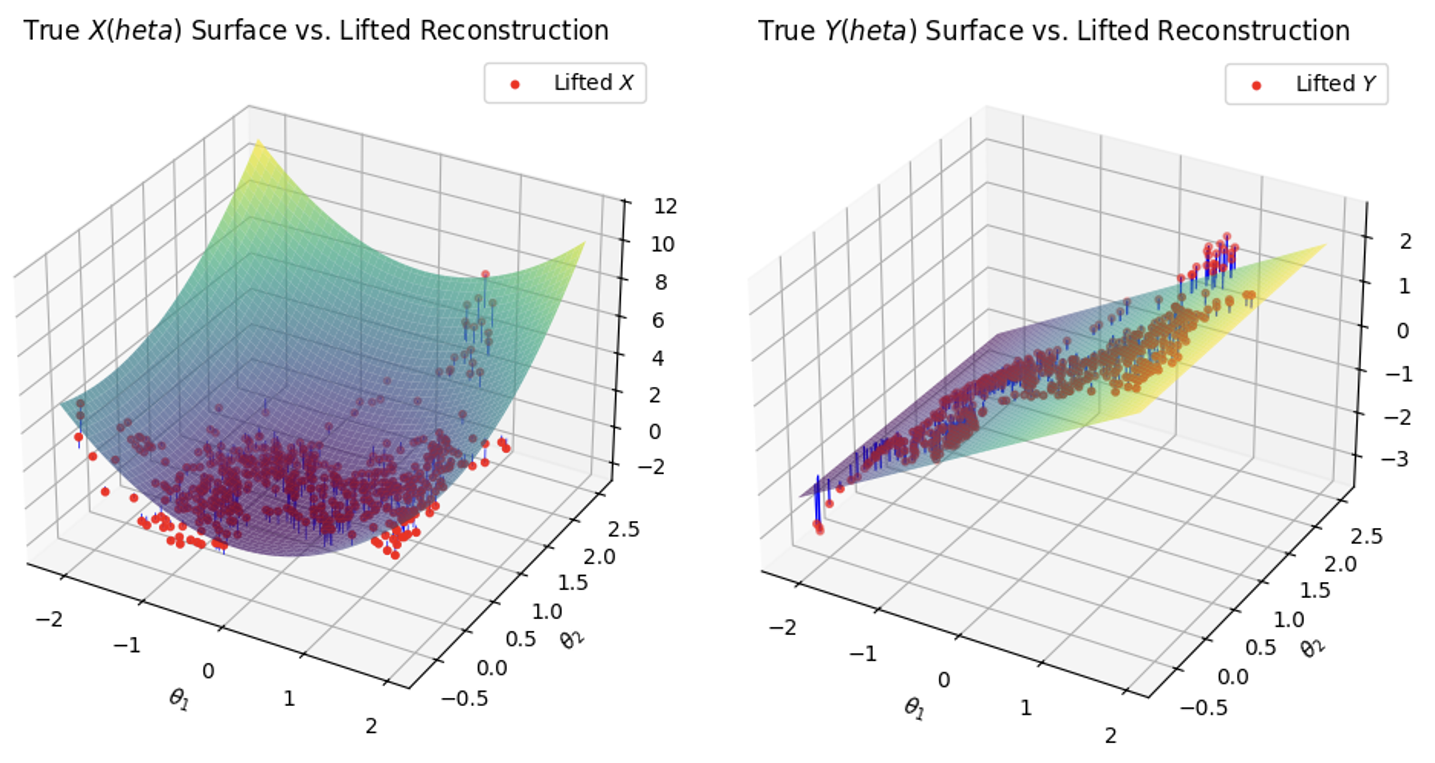} \label{fig:reconstructed_time_series_1}}
    \\
    \subfloat[
    True $X_t$ (left) and $Y_t$ (right) are plotted as functions of time alongside the lifted time series via $\mathbf{H}\psi(t)$.
     ]{  \includegraphics[width=0.98\linewidth, trim={0.5cm 0 .19cm, 1.1cm},clip]{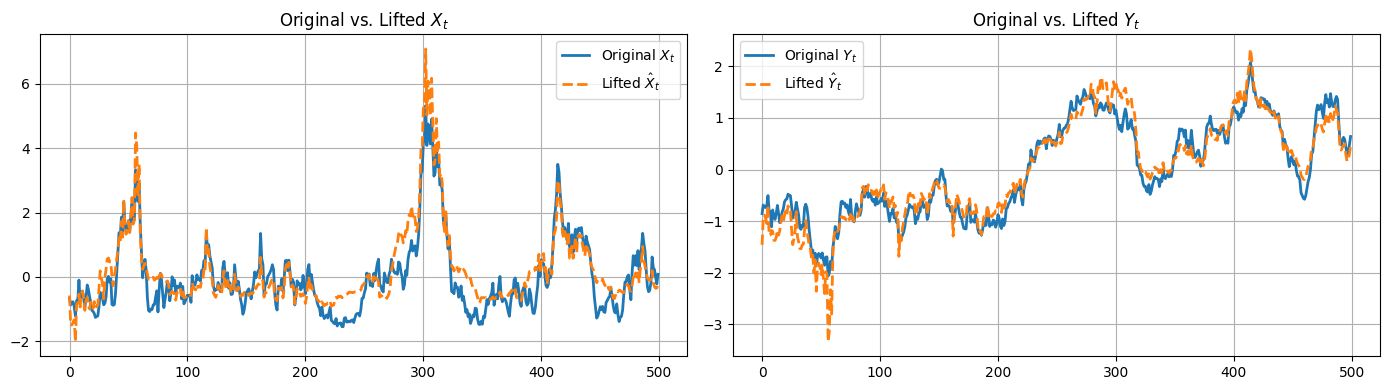} \label{fig:reconstructed_time_series_2}}
     \caption{
Comparisons showing that $\ell=10$ diffusion coordinates accurately reconstruct the OU and CIR dynamics from  Example \ref{ex:example}.
}
\end{figure}

Similarly, Figures \ref{fig:eigenfunction_1} and \ref{fig:eigenfunction_2} below display two ground-truth Hermite eigenfunctions of the backward Fokker-Planck operator $\L$ and their approximation via the diffusion map eigenvectors, both in terms of time series and hypersurfaces. The observed fit is satisfactory in these cases; however, we note that the diffusion maps algorithm may be learning a \emph{different} basis than these Hermite eigenfunctions (even restricting to polynomials, there are an uncountable number of such bases due to the rotational invariance of the Gaussian density).

\begin{figure}[H]
    \centering
    \subfloat[Hermite polynomials $h_{1,0}$ (left) and $h_{1,3}$ (right) surfaces as a function of $\theta_t$ alongside stem plots of $\psi_1(t)$ and $\psi_3(t)$ respectively.]{
    \includegraphics[width=0.94\linewidth, trim={0 0 0cm, .8cm},clip]{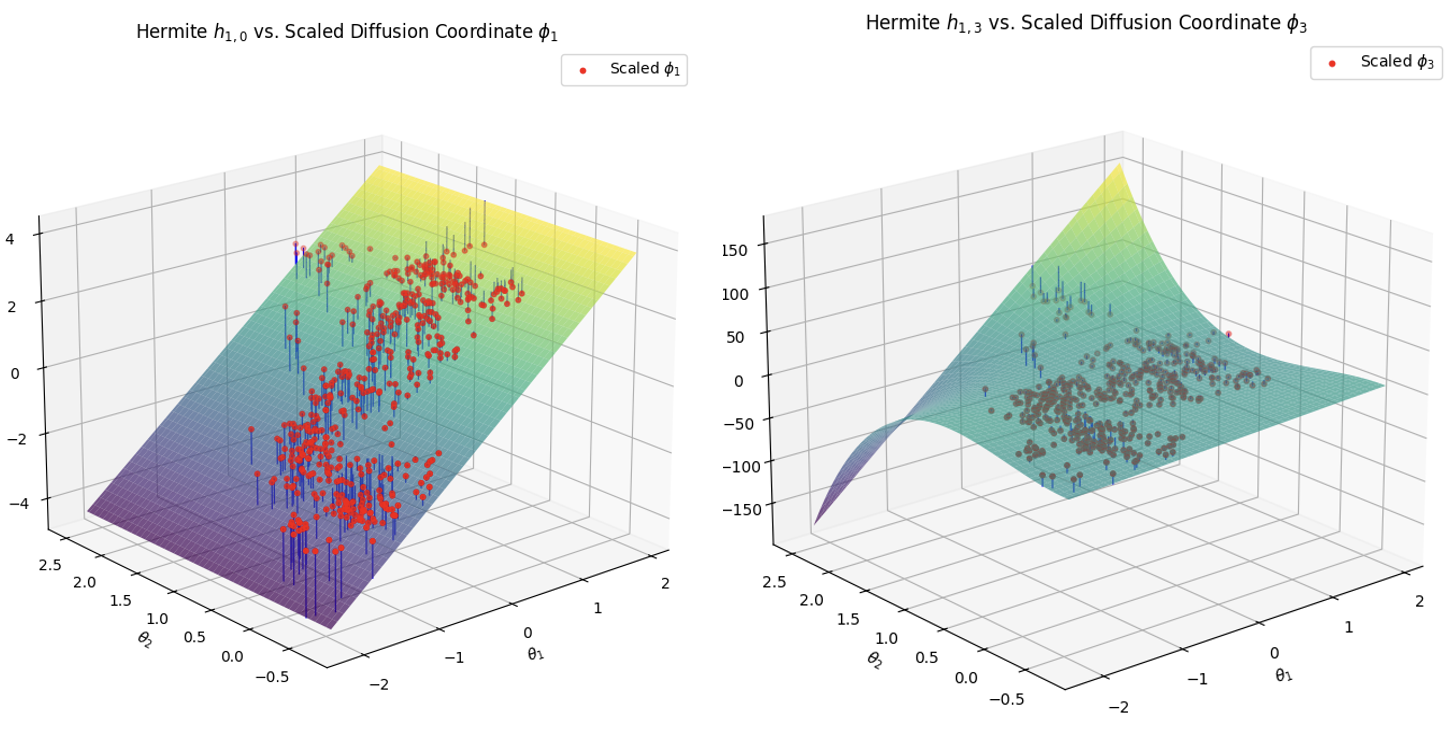} \label{fig:eigenfunction_1}}
    \\
    \subfloat[Hermite polynomials $h_{1,0}$ (left) and $h_{1,3}$ (right) as a function of time against the time series $\psi_1(t)$ and $\psi_3(t)$ respectively.]{
    \includegraphics[width=0.98\linewidth, trim={0 .5cm .11cm, .45cm},clip]{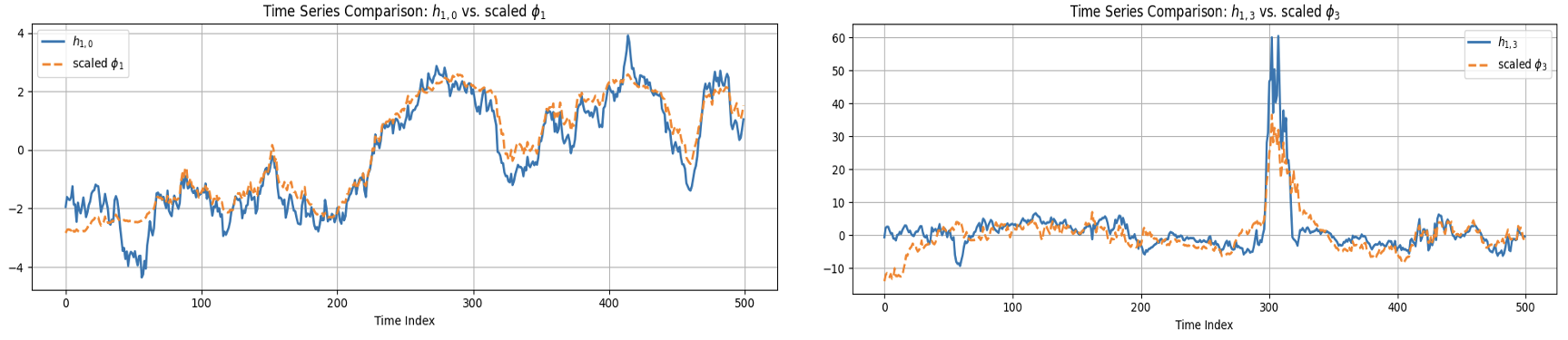} \label{fig:eigenfunction_2}}
    \caption{
   Eigenfunction recovery and path reconstructions for the simulation of Example \ref{ex:example}
}
\end{figure}

% Theory

\section{Convergence of Averages and Robustness of Approximations}
\label{sec:results}

Having described the diffusion maps algorithm and illustrated its performance on synthetic data, in this section we consider the theoretical foundations. We establish three key results: (i) the graph Laplacian constructed from time series data converges to the true generator (Theorem \ref{thm:concentration}), (ii) the lifting operator converges to the $\mathbb{L}^2(\mu)$ projection (Proposition \ref{prop:clt}), and (iii) the linear approximation of eigenfunction dynamics is robust under ergodic averaging (Section \ref{sec:linearapprox}).  These results rigorously justify the Kalman filtering framework we develop in Section \ref{sec:jdkf_model}.

The works by \citet{shnitzer2020diffusion,shnitzer2022manifold} leverage the algorithm put forth in \citet{singer2008nonlinear} to learn the eigenvalues of $\L$, to estimate the dynamics of the eigenfunctions, and to approximate the observations using a linear lift. The Kalman filtering framework of \citet{shnitzer2020diffusion} is the motivation for the Joint Diffusion Kalman Filter and conditional sampling procedure we subsequently develop in Section \ref{sec:jdkf_model}. However, there are some mathematical questions that first need to be addressed in order to rigorously justify this procedure. Under the assumptions in Subsection \ref{sec:assumps}, we show in this section that it is admissible to learn the eigenvalues and the lift from time series data. In particular, a Poincar\'{e} Inequality (Assumption \ref{ass3}) implies that initialization errors are forgotten, and time averages converge (see Subsection \ref{sec:lifts}). In Subsection \ref{sec:concentration}, we bound the approximation error of the eigenfunctions of $\L$ by the diffusion map eigenvectors using a concentration inequality. Furthermore, in Subsection \ref{sec:linearapprox} we quantify the error of approximating the eigenfunction dynamics using linear diffusions, demonstrating robustness.

At the heart of the error analysis for diffusion maps and related algorithms is the consideration of empirical averages of the form
\begin{align}\label{eq:ergodic}
S_N=\frac{1}{N+1}\sum_{i=0}^{N} V(\theta(t_i)).
\end{align}
For instance, in the setting of \citet{singer_graph_2006}, $(\theta(t_i))_{i=0}^N$ are independent uniform samples from a compact Riemannian manifold $\M$, and the graph Laplacian is shown to approximate the Laplace--Beltrami operator $\Delta_\M$, that is, the following holds for smooth $f$ and a point $\theta(t_i)$ fixed:
\begin{align}\label{eq:quotienttaylor}
\frac{\sum_{j\neq i} \exp\lb(-\frac{\|\theta(t_i)-\theta(t_j)\|^2}{2\veps}\rb)f(\theta(t_i))}{\sum_{j\neq i}\exp\lb(-\frac{\|\theta(t_i)-\theta(t_j)\|^2}{2\veps}\rb)}\to \frac{\E \lb[\exp\lb(-\frac{\|\theta-\theta(t_i)\|^2}{2\veps}\rb)f(\theta)\rb]}{\E \lb[\exp\lb(-\frac{\|\theta-\theta(t_i)\|^2}{2\veps}\rb)\rb]}=f(\theta(t_i))+\frac{\veps}{2}\Delta_\M f(\theta(t_i))+O(\veps^{2}).
\end{align}
The first arrow denotes convergence of the numerator and denominator to their averages by the law of large numbers, and these averages undergo Taylor expansions to give the error estimate in $\veps$. Error analysis in the number of samples is then carried out by showing concentration of the empirical averages around their expectation.

In our setting, the summands in \eqref{eq:ergodic} arise not from independent samples, but from a sampled Langevin diffusion, which yields a Markov chain. In this case, the convergence of $S_N$ to the expectation $\E_{\mu}[V]$ as $N\to\infty$ is an \emph{ergodic property}, which is guaranteed for $V\in\bL^2(\mu)$ under our assumption of a Poincar\'{e} inequality or (equivalently) a spectral gap.
If the Langevin diffusion is started in equilibrium, i.e.~$\theta_0\sim\mu$, then $(\theta(t_i))_{i=0}^N$ is a stationary sequence and $\E_{\theta_0\sim\mu}[S_N]=\E_{\mu}[V]$ for any $N\ge0$. In this case, the fluctuations $\sqrt{N}S_N$ converge as $N\to\infty$ to a Gaussian distribution by the celebrated Central Limit Theorem of \citet{kipnis_central_1986}. 
A fast rate of convergence to equilibrium for Markov chains is a desirable property in the context of Markov chain Monte Carlo algorithms, and many central limit theorem results are collected in the survey paper \citet{roberts_general_2004}. For continuous time averages, we refer the reader to \citet{cattiaux_central_2012}, where the trend to equilibrium is studied under weaker assumptions than ours (for instance, some laws with sub-exponential tails are shown to be ergodic). In general, the rate of convergence depends both on the integrability of $V$ and the weight of the tails of $\mu$. To fix notation, throughout the remainder of this section we consider a constant fixed step size $\Delta t=t_i-t_{i-1}$ and set $t_0=0$.

\subsection{Short-Term Equilibration and Long-Term Convergence}\label{sec:lifts}

We collect here some results on Markov chains which hold under Assumption \ref{ass3}, and apply these results to the sampled Markov process $(\theta(t_i))_{i=0}^N$. For any $t\ge0$, consider the operator $P_t$ which acts on $f\in\bL^2(\mu)$ by $P_tf(x)=\E[f(\theta_t)|\theta_0=x]$. Some calculus gives
\begin{align*}
\pp_t \|P_tf\|_{\bL^2(\mu)}^2=\pp_t \int_{\R^d} (P_tf)^2\d \mu=2\int_{\R^d} P_tf \L P_tf \d \mu =-2\int_{\R^d} \Gamma(P_tf,P_tf)\d \mu\le -\frac{2}{C} \|P_tf\|_{\bL^2(\mu)}^2,
\end{align*}
where we have used Assumption \ref{ass3} for the inequality. The differential form of Gr\"{o}nwall's Inequality yields $\|P_tf\|_{\bL^2(\mu)}^2\le \exp(-2t/C)\|f\|_{\bL^2(\mu)}^2$. Hence, $P_tf$ converges to equilibrium in $\bL^2$ exponentially fast. In fact, this property is equivalent to a Poincar\'{e} Inequality (see for instance \citet[Theorem 4.2.5]{bakry_analysis_2014}). In the next lemma, we exploit this fact to bound the deviation of observables from their values at equilibrium when $\theta_0$ is initialized at some $\mu_0$ which is absolutely continuous with respect to the equilibrium measure $\mu$. In practice, this error can be diminished by discarding an initial \emph{burn-in} period (as we have done in Subsection \ref{sec:backtest}). Similar results hold in other $\bL^p$-spaces by Riesz--Thorin Interpolation (see, for instance, \citet[Proposition 3.17]{rudolf_explicit_2012} and the application to MCMC methods in \citet[Theorem 12]{fan_hoeffdings_2021}), but we restrict to the $\bL^2$ case to reduce notation.

\begin{lemma}\label{lem:burnin}
Suppose $\theta_0\sim \mu_0$ where $\mu_0$ admits a Radon--Nikodym derivative with respect to the invariant measure $\mu$ such that $\|\frac{\d\mu_0}{\d\mu}-1\|_{\bL^2(\mu)}<\infty$, and let $\mu_t$ denote the distribution of $\theta_t$. If $f\in \bL^2(\mu)$ then for $t\ge0$,
\begin{align*}
\E_{\theta_0\sim\mu_0}\lb[f(\theta_t) \rb]\le \E_\mu[f]+
\exp\lb(-\lambda_1 t\rb)\lb\|\frac{\d\mu_0}{\d\mu}-1\rb\|_{\bL^2(\mu)}\lb\|f\rb\|_{\bL^2(\mu)}
\end{align*}
\end{lemma}

\begin{proof}
{Since $\mu_t$ is absolutely continuous with respect to $\mu$ (as the Langevin diffusion preserves absolute continuity), we can write}
\begin{align*}
\E_{\theta_0\sim\mu_0}\lb[f(\theta_t)\rb] 
&= \int_{\R^d} f(w) \mu_t(\d w) = \int_{\R^d} f(w) \frac{\d\mu_t}{\d\mu}(w) \mu(\d w)\\
&= \int_{\R^d} f(w) \lb(1 + \frac{\d\mu_t}{\d\mu}(w) - 1\rb) \mu(\d w)\\
&= \E_\mu[f] + \int_{\R^d} f(w) \lb(\frac{\d\mu_t}{\d\mu}(w) - 1\rb) \mu(\d w).
\end{align*}
{Applying H\"{o}lder's inequality to the second term yields
\begin{align*}
\lb|\int_{\R^d} f(w) \lb(\frac{\d\mu_t}{\d\mu}(w) - 1\rb) \mu(\d w)\rb|
&\leq \|f\|_{\bL^2(\mu)} \lb\|\frac{\d\mu_t}{\d\mu} - 1\rb\|_{\bL^2(\mu)}.
\end{align*}}

{The density evolution $\frac{\d\mu_t}{\d\mu}$ satisfies $\frac{\d\mu_t}{\d\mu} = P_t\lb(\frac{\d\mu_0}{\d\mu}\rb)$ where $P_t$ is the Markov semigroup associated with $\theta$. Since $\E_\mu\lb[\frac{\d\mu_0}{\d\mu}\rb] = 1$, we have $\E_\mu\lb[\frac{\d\mu_0}{\d\mu} - 1\rb] = 0$, and thus}
\begin{align*}
\lb\|\frac{\d\mu_t}{\d\mu} - 1\rb\|_{\bL^2(\mu)} 
&= \lb\|P_t\lb(\frac{\d\mu_0}{\d\mu} - 1\rb)\rb\|_{\bL^2(\mu)}.
\end{align*}
{By the $\bL^2$-exponential convergence of $P_t$ established via Gr\"{o}nwall's inequality in the text preceding this lemma (with rate $\lambda_1 = 1/C$ where $C$ is the Poincar\'{e} constant), we obtain
\begin{align*}
\lb\|P_t\lb(\frac{\d\mu_0}{\d\mu} - 1\rb)\rb\|_{\bL^2(\mu)}
&\leq \exp\lb(-\lambda_1 t\rb) \lb\|\frac{\d\mu_0}{\d\mu} - 1\rb\|_{\bL^2(\mu)}.
\end{align*}
Combining these estimates yields the desired result.}
\end{proof}

% \begin{proof}
% We have that
% \begin{align*}
% \E_{\theta_0\sim\mu_0}\lb[f(\theta_t) \rb]&=\int_{\R^d}f(w)\lb(1+\frac{\d\mu_t}{\d\mu}(w)-1\rb)\mu(\d w)\\
% &\le \E_\mu[f]+\lb\|\frac{\d\mu_t}{\d\mu}-1\rb\|_{\bL^2(\mu)}\lb\|f\rb\|_{\bL^2(\mu)}\\
% &\le\E_\mu[f]+
% \exp\lb(-\lambda_1 t\rb)\lb\|\frac{\d\mu_0}{\d\mu}-1\rb\|_{\bL^2(\mu)}\lb\|f\rb\|_{\bL^2(\mu)}
% \end{align*}
% where we have used H\"{older's} Inequality for the first inequality, and the $\bL^2$-exponential convergence of $P_t$ for the second (recall that the spectral gap and first non-zero eigenvalue of $-\L$ satisfy $C=1/\lambda_1$).
% \end{proof}

Next, we turn to the long-term behavior of the Markov chain $(\theta(t_i))_{i=0}^N$.
As a particular application, we focus on the empirical sums
\begin{align*}
\mathbf{H}_{j,k}=\langle z_j,\psi_k \rangle_{\text{data}}:=\frac{1}{N+1}\sum_{i=0}^N z_j(t_i)\psi_k(t_i)
\end{align*}
which define the lifting operator $\mathbf{H}$. In light of Remark \ref{rmk:scaling}, we will assume that $\psi_k$ has been rescaled so that $(N+1)^{-1}\sum_{i=0}^N\psi_k(t_i)=\E_\mu[\varphi_k(\theta)]$. Recall that $z_j(t_i)=F^j(\theta(t_i))$ and we have assumed mean-centering: $\E_\mu[F^j]=0$. Setting $a_{jk}=\langle F^j,\varphi_k\rangle_{\mu}$,
the triangle inequality gives
\begin{align*}
\lb|\mathbf{H}_{j,k}-a_{j,k}\rb|\le \frac{1}{N+1}\lb|\sum_{i=0}^N F^j(\theta(t_i))(\psi_k(t_i)-\varphi_k(\theta(t_i)))\rb|+\frac{1}{N+1}\lb|\sum_{i=0}^N \lb\{F^j(\theta(t_i))\varphi_k(\theta(t_i))-a_{jk}\rb\}\rb|.
\end{align*}
The first sum concerns the approximation of the eigenfunctions of $\L$ by diffusion coordinates, which is addressed in the next subsection. 
Focusing on the second term, we shall apply the Kipnis--Varadhan Central Limit Theorem \citet{kipnis_central_1986}. We state the result with $\theta_0\sim\mu$; however, under suitable assumptions on $F^j\varphi_k$ similar results hold even when the Markov chain is started from a deterministic initial condition (see, for instance, \citet{derriennic_central_2001,derriennic_central_2003,cuny_central_2012,cattiaux_central_2012}).

\begin{proposition}[\citet{kipnis_central_1986}]\label{prop:clt}
Suppose $\theta_0\sim\mu$ and there exists $g$ such that $F^j\varphi_k=\sqrt{I-P_{\Delta t}}g$. Then the weak convergence
\begin{align*}
\frac{1}{\sqrt{N}}\sum_{i=0}^N \lb\{F^j(\theta(t_i))\varphi_k(\theta(t_i))-a_{jk}\rb\}\implies \calN\lb(0,\sigma^2\rb)\text{ as }N\to\infty
\end{align*}
holds, where $\sigma^2$ satisfies
\begin{align*}
\sigma^2=\lim_{N\to\infty}\frac{1}{N}\Var_{\mu}\lb(\sum_{i=0}^N F^j(\theta(t_i))\varphi_k(\theta(t_i))\rb)=\int_{-1}^1 \frac{1+\lambda}{1-\lambda}\rho_{F^j\varphi_k}(\d \lambda)<\infty
\end{align*}
and $\rho_{F^j\varphi_k}$ is the spectral measure of $F^j\varphi_k$.
\end{proposition}

In the context of Example \ref{ex:example}, $F^j\varphi_k$ is a polynomial and $\sigma^2$ can be computed using knowledge of the Hermite eigenfunctions.

\subsection{Error Analysis of the Graph Laplacian}\label{sec:concentration}

The ADM approach of \citet{singer2008nonlinear} relies on the approximation of the Fokker--Planck operator by the graph Laplacian. Suppose $(\theta(t_i))_{i=0}^N$ are i.i.d.~samples from $\mu$, let $f\in \bL^2(\mu)$, and set $z(t_i)=H(\theta_i)$. Then an analogous argument to \eqref{eq:quotienttaylor} (see \citet[Equation (30)]{singer2008nonlinear} along with extra accounting for a factor of 2 from differences in some definitions given there) yields
\begin{align*}
\sum_{j=0}^N\mathbf{P}_{ij} f(z(t_j))\to \frac{\E_\mu \lb[\exp\lb(-\frac{\|\theta-\theta(t_i)\|^2}{2\veps}\rb)f\circ H(\theta)\rb]}{\E_\mu \lb[\exp\lb(-\frac{\|\theta-\theta(t_i)\|^2}{2\veps}\rb)\rb]}= f\circ H(\theta(t_i)) +\veps \L (f\circ H)(\theta(t_i))+O(\veps^2)
\end{align*}
as $N\to\infty$, where $\mathbf{P}$ is the stochastic matrix described in Subsection \ref{sec:embedding}.
This suggests that an eigenvector of $\mathbf{P}-I$ is close to an eigenfunction of $\L$ on the support of the samples $(\theta(t_i))_{i=0}^N$.

Our aim in this section is to generalize this approximation for the case when $(\theta(t_i))_{i=0}^N$ is a time series corresponding to a sampled Langevin diffusion, and to quantify the error using a concentration inequality.
We begin by fixing some notation. Fix $f\in \mathbb{L}^2(\mu)$, $w\in\R^d$, and let 
\begin{align}\label{eq:uvdef}
u(\theta)=\exp\lb(-\frac{\|\theta-w\|^2}{2\veps}\rb)f(\theta) \quad\text{and}\quad
v(\theta)=\exp\lb(-\frac{\|\theta-w\|^2}{2\veps}\rb).
\end{align}
Setting $w=\theta(t_i)$, we have that
\begin{align*}
\sum_{j=0}^N\mathbf{P}_{ij} f(z(t_j))=\frac{\sum_{j=0}^N u(\theta(t_j))}{\sum_{j=0}^Nv(\theta(t_j))}
\end{align*}
The following theorem demonstrates concentration of the right hand side around $\E_{\mu} [u]/\E_{\mu} [v]$.

\begin{theorem}\label{thm:concentration}
Suppose that Assumptions \ref{ass1}, \ref{ass2}, and \ref{ass3} hold. Let $f\in \mathbb{L}^2(\mu)$, fix $w\in \R^d$, and define $u,v$ by \eqref{eq:uvdef}. Then
\begin{align*}
\bP_{\theta_0\sim \mu}\lb(\lb|\frac{\sum_{j=0}^N u(\theta(t_j))}{\sum_{j=0}^Nv(\theta(t_j))} -\frac{\E_{\mu} [u]}{\E_{\mu} [v]}\rb|>\alpha \rb)\le 2\exp\lb(\frac{-(1-\exp(-\lambda_1 \Delta t))(N+1)p(w) \alpha^2}{10 C \alpha p(w)+2^{-d+2} \pi^{-d/2} \veps^{-d/2}\lb(\veps\|\nabla f(w)\|^2+O(\veps^2)\rb)}\rb)
\end{align*}
where $C$ depends on $\|u\|_\infty$.
\end{theorem}

\begin{remark}
If $\theta$ is initialized at $\mu_{-r}$ with $\|\frac{\d\mu_{-r}}{\d\mu}-1\|_{\bL^2(\mu)}<\infty$, then Lemma \ref{lem:burnin} may be applied after discarding a burn-in period of length $r$. This yields an extra factor of $1+\exp\lb(-\lambda_1 r\rb)\|\frac{\d\mu_{-r}}{\d\mu}-1\|_{\bL^2(\mu)}$ on the right hand side of the above inequality.
\end{remark}

\begin{proof}
We adapt the argument from \citet[Section 3]{singer_graph_2006} to samples $(\theta(t_i))_{i=0}^N$ from the Langevin diffusion \eqref{eq:thetaSDE}. In this case, $(\theta(t_i))_{i=0}^N$ forms a reversible Markov chain with spectral gap $1-\exp(-\lambda_1 \Delta t)$. We will show an upper bound for the right tail in the desired inequality. The left tail follows similarly and the statement of the theorem holds by a union bound. If we let
\begin{align*}
V(\theta(t_j))=\frac{\E_{\mu} [v]u(\theta(t_j))-\E_{\mu} [u]v(\theta(t_j))+\alpha \E_{\mu} [v](\E_{\mu} [v]-v(\theta(t_j)))}{\E_{\mu} [v]^2}.
\end{align*}
then we see that
\begin{align*}
\bP_{\theta_0\sim \mu}\lb(\frac{\sum_{j=0}^N u(\theta(t_j))}{\sum_{j=0}^Nv(\theta(t_j))}  -\frac{\E_{\mu} [u]}{\E_{\mu} [v]}>\alpha \rb)=\bP_{\theta_0\sim \mu}\lb(\sum_{j=0}^N V(\theta(t_j))>\alpha(N+1)\rb)
\end{align*}
Note that $\E_{\mu}[V]=0$ and $V$ is bounded in absolute value by some constant $C$ which depends on $\|u\|_\infty$. We seek to apply a concentration inequality which will capture the second moment of $V$ in the variance proxy. Since $(\theta(t_i))_{i=0}^N$ forms a reversible Markov chain with respect to $\mu$ and spectral gap $1-\exp(-\lambda_1 \Delta t)$ (where $\lambda_1$ is the spectral gap of $-\L$ given by the Poincaré inequality in Assumption \ref{ass3}), we can apply Bernstein's inequality for Markov chains \citep[Theorem 3.3]{paulin_concentration_2015}. This inequality requires bounding the second moment $\mathbb{E}_\mu[V^2]$, which serves as a variance proxy in the non-i.i.d. setting. The inequality  gives
\begin{align*}
\bP_{\theta_0\sim \mu}\lb(\lb|\frac{\sum_{j=0}^N u(\theta(t_j))}{\sum_{j=0}^Nv(\theta(t_j))} -\frac{\E_{\mu} [u]}{\E_{\mu} [v]}\rb|>\alpha \rb)\le 2\exp\lb(\frac{-(1-\exp(-\lambda_1 \Delta t))(N+1) \alpha^2}{4\Var_{\mu}(V)+10 C \alpha}\rb)
\end{align*}
A lengthy computation (which is relegated to Appendix \ref{sec:appc}) shows that
\begin{align*}
\Var_{\mu}(V)=\frac{2^{-d} \pi^{-d/2} \veps^{-d/2}\lb(\veps\|\nabla f(w)\|^2+O(\veps^2)\rb)}{p(w)}
\end{align*}
and the conclusion of the theorem follows. \qedhere
\end{proof}

\subsection{Approximating the Eigenfunctions by Independent Linear Equations}\label{sec:linearapprox}

Let $\varphi=(\varphi_1,\dots,\varphi_\ell)$ be a vector of $\L$-eigenfunctions with eigenvalues $0> -\lambda_1\ge \dots \ge -\lambda_\ell$ such that ${\L \varphi_i=-\lambda_i \varphi_i}$ for $i= 1,\dots,\ell$.
In Subsection \ref{sec:embedding}, we saw that
\begin{equation}\label{eq:varphiSDE}
\begin{aligned}
\d\varphi_i(\theta_t)&=-\nabla U(\theta_t)\cdot \nabla \varphi_i(\theta_t)\d t+\Delta\varphi_i(\theta_t) \d t+\sqrt{2}\nabla \varphi_i(\theta_t) \cdot \d W_t\\
&=\L \varphi_i (\theta_t)\d t+\sqrt{2}\|\nabla \varphi_i(\theta_t)\|_2 \d B^i_t\\
&=-\lambda_i \varphi_i(\theta_t)\d t+\sqrt{2}\|\nabla \varphi_i(\theta_t)\|_2 \d B^i_t,
\end{aligned}
\end{equation}
where $B^i$ is a Brownian motion,
and for $i\neq j$ we have
\begin{align*}
\d\langle\varphi_i(\theta_\cdot),\varphi_j(\theta_\cdot)\rangle _t&=2\nabla \varphi_i(\theta_t)\cdot\nabla\varphi_j(\theta_t) \d t.
\end{align*}
If $\theta_0$ is initialized according to the invariant measure $\mu$, then
\begin{equation}\label{eq:uncorr}
\begin{aligned}
\E_{\theta_0\sim\mu}\lb[\langle\varphi_i(\theta_\cdot),\varphi_j(\theta_\cdot)\rangle _t\rb]&=\E_{\theta_0\sim\mu}\lb[\int_0^t 2\nabla \varphi_i(\theta_t)\cdot\nabla\varphi_j(\theta_t) \d t\rb]\\
&=2t\int_{\R^d} \Gamma(\varphi_i(z),\varphi_j(z))\mu(\d z)\\
&=-2t\int_{\R^d} \varphi_i(z)\L \varphi_j(z)\mu(\d z)\\
&=2\lambda_jt\int_{\R^d} \varphi_i(z) \varphi_j(z)\mu(\d z)=0
\end{aligned}
\end{equation}
where the third equality holds by identity \eqref{eq:dirichlet}, the fourth uses the fact that $\varphi_j$ is an eigenfunction, and the final equality holds by the orthogonality of eigenfunctions.
If $\theta_0$ is started out of equilibrium, then by the argument of Lemma \ref{lem:burnin}, it holds that $\varphi_i(\theta_t)$ and $\varphi_j(\theta_t)$ are close to being uncorrelated after an initial burn-in period:
\begin{corollary}\label{coro:uncorr}
Suppose $\theta_0\sim \mu_0$ where $\|\frac{\d\mu_0}{\d\mu}-1\|_{\bL^2(\mu)}<\infty$. Then, for $i,j\in\{1,\dots,\ell\}$ with $i\neq j$ and $0\le s\le t$ we have that
\begin{align*}
&\lb|\E_{\theta_0\sim\mu_0}\lb[\langle\varphi_i(\theta_\cdot),\varphi_j(\theta_\cdot)\rangle _t-\langle\varphi_i(\theta_\cdot),\varphi_j(\theta_\cdot)\rangle _s \rb]\rb|\\
&\quad\le \lambda_1^{-1}\lb(\exp(-\lambda_1 s)-\exp(-\lambda_1 t)\rb)\lb\|\frac{\d\mu_0}{\d\mu}-1\rb\|_{\bL^2(\mu)}\lb\|\nabla \varphi_i\cdot\nabla\varphi_j\rb\|_{\bL^2(\mu)}.
\end{align*}
\end{corollary}
\begin{proof}
For $u\ge 0 $, let $\mu_u$ denote the law of $\theta_u$. We have that
\begin{align*}
&\lb|\E_{\theta_0\sim\mu_0}\lb[\langle\varphi_i(\theta_\cdot),\varphi_j(\theta_\cdot)\rangle _t-\langle\varphi_i(\theta_\cdot),\varphi_j(\theta_\cdot)\rangle _s \rb]\rb|\\
&\quad = \lb|2\int_s^t\int_{\R^d}\nabla \varphi_i(w)\cdot\nabla\varphi_j(w)\lb(1+\frac{\d\mu_u}{\d\mu}(w)-1\rb)\mu(\d w)\d u\rb|\\
&\quad \le 2\int_s^t\lb|\int_{\R^d}\nabla \varphi_i(w)\cdot\nabla\varphi_j(w)\mu(\d w)\rb|\d u+2\int_s^t\lb|\int_{\R^d}\nabla \varphi_i(w)\cdot\nabla\varphi_j(w)\lb(\frac{\d\mu_u}{\d\mu}(w)-1\rb)\mu(\d w)\rb|\d u
\end{align*}
The first term is null by \eqref{eq:uncorr}.
As seen previously, $\|\frac{\d\mu_u}{\d\mu}-1\|_{\bL^2(\mu)}\le\exp\lb(-\lambda_1 u\rb)\|\frac{\d\mu_0}{\d\mu}-1\|_{\bL^2(\mu)}$. Applying this along with H\"{o}lder's Inequality to the second term above and integrating in time gives the desired inequality. \qedhere
\end{proof}

In \citet{shnitzer2020diffusion}, the gradients $\|\nabla \varphi_i(\theta_t)\|_2$, for $i=1,\dots,\ell$ are assumed to be constant, so that the SDE for $\varphi_i(\theta_t)$ is linear-Gaussian and a Kalman filtering framework may be applied. As another corollary of Lemma \ref{lem:burnin}, we show next that the approximation error between the contractive dynamics \eqref{eq:varphiSDE} and a linear SDE can remain small. For $i= 1,\dots,\ell$, let $\xi^1,\dots,\xi^\ell$ satisfy
\begin{align*}
\d \xi^i_t=-\lambda_i \xi^i_t \d t+\sqrt{2}\gamma_i\d B_t^i
\end{align*}
where $\gamma_1,\dots,\gamma_\ell>0$ are fixed constants, $B^1,\dots,B^\ell$ are the same Brownian motions which drive $\varphi^1(\theta_t),\dots,\varphi^\ell(\theta_t)$, and we fix a deterministic initial condition $(\xi^1_0,\dots,\xi^\ell_0)\in\R^\ell$.
\begin{corollary}
Suppose $\theta_0\sim \mu_0$ where $\|\frac{\d\mu_0}{\d\mu}-1\|_{\bL^2(\mu)}<\infty$.
For $i\in\{1,\dots,\ell\}$ the following holds:
\begin{equation}\label{eq:squarederror}
\begin{aligned}
&\E_{\theta_0\sim\mu_0}\lb[(\varphi_i(\theta_t)-\xi_t^i)^2\rb]\\
&\quad\le \exp(-2\lambda_i t)\lb(\E_{\theta_0\sim\mu_0}\lb[(\varphi_i(\theta_0)-\xi_0^i)^2\rb]+2t\E_{\mu}\lb[\lb( \|\nabla \varphi_i\|_2-\gamma_i\rb)^2\rb]+\alpha(t)\rb)
\end{aligned}
\end{equation}
where
\begin{align*}
\alpha(t)=\lambda_1^{-1}\lb(1-\exp(-\lambda_1 t)\rb)\lb\|\frac{\d\mu_0}{\d\mu}-1\rb\|_{\bL^2(\mu)}\E_{\mu}\lb[\lb( \|\nabla \varphi_i\|_2-\gamma_i\rb)^{4}\rb]^{1/2}.
\end{align*}
\end{corollary}
\begin{remark}
Setting $\gamma_i=\E_\mu[\|\nabla \varphi_i\|_2]$ minimizes $\E_{\mu}\lb[\lb( \|\nabla \varphi_i\|_2-\gamma_i\rb)^2\rb]$ yielding $\Var_{\mu}\lb( \|\nabla \varphi_i\|_2\rb)$.
\end{remark}

\begin{proof}
By It\^{o}'s formula, the squared error at time $t\ge 0$ satisfies
\begin{align*}
(\varphi_i(\theta_t)-\xi^i_{t})^2&=(\varphi_i(\theta_0)-\xi^i_{0})^2+2\int_0^{t} (\varphi_i(\theta_{u})-\xi^i_{u})(\d \varphi_i(\theta_{u})-\d\xi^i_{u}))+\int_0^t\d\lb\langle \varphi_i(\theta_{\cdot})-\xi^i_\cdot\rb\rangle_u\\
&= (\varphi_i(\theta_0)-\xi^i_{0})^2 -2\lambda_1 \int_0^t (\varphi_i(\theta_u)-\xi^i_{u})^2\d u+M_t+2\int_0^t\lb( \|\nabla \varphi_i(\theta_{u})\|_2-\gamma_i\rb)^2\d u,
\end{align*}
where $M_t=2\int_0^{t} (\varphi_i(\theta_{u})-\xi^i_{u})(\|\nabla\varphi_i(\theta_{u})\|_2-\gamma_i)\d B_u^i$ is a local martingale term, which we assume to be a true martingale (this can be verified depending on the dynamics of $\theta$, for instance, it holds in the case of Example \ref{ex:example} where $\theta$ is an OU process and $\varphi_i$ is a Hermite polynomial). Taking expectations and applying Fubini's Theorem yields
\begin{align}\label{eq:squaredprelim}
\E\lb[(\varphi_i(\theta_t)-\xi^i_{t})^2\rb]
= (\varphi_i(\theta_0)-\xi^i_{0})^2 -2\lambda_i \int_0^t \E\lb[(\varphi_i(\theta_u)-\xi^i_{u})^2\rb]\d u+2\int_0^t\E\lb[\lb( \|\nabla \varphi_i(\theta_{u})\|_2-\gamma_i\rb)^2\rb]\d u.
\end{align}
For $u\ge0$, let $\mu_u$ denote the law of $\theta_u$. By Lemma \ref{lem:burnin}, we have
\begin{align*}
&\E_{\theta_0\sim\mu_0}\lb[\lb( \|\nabla \varphi_i(\theta_{u})\|_2-\gamma_i\rb)^2\rb]\\
&\quad\le \E_{\mu}\lb[\lb( \|\nabla \varphi_i\|_2-\gamma_i\rb)^2\rb] +\exp(-\lambda_1 u)\lb\|\frac{\d\mu_0}{\d\mu}-1\rb\|_{\bL^2(\mu)}\E_{\mu}\lb[\lb( \|\nabla \varphi_i\|_2-\gamma_i\rb)^{4}\rb]^{1/2}
\end{align*}
The final result follows from using the above inequality in \eqref{eq:squaredprelim} and applying the integral form of Gr\"{o}nwall's Inequality. \qedhere
\end{proof}

% Methodology
\section{Joint Diffusion Kalman Filter}
\label{sec:jdkf_model}

In Section~\ref{sec:results}, we have developed the main results to justify a state space model in terms of the diffusion coordinate embeddings: 1) the generator eigenfunction dynamics can be approximated by a linear SDE with drift equal to the corresponding eigenvalue (Subsection \ref{sec:linearapprox}); 2) the diffusion coordinates $\psi$ approximate the eigenfunctions of the generator associated with the latent state $\theta$ (Subsection \ref{sec:concentration}); 3) the lifting operator learned from time series data converges to the true $\mathbb{L}^2(\mu)$ projection of the observation map onto the generator eigenfunctions (Subsection \ref{sec:lifts}). In this section, we describe our linear state space approximation for the dynamics of the diffusion coordinates, leveraging these theoretical results developed in Section \ref{sec:results}.

Assuming that the gradients in \eqref{eq:phiSDE} are constant, \citet{shnitzer2020diffusion}
define a state-space model, dubbed the Diffusion Kalman Filter (DKF), for the case where the data arises from a single data set $(x(t_i))_{i=0}^N$ (see Appendix \ref{sec:appa}).
We now modify DKF for a supervised learning framework to incorporate information in the series of the dependent variables $y(t_i)$, which depend jointly with $x(t_i)$ on $\theta(t_i)$. If one were to use the estimated diffusion coordinates from DKF as developed in \citet{shnitzer2020diffusion}, a simple OLS regression of the dependent variables on these coordinates will show that most of the explanatory power of the original covariates $x(t_i)$ for the dependent variables is lost (see, for instance, Appendix \ref{sec:appb} for the regression of individual stock returns on a high-dimensional set of macroeconomic covariates). Thus, it is important to extend DKF to also account for the variation in the dependent variable in order for the diffusion coordinates to preserve explanatory power. We call this approach the Joint Diffusion Kalman Filter (JDKF).

%%%
Our framework relies on constructing dynamic diffusion coordinate embeddings of the observations that serve as proxies for the dynamics of $\theta$, combining state space modeling with diffusion maps. Commonly, dimensionality reduction methods, when applied to time series data, do not explicitly consider the temporal component in the data. This can be limiting in many applications, for instance, when dealing with high-dimensional financial market data where a portfolio manager (PM) might want to study how shocks to a given market factor propagate through time. By integrating Kalman filtering techniques with diffusion maps and extending the framework first introduced in \citet{shnitzer2020diffusion}, we construct dynamic diffusion coordinate embeddings which explicitly factor in the temporal component of the data. We construct our state space equations by leveraging a linear lifting operator that takes diffusion map embeddings back into the original measurement space.

Since we approximate the intrinsic dynamics with a linear SDE in terms of the diffusion coordinate embeddings (which we rigorously justify in Section \ref{sec:linearapprox}), our framework does not require us to discover the explicit joint dynamics of the observations $(x(t), y(t))$. Leveraging linear state-space equations, we devise a procedure to sample dynamic embeddings while conditioning on an information set defined in the original measurement space. 
It is well known that dimensionality reduction techniques give rise to the challenge of interpretability of the embeddings. In a framework such as ours that takes place completely in the embedding space, it is hence desirable to define fully interpretable conditioning events in the original measurement space. To define such a sampling procedure, we leverage the power of the linear lifting operator in a three step process: (1) lift statistics from the embedding space to the original measurement space, (2) condition on the desired information set in the  covariate space and obtain the conditional distribution in terms of $(x(t), y(t))$, (3) restrict back to the embedding space with a pseudoinverse of the lifting operator. Our sampling procedure simultaneously learns the intrinsic geometry and dynamics from the data, and then \textit{generates} sample paths for the diffusion coordinate embeddings and by extension of $(x(t), y(t))$. Of course, it is known that dimensionality reduction techniques do not necessarily preserve statistics in the data, hence approximating statistics in the measurement space with their lifted counterparts incurs some error.
Our robustness results presented in Section \ref{sec:results} provide theoretical support for the applicability of our method.
%%%

%%%%

To formalize the Joint Diffusion Kalman Filter described above, we now specialize to the supervised learning framework and adopt the dynamical system setting from Subsection~\ref{sec:assumps}:

\begin{equation*}
\begin{aligned}
d \theta_t&= -\nabla U(\theta_t)+\sqrt{2}dW_t\in \R^d,\\
X_t&=F(\theta_t)\in \R^m,\text{ and }Y_t=G(\theta_t)\in \R^n.
\end{aligned}
\end{equation*}
Following the embedding procedure in Subsection \ref{sec:embedding}, the diffusion coordinates $\psi(t_i)$ are computed jointly on samples of $(X,Y)$, denoted $(x(t_i),y(t_i))_{i=0}^N$, and $x$ and $y$ can be similarly reconstructed by the linear lifts:
\begin{align*}
x(t_i)&\approx \mathbf{H}^x \psi(t_i)\\
y(t_i)&\approx \mathbf{H}^y \psi(t_i)
\end{align*}
where $\mathbf{H}_{i,j}^x=\langle x_i, \psi_j \rangle_{\text{data}}$ and $\mathbf{H}_{i,j}^y=\langle y_i, \psi_j \rangle_{\text{data}}$. Let $\Lambda$ denote the transition matrix for the linear dynamics containing the mapped eigenvalues ${-\veps^{-1}\log \kappa_1,\dots -\veps^{-1}\log \kappa_\ell}$ on its diagonal.

Going forward, we denote variables in our linear state space using lower case letters and subscripts taking values in $\Z_{\ge0}$ (e.g.~$x_t$ and $\psi_t$). As with DKF, we consider a state transition equation where $\psi_t$ evolves over time according to the linear dynamics
\begin{align}\label{eq:psistep}
\psi_{t+1} = \mathbf{A} \psi_t + w_t, \quad w_t \sim \mathcal{N}(0, \mathbf{Q})
\end{align}
where $\psi_t \in \mathbb{R}^{\ell}$ is the latent state vector at time $t$; $\mathbf{A}=I-\Lambda \in \mathbb{R}^{\ell \times \ell}$ is the state transition matrix; and $\mathbf{Q} \in \mathbb{R}^{\ell \times \ell}$ is the process noise covariance.
We observe two variables, $x_t$ (the discrete time covariate samples) and $y_t$ (the discrete time response samples) and postulate the following two measurement state-space equations:
\begin{equation}\label{eq:xystep}
\begin{split}
 x_t = \mathbf{H}^x \mathbf{\psi}_t +{v}_{x,t}\\
 y_t = \mathbf{H}^y \mathbf{\psi}_t +{v}_{y,t}
\end{split}
\end{equation}
where $v_{x,t}\sim \mathcal{N}(0, {\mathbf{R}_x})$ and $v_{y,t}\sim \mathcal{N}(0, {\mathbf R}_y)$ with cross-covariance ${\mathbf R}_{xy}$ model noise due to reconstruction error.
We estimate $\mathbf{R}_x$ and $\mathbf{R}_y$ and construct estimates $\hat{\psi}_{t+1} = \mathbb{E} [\psi_{t+1} | \left\{x_i,y_i \right\}_{i=0}^{t}]$ using the EM algorithm described below: $\mathbf{A}=I-\Lambda$ is computed from applying diffusion maps jointly on $(x(t_i),y(t_i))_{i=0}^N$, and $\mathbf{Q}$ is computed as the empirical covariance matrix of the resulting diffusion coordinates. 

An important special case is the linear multi-factor model, with $G=\mathbf{B} \circ F$ where $\mathbf{B}$ is a matrix of regression coefficients (also known as factor loadings). In this case, we compute $\Lambda$ and $\mathbf{H}^x$ using $(x(t_i))_{i=0}^N$ and postulate:
\begin{align*}
Y_t&=G(\theta_t)=\mathbf{B} \circ F(\theta_t)=\mathbf{B}X_t,\\
y_t&=\mathbf{B}\mathbf{H}^x \psi_t+v_{y,t}.
\end{align*}
Note that $v_{y,t}$ here is a noise \textit{correlated} with $v_{x,t}$ through $\mathbf{B}$. Concretely, for a \textit{linear factor model} of the form $y_t=\mathbf{B}x_t+\veps_{y,t}$, where $\veps_{y,t}$ represents Gaussian idiosyncratic noise, we can cast the state space equation for $y_t$ in terms of $\psi_t$ as follows:
\[y_t=\mathbf{B}x_t+\veps_{y,t}
=\mathbf{B}(\mathbf{H}^x\psi_t+v_{x,t})+\veps_{y,t}
={\mathbf B} {\mathbf H}^x \psi_t + (\underbrace{\mathbf{B}v_{x,t}+\veps_{y,t}}_{v_{y,t}})\]
In this linear case, we additionally aim to estimate $\mathbf{B}$ and we have the following state-space representation in terms of $\psi_t$:

 \begin{tcolorbox}[colback=blue!5!white,colframe=blue!75!black,title=Special Case: Linear Multi-Factor Model]
 State-Space Equations:
\begin{equation*}
\begin{aligned}
x_t&={\bf H}^x \psi_t + v_{x,t}\\
y_t &= {\mathbf B} {\mathbf H}^x \psi_t + v_{y,t}\\
\psi_{t+1} &= \mathbf{A} \psi_t + w_t
\end{aligned}
\end{equation*}
$\mathbf{A}=I-\Lambda$ is computed using anisotropic diffusion maps on covariates $(x(t_i))_{i=0}^N$. We estimate $\mathbf{B}, \mathbf{R}$, and estimate $\hat{\psi}_{t+1} = \mathbb{E} [\psi_{t+1} | \left\{x_1,\ldots,x_{t} \right\}]$ using the EM algorithm. 
\end{tcolorbox}
 
We now aim to estimate the unknown parameters in the state-space using the standard EM algorithm. The procedure involves alternating between:
\begin{enumerate}
    \item E-step: Run the Kalman filter to estimate the latent state $\psi_t$.
    \item M-step: Update the parameters $\mathbf{R}_x$, $\mathbf{R}_y$, and the cross-covariance $\mathbf{R}_{xy}$ (and $\mathbf{B}$ in the linear case) based on the current estimates of $\psi_t$. 
\end{enumerate}

\subsection{E-Step: Kalman Filter with Joint Observations}

The E-step estimates the latent diffusion coordinates $\psi_t$ and their covariance $P_{t|t}$ using a Kalman filter that processes joint observations from both covariates and responses. We define the joint observation vector:
\[z_t = \begin{bmatrix} x_t \\ y_t \end{bmatrix} = \begin{bmatrix} \mathbf{H}^x \\ \mathbf{H}^y \end{bmatrix} \psi_t + \begin{bmatrix} v_{x,t} \\ v_{y,t} \end{bmatrix} = \mathbf{H}^z \psi_t + v_t\]
where $\mathbf{H}^z = [\mathbf{H}^x; \mathbf{H}^y]$ and the joint noise covariance is:
\[\mathbf{R} = \begin{bmatrix} \mathbf{R}_x & \mathbf{R}_{xy} \\ \mathbf{R}^\top_{xy} & \mathbf{R}_y \end{bmatrix}.\]

The crucial observation is that by filtering on the stacked observation vector $z_t$, the Kalman gain optimally weights information from both covariates and responses when updating the latent state estimate. This contrasts with DKF \citep{shnitzer2020diffusion}, which processes only a single observation stream and thus may lose explanatory power for the responses.

Running the forward Kalman filter yields filtered estimates $\hat{\psi}_{t|t}$ and covariances $P_{t|t}$ for $t = 1, \ldots, T$. We then apply the Rauch--Tung--Striebel (RTS) smoother \citet{rauch_tung_striebel} in a backward pass to obtain smoothed estimates $\hat{\psi}_{t|T}$ and $P_{t|T}$ that incorporate information from the entire observation sequence. The standard Kalman filtering and RTS smoothing equations are provided in Appendix~\ref{app:kalman_em_details} for completeness.

In the \textit{linear multi-factor case} where $\mathbf{H}^y = \mathbf{B}\mathbf{H}^x$, this formulation naturally accommodates the correlation structure induced by the factor loadings $\mathbf{B}$.

\subsection{M-Step: Parameter Updates}

The M-step updates the observation noise covariances based on the smoothed state estimates from the E-step. Given $\hat{\psi}_{t|T}$ for $t = 1, \ldots, T$, we update:

\begin{itemize}
\item $\mathbf{R}_x$: covariance of residuals $x_t - \mathbf{H}^x \hat{\psi}_{t|T}$
\item $\mathbf{R}_y$: covariance of residuals $y_t - \mathbf{H}^y \hat{\psi}_{t|T}$  
\item $\mathbf{R}_{xy}$: cross-covariance of these residuals
\end{itemize}

In the linear multi-factor case, we additionally estimate the factor loading matrix $\mathbf{B}$ via least squares regression on the lifted coordinates:
\[\mathbf{B} = \left( \sum_{t=1}^{T} y_t (\mathbf{H}^x \hat{\psi}_{t|T})^\top \right) \left( \sum_{t=1}^T \mathbf{H}^x \hat{\psi}_{t|T} (\mathbf{H}^x \hat{\psi}_{t|T})^\top \right)^{-1}.\]

This update is consistent with the observation equation $y_t = \mathbf{B}\mathbf{H}^x\psi_t + v_{y,t}$ and ensures that the estimated loadings reflect the relationship between responses and the lifted covariate embeddings.

The complete parameter update formulas are standard and provided in Appendix~\ref{app:kalman_em_details}. The key distinction from classical EM for state-space models is that our observation matrices $\mathbf{H}^x, \mathbf{H}^y$ are \textit{learned from data} via diffusion maps rather than specified a priori, and the transition matrix $\mathbf{A} = I - \Lambda$ is determined by the diffusion map eigenvalues rather than estimated from the data.

\subsection{EM Algorithm}

\begin{algorithm}[H]
\caption{EM Algorithm Iteration}
\begin{algorithmic}[1]

\State \textbf{Initialize} the parameters $\mathbf{R}_x$, $\mathbf{R}_y$, and $\mathbf{R}_{xy}$ (and $\mathbf{B}$ in the linear case).
\State \textbf{Set} tolerance $\veps$ and maximum iterations $N_{\text{max}}$.
\State \textbf{Initialize} log-likelihood $\mathcal{L}^{(0)} \leftarrow -\infty$, iteration counter $k \leftarrow 0$.

\While{not converged and $k < N_{\text{max}}$}
    \State \textbf{E-step}: Run Joint Diffusion Kalman filter using current estimates of $\mathbf{R}_x$, $\mathbf{R}_y$ and $\mathbf{R}_{xy}$.
    \State \quad \textbf{Output}: Estimate $\psi_t$ and compute the log-likelihood $\mathcal{L}^{(k+1)}$.
    
    \State \textbf{M-step}: Update $\mathbf{R}_x$, $\mathbf{R}_y$, and $\mathbf{R}_{xy}$ using the smoothed estimates of $\psi_t$.
    
    \State \textbf{Convergence check}: 
    \If{ $\frac{|\mathcal{L}^{(k+1)} - \mathcal{L}^{(k)}|}{|\mathcal{L}^{(k)}|} < \veps$}
        \State \textbf{Converged:} Stop the algorithm.
    \Else
        \State $k \leftarrow k + 1$
    \EndIf
\EndWhile

\end{algorithmic}
\label{alg:em}
\end{algorithm}
 We have added an extra layer with the measurement equation for the dependent variables to maintain the explanatory power in the embeddings and taken this into account via the modified state-space. JDKF thus gives us diffusion coordinates that propagate linearly and preserve much of the original explanatory power from the full high-dimensional set of covariates. Algorithm \ref{alg:em} summarizes the EM procedure. Figure \ref{fig:model_flow} displays a broad flow map of the JDKF procedure. In the next subsection, we describe the conditional sampling procedure. 

 \begin{figure}
  \centering
  {\includegraphics[width=0.8\linewidth]{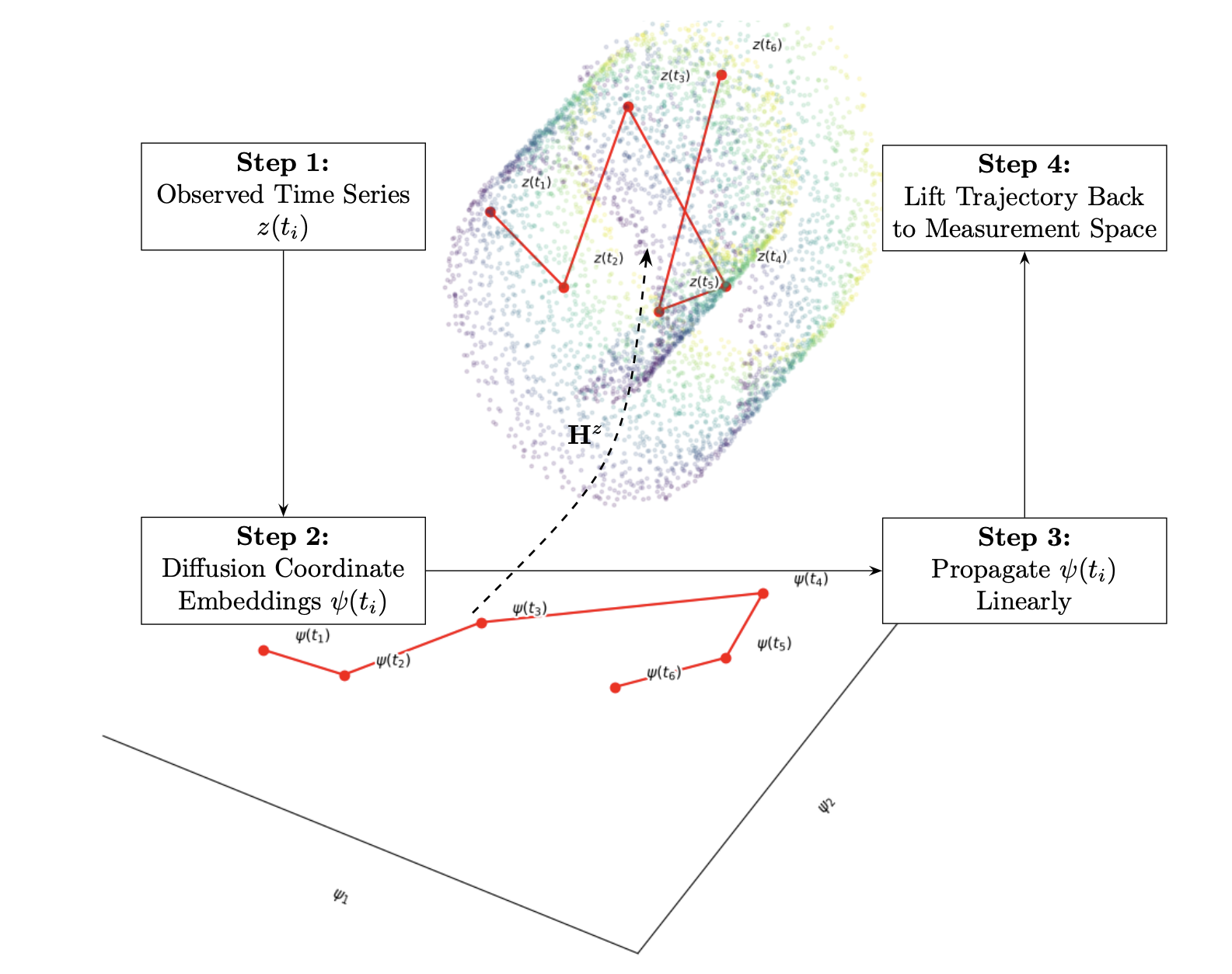}}
  \hfill
  \caption{A flow map of the Joint Diffusion Kalman Filter (JDKF) procedure. Starting from observations $z(t_i)$, we apply anisotropic diffusion maps to obtain embeddings $\psi(t_i)$ that capture the manifold geometry. We propagate a linear approximation of the diffusion coordinates forward in time. We predict coordinates using a Kalman filter and lift back to the measurement space using the linear operator $\mathbf{H}$.}
  \label{fig:model_flow}
\end{figure}

\subsection{Sampling $\psi_{t+1}$ Conditional on Information from $(x_{t+1}, y_{t+1})$}
\label{sec:sample}
When implementing the Joint Diffusion Kalman Filter, we seek to sample diffusion coordinates conditional on information in the original measurement spaces of $x_{t+1}$ and $y_{t+1}$, where this information is typically that a subset of coordinates of $x_{t+1}$ (and/or $y_{t+1}$) are fixed to some pre-determined values. To do this, we can leverage the power of the linear lifting operator to go back and forth between the measurement and embedding spaces.
Given the setup in equations \eqref{eq:psistep} and \eqref{eq:xystep}, we desire to sample $\hat{\psi}_{t+1}$ given $\hat{\psi}_t$ while conditioning on a subset of coordinates of the reconstructed $(x_{t+1}, y_{t+1})$, or equivalently $(\mathbf{H}^x \hat{\psi}_{t+1}, \mathbf{H}^y \hat{\psi}_{t+1})$ being fixed to specific values. Let $\psi_t$ be the latent state at time $t$. Recall the state transition model is given by:
\[\psi_{t+1} = \mathbf{A} \psi_t + w_t\]
where \(\mathbf{A}\) is the state transition matrix and $w(t) \sim \mathcal{N}(0, \mathbf{Q})$ is the process noise with covariance $\mathbf{Q}$. The filtered observation estimates are given from the observation equation by:
\begin{align*}
\hat{x}_{t+1} &= \mathbf{H}^x \hat{\psi}_{t+1}\\
\hat{y}_{t+1} &= \mathbf{H}^y \hat{\psi}_{t+1}
\end{align*}
where $\mathbf{H}^x$ and $\mathbf{H}^y$ are the lifting operators that map the latent state to the observation spaces.

Now, given the current state $\psi_t$, the predicted mean and covariance of the next state $\psi_{t+1}$ are:
\begin{align*}
\mu_{\psi_{t+1}} &= \mathbf{A} \psi_t\\
\Sigma_{\psi_{t+1}} &= \mathbf{Q}
\end{align*}
The predicted mean and covariance of the reconstructed observation $(x_{t+1}, y_{t+1}) = (\mathbf{H}^x \psi_{t+1}, \mathbf{H}^y \psi_{t+1})$ are:
\begin{align*}\mu_{(x_{t+1}, y_{t+1})} &= (\mathbf{H}^x \mu_{\psi_{t+1}}, \mathbf{H}^y \mu_{\psi_{t+1}})= (\mathbf{H}^x \mathbf{A} \psi_t, \mathbf{H}^y \mathbf{A} \psi_t)\\
\Sigma_{(x_{t+1}, y_{t+1})} &= \begin{pmatrix} \mathbf{H}^x \Sigma_{\psi_{t+1}} {\mathbf{H}^x}^T & \mathbf{H}^y \Sigma_{\psi_{t+1}} {\mathbf{H}^y}^T \end{pmatrix} = \begin{pmatrix} \mathbf{H}^x \mathbf{Q} {\mathbf{H}^x}^T & \mathbf{H}^y \mathbf{Q} {\mathbf{H}^y}^T \end{pmatrix}
\end{align*}
\newline
Now, to condition on the subset of coordinates, let $(x_{t+1}, y_{t+1})$ be partitioned into two parts: $(x_{t+1}, y_{t+1})^{(1)}$ and $(x_{t+1}, y_{t+1})^{(2)}$, where $(x_{t+1}, y_{t+1})^{(1)}$ are the coordinates we condition on, and $(x_{t+1}, y_{t+1})^{(2)}$ are the coordinates we want to sample. Suppose $(x_{t+1}, y_{t+1})^{(1)}$ is fixed to some vector $\mathbf{c}$. We can express $(x_{t+1}, y_{t+1})$ as being distributed as:
\[
   \begin{pmatrix}
   (x_{t+1}, y_{t+1})^{(1)} \\
   (x_{t+1}, y_{t+1})^{(2)}
   \end{pmatrix}
   \sim \mathcal{N}\left(
   \begin{pmatrix}
   \mu_{(x_{t+1}, y_{t+1})}^{(1)} \\
   \mu_{(x_{t+1}, y_{t+1})}^{(2)}
   \end{pmatrix},
   \begin{pmatrix}
   \Sigma_{(x_{t+1}, y_{t+1})}^{(11)} & \Sigma_{(x_{t+1}, y_{t+1})}^{(12)} \\
   \Sigma_{(x_{t+1}, y_{t+1})}^{(21)} & \Sigma_{(x_{t+1}, y_{t+1})}^{(22)}
   \end{pmatrix}
   \right)
   \]

The conditional distribution of $(x_{t+1}, y_{t+1})^{(2)}$ given $(x_{t+1}, y_{t+1})^{(1)} = \mathbf{c}$ is then:
\begin{align*}(x_{t+1}, y_{t+1})^{(2)} \mid (x_{t+1}, y_{t+1})^{(1)} = \mathbf{c} \sim \mathcal{N}\left(\mu_{\mathrm{cond}}, \Sigma_{\mathrm{cond}}\right)
\end{align*}
   where
\begin{align*}
   \mu_{\mathrm{cond}} &= \mu_{(x_{t+1}, y_{t+1})}^{(2)} + \Sigma_{(x_{t+1}, y_{t+1})}^{(21)} (\Sigma_{(x_{t+1}, y_{t+1})}^{(11)})^{-1} (\mathbf{c} - \mu_{(x_{t+1}, y_{t+1})}^{(1)})\\
   \Sigma_{\mathrm{cond}} &= \Sigma_{(x_{t+1}, y_{t+1})}^{(22)} - \Sigma_{(x_{t+1}, y_{t+1})}^{(21)} (\Sigma_{(x_{t+1}, y_{t+1})}^{(11)})^{-1} \Sigma_{(x_{t+1}, y_{t+1})}^{(12)}
\end{align*}

Next, to sample $\psi_{t+1}$ conditioned on the fixed values for a subset of coordinates in the observation spaces, we need to convert the conditional mean and covariance from the observation spaces back to the latent space. Let $\mathbf{H}^x_{2}$ and $\mathbf{H}^y_{2}$ be the rows of $\mathbf{H}^x$ and $\mathbf{H}^y$ corresponding to the free indices (the coordinates we want to sample), and let $\mathbf{H}^x_{1}$ and $\mathbf{H}^y_{1}$  be the rows of $\mathbf{H}^x$ and $\mathbf{H}^y$ corresponding to the fixed indices. We can solve the linear system to find the corresponding conditional mean and covariance in the latent space:
\begin{align*}\mu_{\mathrm{cond, diff}}& = {(\mathbf{H}^x_{2}}^{+} \mu_{\mathrm{cond}}, {\mathbf{H}^y_{2}}^{+} \mu_{\mathrm{cond}})\\\Sigma_{\mathrm{cond, diff}} &= \begin{pmatrix} {\mathbf{H}^x_{2}}^{+} \Sigma_{\mathrm{cond}} ({\mathbf{H}^x_{2}}^{+})^T & {\mathbf{H}^y_{2}}^{+} \Sigma_{\mathrm{cond}} ({\mathbf{H}^y_{2}}^{+})^T\end{pmatrix}
\end{align*}
where ${\mathbf{H}^x_{2}}^{+}$ and ${\mathbf{H}^y_{2}}^{+}$ are the Moore-Penrose pseudo-inverses of $\mathbf{H}^x_{2}$ and $\mathbf{H}^y_{2}$.
Finally, we sample from the conditional multivariate normal distribution to get $\psi_{t+1}$ given the conditional information in the original space:
\[\hat\psi_{t+1} \sim \mathcal{N}(\mu_{\mathrm{cond, diff}}, \Sigma_{\mathrm{cond, diff}})\]

\section{Application to Financial Stress Testing}
\label{sec:validation}

We now turn to an application of our framework to financial stress testing and scenario analysis of equity portfolios. As in \citet{chen1986economic}, we consider a standard macroeconomic multi-factor model where stock returns respond to unexpected changes in economy-wide factors. We stress our portfolios using the supervisory scenarios specified in \citet{federal_reserve_2022}. It is important to note that even unconditionally uncorrelated factors can display moderate dependence, for instance, factors can be uncorrelated and yet exhibit significant conditional dependence when a subset of them is learned. As a result the distribution of these factors conditional on an extreme scenario can display strong dependence.

Rather than selecting a handful of macro variables to explain stock returns, we incorporate \textbf{all} covariates in our stress-testing framework and instead allow the diffusion mapping to uncover the lower-dimensional intrinsic state $\theta$ that is driving the joint dynamics while jointly preserving most of the explanatory power of the higher-dimensisonal set. This allows for a more realistic scenario analysis where a Portfolio Manager (PM) can incorporate his beliefs about the state of the world rather than limiting defining a scenario in terms of three or four variables.\footnote{Note that the goal of this paper is not to come up with the best factors (in terms of explanatory power) to explain stock returns or the best factor model, but to illustrate the power of our data-driven framework to stress test portfolios where the risk factors are a function of a high-dimensional vector of characteristics.}

The standard multi-factor model assumes that the (excess) return of an asset $j$ at time $t$, denoted as $Y_{j,t}$, can be expressed as a linear function of a set of $p$ common factors $X_{1,t}, X_{2,t}, \dots, X_{p,t}$ and a firm-specific error term $\veps_{j,t}$:

\[Y_{j,t} = \beta_{1,j} X_{1,t} + \beta_{2,j} X_{2,t} + \cdots + \beta_{p,j} X_{p,t} + \veps_{j,t}\]
\newline
where $Y_{j,t}$ is the return (excluding dividends) of asset $j$ at time $t$; $X_{i,t}$ represents the value of the $i$-th factor at time $t$, where $i = 1, \dots, p$ for $p$ factors; $\beta_{i,j}$ represents the (time-independent) loading or sensitivity of asset $j$'s returns to factor $i$; $\veps_{j,t}$ is the idiosyncratic (or specific) error term for asset $j$ at time $t$, which is assumed to have a mean of zero and be uncorrelated across assets and time. We use this as our observation equation in the JDKF framework, which can be written in vector-matrix form as:
\[
\mathbf{Y}_t = \mathbf{B} \mathbf{X}_t + \mathbf{\veps}_t
\]
where \( \mathbf{Y}_t \) is the \( n \times 1 \) vector of excess returns at time \( t \) (where $n$ is the number of assets), \( \mathbf{X}_t\) is the \( p \times 1 \) vector of factor realizations at time \( t \), \( \mathbf{B} = (\mathbf{\beta}_1, \mathbf{\beta}_2, \dots, \mathbf{\beta}_p) \) is the \( n \times p \) matrix of factor loadings, where each column \( \mathbf{\beta}_i = (\beta_{i,1}, \beta_{i,2}, \dots, \beta_{i,n})^\top \) represents the loadings of all assets on factor \( i \), and \( \mathbf{\veps}_t = (\veps_{1,t}, \veps_{2,t}, \dots, \veps_{n,t})^\top \) is the \( n \times 1 \) vector of idiosyncratic errors. We note that this linear multi-factor model is the \textit{special case in our JDKF framework} where $G=\mathbf{B} \circ F$ described in Section \ref{sec:jdkf_model}, where we recall that under this situation, the diffusion coordinates are computed only on the observations ${x(t)}$ and there is only a lifting operator $\mathbf{H}^x$ from the diffusion coordinate space to the space of the observed common factors.

\subsection{Scenario Analysis Models}
\label{sec:stress_benchmarks}

We compare JDKF against three benchmark approaches for stress testing equity portfolios under macroeconomic scenarios. Each method addresses the challenge of predicting portfolio returns conditional on shocks to a subset $\mathcal{S}$ of risk factors, while accounting for co-movements in the remaining factors $\mathcal{S}^C$. Figure~\ref{fig:scenario_flow} provides an overview of the scenario analysis framework common to all methods.

\begin{figure}[H]
  \centering
  {\includegraphics[width=1\linewidth]{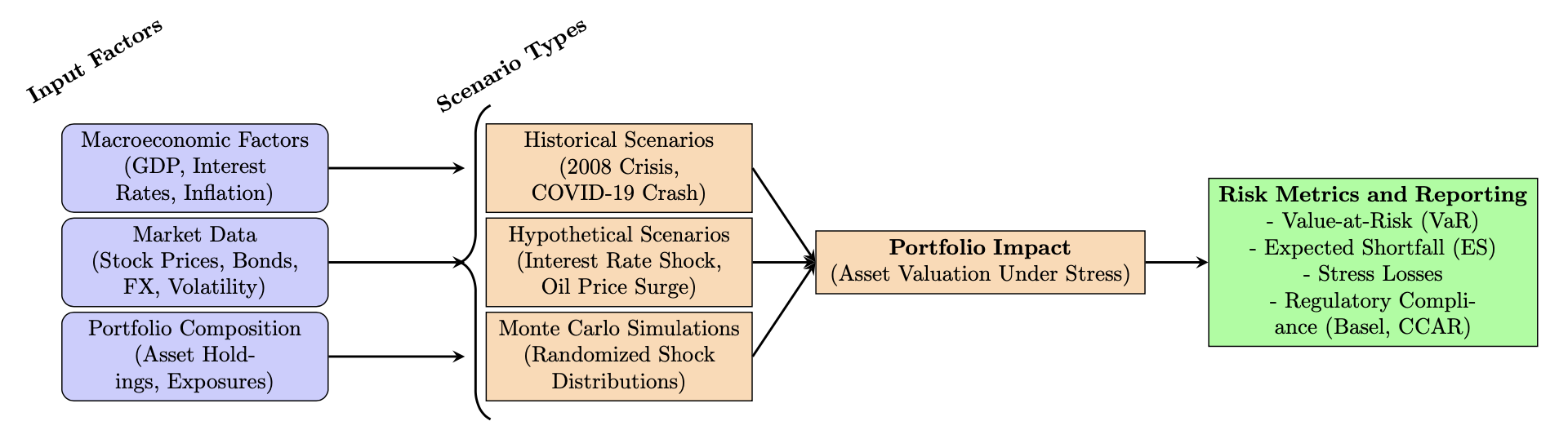}}
  \hfill
  \caption{{Diagram for scenario analysis in stress testing. Input factors 
    (macroeconomic variables, market data, portfolio composition) feed into scenario definitions, which can be historical (e.g., 2008 crisis, COVID-19), hypothetical (e.g., interest rate shock, oil price surge), or generated via Monte Carlo simulations. We then assess the portfolio change under stress and evaluate it through risk metrics 
    such as Value-at-Risk (VaR), Expected Shortfall, stress losses, and regulatory capital 
requirements (Basel, CCAR).}}
  \label{fig:scenario_flow}
\end{figure}

\subsubsection{Standard Scenario Analysis}

Standard Scenario Analysis (SSA) for financial portfolios aims to compute a portfolio's profit and loss (P\&L) or overall return under various combinations of stresses to a sub-collection of the risk factors that drive the portfolio's valuation. It is well known that scenario portfolio numbers calculated under SSA can be quite inaccurate for several reasons. First, post-stress P\&L is computed under the assumption that the un-stressed factors remain unchanged, however, there is often clear statistical dependence among factors, and in particular conditional dependence. Applying a stress to a sub-collection of factors at a given time period $t$ should also have dynamic propagation effects through future time periods. A good stress test must not only take into account cross-factor dependencies but also the joint dynamics of the common factors themselves to assess how a stress propagates through time.

One drawback of SSA is its lack of backtesting to validate the accuracy of scenario-conditional P\&L. Since SSA does not explicitly impose a probability model on the factors, it does not naturally lend itself to backtesting. In \citet{haugh2020scenario}, they propose a standard linear Gaussian state space model for the dynamics of the risk factors and scenario analysis is married with dynamic factor models. A dynamic factor model for scenario analysis has the advantage that it now allows for backtesting the performance of a stress test for a given portfolio. However, when considering equity portfolios where the assets are modeled with a high-dimensional set of common factors and their joint dynamics are highly nonlinear, such a standard state space model no longer applies.

Under the standard multi-factor model, SSA stresses the components of $\mathbf{X}_{\mathcal{S},t}$ according to the given scenario while keeping components in $\mathbf{X}_{\mathcal{S}^C,t}$ unchanged (i.e., equal to their current value). The portfolio return is then computed as $V_t^{\text{SSA}} = \mathbf{w} \cdot \mathbf{Y}_t^{\text{SSA}}$, where $\mathbf{Y}_t^{\text{SSA}} = \tilde{\mathbf{B}}\mathbf{X}_t^{\text{SSA}}$ and $\tilde{\mathbf{B}}$ is estimated via OLS regression. We provide the complete algorithmic details  in Appendix~\ref{app:ssa}.  {This makes the implicitly assumption $\mathbb{E}[\mathbf{X}_{\mathcal{S}^C,t+1}|\mathcal{F}_t, \mathbf{X}_{\mathcal{S},t+1}] = 0$, which is a fundamental limitation. It ignores dynamic correlations: stressing inflation and unemployment should affect 
expectations of unstressed risk factors, yet SSA holds the unstressed factors fixed. Our JDKF framework addresses 
this by learning the conditional distribution of unstressed factors given stressed ones through the joint covariation structure captured in diffusion coordinates.}

\subsubsection{Static PCA}

We employ a PCA-based benchmark to assess the impact of stressed scenarios on asset returns. This benchmark is closer to SSA in that we do not impose model dynamics on the PC components; however, unlike SSA, the un-stressed factors are not assumed to remain unchanged and their co-movement is given from the projection onto the principal component space.

Given a factor matrix $\mathbf{X} \in \mathbb{R}^{T \times p}$ and a return matrix $\mathbf{Y} \in \mathbb{R}^{T \times n}$, we first estimate the linear multi-factor regression model to obtain $\tilde{\mathbf{B}}$, then compute principal components from the differenced and centered factor data $\tilde{\mathbf{X}}_t = \Delta \mathbf{X}_t - \boldsymbol{\mu}_{\Delta X}$. For each specified PCA dimension $d$, we compute the PCA transformation and extract the principal component loading matrix $\mathbf{W} = \mathbf{V}^T$ from the singular value decomposition.

To apply a stressed scenario, we define the unstressed scenario by setting $\mathbf{X}^{\text{scenario}}_{t+1,i} = \mathbf{X}_{t,i}$ for all $i \notin \mathcal{S}$ (unchanged factors), and $\mathbf{X}^{\text{scenario}}_{t+1,i} = \mathbf{X}_{t+1,i}$ for $i \in \mathcal{S}$ (stressed factors set to realized historical values). The perturbation vector $\Delta \mathbf{X} = \mathbf{X}^{\text{scenario}}_{t+1} - \mathbf{X}_t$ is centered and projected onto the PCA space, i.e., $\text{proj}_{\text{PCA}} = \mathbf{W}(\mathbf{W}^T \Delta \tilde{\mathbf{X}})$. The perturbed scenario factor vector is then $\mathbf{X}^{\text{PCA-stress}}_{t+1} = \mathbf{X}_t + \text{proj}_{\text{PCA}} + \boldsymbol{\mu}_{\Delta X}$, and stressed returns are predicted as $\mathbf{Y}^{\text{PCA-stress}}_{t+1} = \mathbf{X}^{\text{PCA-stress}}_{t+1} \tilde{\mathbf{B}}$.

For each backtest, we select the PCA dimension $d$ that minimizes mean absolute error (MAE), giving Static PCA the best possible performance. Complete implementation details including the preprocessing steps are provided in Algorithm~\ref{alg:static_pca} (Appendix~\ref{app:static_pca}).

\subsubsection{Dynamic PCA}

A natural comparison benchmark is a dynamic version of Principal Component Analysis (PCA), which provides a simpler alternative to the arguably more complex diffusion mapping. Dynamic PCA combines principal component dimensionality reduction with a linear-Gaussian
state-space model for the PC time series. The model structure mirrors JDKF but replaces diffusion
coordinates with PCs.

We define the following PCA-based linear Gaussian state space model:

\begin{align}
\text{Observation Equation:} \quad & \mathbf{X}_t = \boldsymbol{\Gamma} \mathbf{Z}_t + \boldsymbol{\varepsilon}_t, \quad \boldsymbol{\varepsilon}_t \sim \mathcal{N}(\mathbf{0}, \mathbf{R}) \\
\text{State Equation:} \quad & \mathbf{Z}_t = \mathbf{A}\mathbf{Z}_{t-1} + \boldsymbol{\eta}_t, \quad \boldsymbol{\eta}_t \sim \mathcal{N}(\mathbf{0}, \mathbf{Q})
\end{align}
where $\mathbf{X}_t$ is the observed data vector at time $t$; $\mathbf{Z}_t$ is the latent state vector (principal component embeddings) at time $t$; $\boldsymbol{\Gamma}$ is the eigenvector matrix (computed from PCA on the data) serving as the observation (loading) matrix; $\mathbf{A}$ is the transition matrix governing the dynamics of the principal component embeddings; and $\mathbf{Q}$ is the process noise covariance matrix.

We estimate the transition matrix $\mathbf{A}$ by fitting a VAR(1) process to the PC components computed from the data. We then implement the standard Kalman filter (see Appendix~\ref{sec:appd}) to obtain dynamic estimates $\hat{\mathbf{Z}}_t$ of the PC embeddings. In our empirical experiments we select the dimension of the PC embeddings as $d = 5$ such that $>99\%$ of the variance in our data is explained.

Once we have obtained our estimates for $\mathbf{Z}_t$, we perform conditional sampling given a scenario to obtain the stressed factor vector. The conditional sampling procedure closely follows that outlined in Section~\ref{sec:sample} with the exception that we replace the lifting operator matrix with $\boldsymbol{\Gamma}$ and the diffusion coordinates with $\mathbf{Z}_t$. We obtain $K$ Monte Carlo conditional samples $\{\mathbf{Z}^{(k)}_{t+1} | \mathbf{Z}_t, \mathbf{X}_{\mathcal{S},t}\}_{k=1}^K$, predict post-stress asset returns for each sample as $\mathbf{Y}^{(k)}_{t+1,\text{Dynamic PCA}} = \tilde{\mathbf{B}}\boldsymbol{\Gamma}^T \mathbf{Z}^{(k)}_{t+1}$, and obtain the final dynamic PCA post-stress returns by averaging across all Monte Carlo predictions.

The complete algorithmic details are provided in Appendix~\ref{app:dynamic_pca}, and summarized in Algorithm~\ref{alg:dynamic_pca}.

\subsubsection{Linear Multi-Factor JDKF}

Unlike SSA, the key advantage of our JDKF framework is that it accounts for dynamic correlations among the factors to compute conditional portfolio return under a given stress scenario, all this without imposing any parametric model on the joint factor dynamics. The conditional sampling procedure described in Section~\ref{sec:sample} allows us to sample dynamic diffusion coordinates $\boldsymbol{\psi}_t$ conditional on a scenario in the observation space, $\mathbf{X}_{\mathcal{S},t}$. 

Under the multi-factor model for asset returns, the multivariate regression equation becomes our observation equation, so that we operate under the following state space:
\begin{align}
\text{Observation Equations:} \quad & \mathbf{Y}_t = \mathbf{B}\mathbf{X}_t + \boldsymbol{\varepsilon}_t = \mathbf{B}\mathbf{H}^X\boldsymbol{\psi}_t + \mathbf{v}_{Y,t}, \quad \mathbf{v}_{Y,t} \sim \mathcal{N}(\mathbf{0}, \mathbf{R}_Y) \nonumber \\
& \mathbf{X}_t = \mathbf{H}^X\boldsymbol{\psi}_t + \mathbf{v}_{X,t}, \quad \mathbf{v}_{X,t} \sim \mathcal{N}(\mathbf{0}, \mathbf{R}_X) \\
\text{State Equation:} \quad & \boldsymbol{\psi}_{t+1} = \mathbf{A}\boldsymbol{\psi}_t + \boldsymbol{\eta}_t, \quad \boldsymbol{\eta}_t \sim \mathcal{N}(\mathbf{0}, \mathbf{Q}) \nonumber
\end{align}
where $\mathbf{X}_t$ is the observed data vector at time $t$; $\boldsymbol{\psi}_t$ is the latent state vector (diffusion coordinates) at time $t$; $\mathbf{H}^X$ is the linear lifting operator (computed from applying diffusion maps on the data); $\mathbf{v}_{Y,t}$ is a Gaussian observation noise with covariance matrix $\mathbf{R}_Y$, correlated with $\mathbf{v}_{X,t}$; $\mathbf{A}$ is the transition matrix governing the dynamics of the diffusion coordinates; and $\mathbf{Q}$ is the process noise covariance matrix.

We first apply diffusion maps on the common factors data matrix $\mathbf{X} \in \mathbb{R}^{T \times p}$ to obtain the matrix of diffusion coordinates $\boldsymbol{\psi} \in \mathbb{R}^{T \times \ell}$, the lifting operator $\mathbf{H}^X$, and the transition matrix $\mathbf{A} = (\mathbf{I} + \boldsymbol{\Lambda})$. We then estimate the JDKF parameters $\mathbf{B}$, $\mathbf{R}_X$, $\mathbf{R}_Y$, and $\mathbf{R}_{XY}$ by fitting the joint Kalman filter as outlined in Section~\ref{sec:jdkf_model} to obtain filtered dynamic diffusion coordinate estimates $\boldsymbol{\psi}^{\text{JDKF}}_t$.

Using the estimated diffusion coordinates from JDKF, we perform conditional sampling given the scenario in the factor observation space following the procedure outlined in Section~\ref{sec:sample}. We obtain $K$ Monte Carlo conditional samples $\{\boldsymbol{\psi}^{(k)}_{t+1} | \boldsymbol{\psi}_t, \mathbf{X}_{\mathcal{S},t}\}_{k=1}^K$, predict post-stress asset returns for each conditional Monte Carlo sample as $\mathbf{Y}^{(k)}_{t+1,\text{JDKF}} = \mathbf{B}\mathbf{H}^X\boldsymbol{\psi}^{(k)}_{t+1}$, and obtain the JDKF post-stress returns by averaging across all Monte Carlo predictions: $\mathbf{Y}_{t+1,\text{JDKF}} = \frac{1}{K}\sum_{k=1}^K \mathbf{Y}^{(k)}_{t+1,\text{JDKF}}$.

The complete implementation details are provided in Appendix~\ref{app:jdkf_implementation},  including the step-by-step procedure described in Algorithm~\ref{alg:jdkf}. 

\subsection{Historical Backtesting}
\label{sec:backtest}

To validate our model with data, we conduct historical backtests covering major periods of financial crisis and set our scenarios equal to the actual realized historical scenarios. For a given factor $i \in \mathcal{S}$ in our scenario, the magnitude of the stress is $\Delta \mathbf{X}_{t}^i=\mathbf{X}_{t}^i-\mathbf{X}_{t-1}^i$. For each historical backtest, we implement JDKF and the comparison benchmarks (SSA, Static PCA, and Dynamic PCA) described in Section \ref{sec:stress_benchmarks}. We then compare the conditional (on the scenario) predicted portfolio returns of each model relative to the realized time $t$ portfolio return. On each period of our backtest, we consider a portfolio manager who constructs portfolio weights via the $1/N$ equal weighting rule and predicts the portfolio return conditional on the scenario under each model.

Historical backtesting involves estimating the performance of a given model \textit{if it had been employed in the past}. The main idea is to predict portfolio return conditional on a scenario that we have observed in the past. More concretely, if at time period $t$ we know the historical magnitude of the stress $\Delta \mathbf{X}_{t}^i$ for our scenario-defining factors $i \in \mathcal{S}$, can we accurately predict the co-movements in the un-stressed factors $i \not \in \mathcal{S}$? The goal of our historical backtest is to assess how well JDKF captures the dynamic correlation among the factors driving a portfolio's overall return as compared to the other benchmarks. Our backtest consists of two phases (see Figure \ref{fig:backtest_flow}):

\begin{itemize}
\item {Refitting Period.} For the linear multi-factor JDKF and Dynamic PCA, we periodically refit the Kalman filters based on a fixed refitting period $R$. We start the Kalman filters with lead time $R$ before the backtesting period and run them through to the end, allowing the filters to pass through their burn-in period and reach stable estimates.

\item {Backtesting Period.} We retrieve the filtered estimates for the diffusion coordinates (for the linear multi-factor JDKF) and PC components (for Dynamic PCA) from the pre-fitted Kalman filters. We then train our models on a one-step-ahead rolling basis using data from $t-s:t$ (where $s$ is the rolling window size) and predict portfolio return one period out.
\end{itemize}
Our rationale for this two-phase design is to avoid refitting the Kalman filter every time the training window rolls forward, which would require the filter to repeatedly pass through its burn-in period. By allowing the filter to run continuously with periodic refitting every $R$ periods, we ensure the estimates remain stable and accurate. Figure \ref{fig:backtest_flow} displays a schematic representation of our backtesting design.

\begin{figure}
  \centering
  {\includegraphics[width=1\linewidth, trim=0cm 0cm 0cm .72cm, clip]{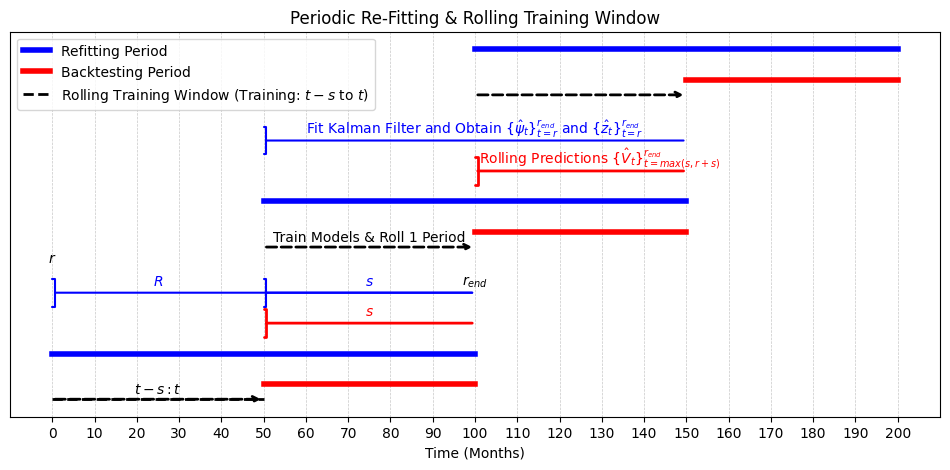}\label{fig:f1}}
  \hfill
  \caption{Structure of the historical rolling backtest with periodic Kalman filter 
refitting. The timeline is divided into refitting periods and backtesting periods. During each refitting period, the Kalman filter is 
fitted once to obtain filtered estimates 
$\{\hat{\psi}_t\}$ and $\{\hat{z}_t\}$. Within each backtesting period, models are 
trained on rolling windows and predictions 
are made one period ahead. We  
show the full 200-month timeline spanning June 2004-December 2016 with three complete 
refitting cycles. The insets illustrate that after a burn-in period, the 
Kalman filter runs through the refitting period to stabilize estimates before rolling 
predictions begin.} 
  \label{fig:backtest_flow}
\end{figure}

Given the time series of portfolio return predictions $\hat{V}_t$ and absolute errors $E_t$ for each model from running the historical backtests, we then consider two main metrics for assessing model performance. Following the literature, as in \citet{haugh2020scenario}, we consider the Mean Absolute Error (MAE) for each model defined as:

\[\frac{1}{T_{\text{test}}}\sum_{t=1}^{T_{\text{test}}} |\hat{V}_t-V_{t, \mathrm{TRUE}}| \]
where $T_{\text{test}}$ is the number of time periods for each test financial crisis. We also consider the percentage accuracy of JDKF against the comparison benchmarks, defined as the percentage of time JDKF yielded a predicted portfolio return closer to the truth, expressed as:

\[\frac{\sum_{t=1}^{T_{\text{test}}} \mathds{1}_{\{E_{\mathrm{JDKF}, t}/E_{\text{bench}, t}<1 \}}}{T_{\text{test}}} \times 100 \\ \textrm{ where } E_t=|\hat{V}_t-V_{t, \mathrm{TRUE}}|\]
We also employ a standard Value-at-Risk (VaR) exceptions test to statistically test whether the predicted portfolio returns $\hat{V}_{t, \mathrm{JDKF}}$ from the backtest are consistent with the realized $V_{t, \mathrm{TRUE}}$. As is standard practice in risk management, and using $V_t$ as the random variable representing portfolio return, the time $t$ level-$\alpha$ VaR, which is an $\mathcal{F}_t$-measurable random variable, is defined for $\alpha \in (0,1)$ as:
\[\mathbb{P}(V_t \leq \operatorname{VaR}_t(\alpha)|\mathcal{F}_t)=1-\alpha\]
In short, the time $t$ VaR at level $\alpha$ is simply the $(1-\alpha)$-quantile of the distribution of the portfolio return conditional on $\mathcal{F}_t$. Now, a VaR exception is defined as the event that the realized portfolio return is below $\operatorname{VaR}_t(\alpha)$, as is standard practice, one can then define an indicator variable:
\[\mathds{1}_t(\alpha)= \begin{cases} 1 &, \text{ if } V_t <\operatorname{VaR}_t(\alpha) \\ 0 & , \text{ otherwise }\end{cases}\]
Our backtest gives us a sequence of empirical VaR exception indicators $\hat{\mathds{1}}_t(\alpha)$ from the empirical distribution of the JDKF portfolio return predictions. Following \citet{kupiec1995}, under the null hypothesis that JDKF is correct, $\sum_{t=1}^T \hat{\mathds{1}}_t(\alpha)$ follows a Binomial distribution with $T$ trials and success probability $1-\alpha$. This yields a two-sided binomial test  with test statistic:

\[Z_T=\frac{\sum_{t=1}^T \hat{\mathds{1}}_t(\alpha)-T(1-\alpha)}{\sqrt{T \alpha (1-\alpha)}}\]
which we use to test the null hypothesis.

The complete implementation details with periodic Kalman filter refitting for the historical rolling backtest are provided in Appendix~\ref{app:backtest_algorithm}, including the detailed steps of the procedure stated in Algorithm~\ref{alg:historical_backtest}.

\section{Model Estimation using Real Data}
\label{sec:empirical}

\subsection{Data Description}
 We model the series of stock returns via a classical macroeconomic factor model with the 124 FRED-MD macro-level factors from \citet{mccracken2016fred} plus the 8 macroeconomic factors from \citet{goyal2008comprehensive}. This gives us a total of D=132 common factors. We utilize the curated data version as provided from \citet{chen1986economic}, which applies some standard transformations to the macro factors to ensure we deal with stationary time series (see Appendix \ref{sec:appe} for the transformation details). The monthly stock returns data is obtained from CRSP. The time period considered is 1967-2016. We consider stocks trading in the NYSE, NASDAQ, and the American Stock Exchange. For each financial crisis period in our backtests, we maximize the size of the universe of firms by using the largest number of firms present throughout the entire time span of the considered period. We select four periods for out-of-sample prediction:
 
\begin{enumerate}
    \item July 1990 - Mar 1991. We fit the Kalman filters only once with sufficient lead time before the start of the crisis. We re-train for each rolling  window time step on all data available from $t-s:t$ and predict conditional portfolio return for each of the 8 periods (months) in the backtesting phase. This period captures to the oil price shock of 1990, debt accumulation, and consumer pessimism present during those times. For this period, we have a universe of N=1,200 firms.
    \item Mar 2001 - Nov 2001. We fit the Kalman filters only once with sufficient lead time before the start of the crisis. We re-train for each rolling window time step on all data available from $t-s:t$ and predict conditional portfolio return for each of the 8 periods in the backtesting phase. This period captures the dot-com bubble and the 9/11 attacks. For this period we have a universe of N=1,344 firms. 
    \item Dec 2007 - June 2009. We fit the Kalman filters only once, with sufficient lead time before the start of the crisis period. We re-train for each rolling window time step on all data available from $t-s:t$ and predict conditional portfolio return for each of the 18 periods in the backtesting phase. This period contains the subprime mortgage crisis. For this period we have a universe of N=1,409 firms. 
    \item June 2004-December 2016, where we periodically refit the Kalman filter every 50 months and predict conditional portfolio return for each of the 150 periods within the crisis; we have three backtesting periods, each of length 50, corresponding exactly as in Figure \ref{fig:backtest_flow}. This period captures the 2007-2009 housing crisis and this test was conducted to allow for a large training set and a large out-of-sample test set to allow for a more robust binomial VAR exceptions test. For this period we have a universe of N=1,776 firms.
\end{enumerate}

 For our empirical tests, we consider a portfolio manager who, in each period, uses the $1/N$ rule to construct the weights of his portfolio. To define our scenarios, we consider the set of macroeconomic variables in Table~\ref{tab:scenario_variables}, which roughly correspond to the FED's 2022 variables for their supervisory scenarios. These variables are taken as our subset $\mathcal{S}$ of the common factors that will define the stress variables in the historical backtests:

\begin{longtable}{|l|p{8cm}|}
\caption{Scenario Stress Macroeconomic Variables. {These 15 variables constitute the stressed subset from the full 132 
FRED-MD and Goyal--Welch factors. This selection 
covers key transmission channels in financial stress testing: equity markets (S\&P 500), 
inflation (CPI), foreign exchange (4 currency pairs vs. USD), monetary policy (Fed 
Funds rate, Treasury rates at 3-month, 5-year, and 10-year maturities), credit spreads 
(AAA and BAA corporate bonds), real economic activity (real personal income, civilian 
unemployment), and market volatility (VXO index).}} 
\label{tab:scenario_variables}
\\
\hline
\textbf{Variable Name} & \textbf{Description} \\ \hline
\endfirsthead

\hline
\textbf{Variable Name} & \textbf{Description} \\ \hline
\endhead

\hline
\endfoot

\hline
\endlastfoot

S\&P 500 & S\&P's Common Stock Price Index: Composite \\ \hline
CPIAUCSL & CPI: All Items \\ \hline
EXSZUSx & Switzerland / U.S. Foreign Exchange Rate \\ \hline
EXJPUSx & Japan / U.S. Foreign Exchange Rate \\ \hline
EXUSUKx & U.S. / U.K. Foreign Exchange Rate \\ \hline
EXCAUSx & Canada / U.S. Foreign Exchange Rate \\ \hline
FEDFUNDS & Effective Federal Funds Rate \\ \hline
RPI & Real Personal Income \\ \hline
UNRATE & Civilian Unemployment Rate \\ \hline
TB3MS & 3-Month Treasury Bill Rate \\ \hline
GS5 & 5-Year Treasury Rate \\ \hline
GS10 & 10-Year Treasury Rate \\ \hline
AAA & Moody’s Seasoned Aaa Corporate Bond Yield \\ \hline
BAA & Moody’s Seasoned Baa Corporate Bond Yield \\ \hline
VXOCLSx & CBOE S\&P 100 Volatility Index: VXO \\ 

\end{longtable}

\subsection{Stress Test Analysis Data}
\label{sec:experiments}
We conduct a historical backtest as outlined in Algorithm \ref{alg:historical_backtest} in Appendix \ref{app:backtest_algorithm} for each one of the financial crises. For each backtest, in addition to reporting the performance metrics MAE and percentage accuracy, we also display the time series plots of the predicted conditional portfolio returns for each one of the models along with the true historical portfolio returns, where we re-emphasize that each portfolio is constructed with the $1/N$ rule. An important note is that for Static PCA, we perform each backtest for every possible PC component dimension and select the best performing one in terms of MAE. We begin with the richest of the four backtests, spanning June 2004-December 2016, where we have 150 prediction periods. As mentioned earlier, in this backtest we periodically refit the Kalman filters every 50 months. Figure \ref{fig:2004_to_2016} displays the time series of the conditional predictions from the backtest. We note that for a few months, Dynamic PCA yielded predictions that were significantly far off from the truth, we removed some of these outlier predictions from the plot to preserve the visibility of all the other time series displayed. We display a gray bar over each month where Dynamic PCA yielded such an outlier prediction.

\begin{figure}[H] 
    \centering 
    {\includegraphics[width=0.79\linewidth, trim=0cm 0cm 0cm .72cm, clip]{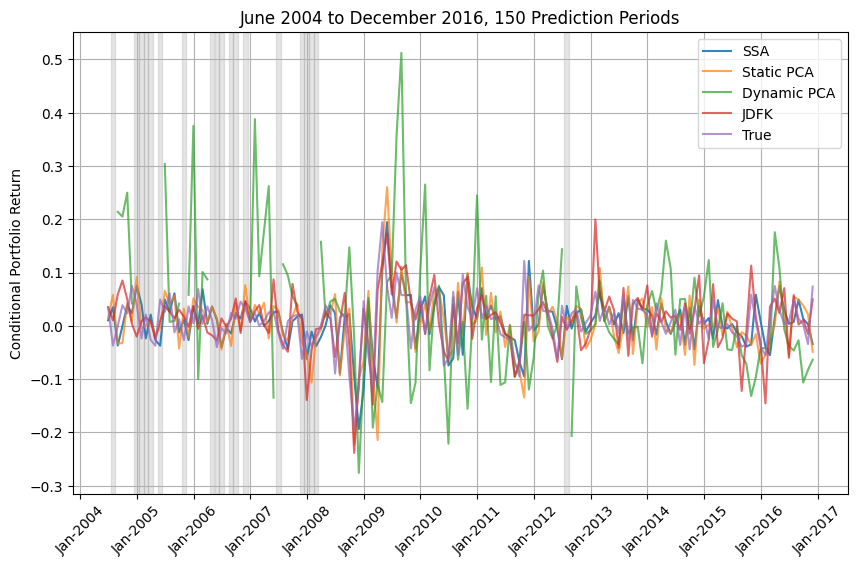}} 
    \caption{{Historical backtest: June 2004-December 2016 (150 out-of-sample monthly predictions). 
True portfolio returns (purple) versus conditional predictions from JDKF (red), 
SSA (blue), Static PCA (orange), and Dynamic PCA (green). Gray vertical bars indicate 
periods where Dynamic PCA produced outlier predictions. 
JDKF tracks realized returns more closely throughout, particularly during the 2007-2009 
financial crisis. We consider an equal-weighted portfolio across 
1,776 stocks. Scenarios are defined by stressing the 15 macroeconomic factors in 
Table \ref{tab:scenario_variables}.}}\label{fig:2004_to_2016}
\end{figure}

We report the MAE for each model in Table \ref{tab:mae}. We find that JDKF achieves a significantly lower MAE than all of the other benchmarks, and a 35.86\% reduction over the MAE achieved by SSA, which is the second best performing benchmark in terms of MAE. In terms of percentage accuracy, as shown in Table \ref{tab:accuracy}, over the 150 predicted periods, JDKF performs better than SSA 66\% of the time, better than Dynamic PCA 66.67\% of the time, and better than Static PCA 76.67\% of the time. Notice that even though Static PCA achieved a slightly higher MAE than SSA (about 5.9\% higher), the percentage of time that JDKF performs better than Static PCA is lower than that of SSA. This indicates that, while on average SSA achieves a slightly lower absolute error, Static PCA is more consistent in yielding predictions closer to JDKF. This is suggestive of SSA exhibiting more frequent large deviations, whereas Static PCA, despite its slightly higher MAE, yields predictions that are more stable relative to JDKF across the 150 prediction periods. 

As is common practice, we also assess whether the realized portfolio values are consistent with the conditionally predicted portfolio values via a 95\% VaR exceptions test on the predictions obtained under JDKF. We have already seen that JDKF is superior to all comparison benchmarks not only in terms of MAE, but also in terms of percentage accuracy, now we statistically assess the null hypothesis that JDKF is the correct model via the two-sided binomial VaR exceptions test, and obtain the following results:
\begin{table}[H]
    \centering
    \begin{tabular}{ll}
        \toprule
        \textbf{Metric} & \textbf{Value} \\
        \midrule
        VaR Threshold ($\alpha = 0.05$) & -0.0691\\
        Actual Exceptions & 7 \\
        Expected Exceptions & 7.5 \\
        Binomial Test p-value & 0.99\\ 
        \bottomrule
    \end{tabular}
    \caption{Value-at-Risk exceptions test at 95\% confidence level for JDKF predictions. Under the null hypothesis that JDKF is correct and its predicted portfolio returns are consistent with the realized returns, the number of exceptions (realized returns below 
the VaR threshold) should follow Binomial$(150, 0.05)$ with expected value 7.5. We observed 7 exceptions. Two-sided binomial test yields $p=0.99$, providing no evidence 
against JDKF's distributional accuracy. The VaR threshold of $-6.91\%$ corresponds to 
the empirical 5th percentile of JDKF's predicted conditional return distributions.}
\label{tab:var_results_95}
\end{table}

Note that for the 95\% VaR exceptions test in Table \ref{tab:var_results_95} there is no strong evidence to reject the null hypothesis that JDKF is correct, i.e., that the VaR estimates from JDKF appear statistically consistent with the expected distribution. In both cases the number of observed exceptions is in line with the expected exceptions. For the 95\% test we observe 7 exceptions with an expected 7.5, JDKF is slightly conservative but certainly well within an acceptable range. This indicates JDKF is not displaying systematic under or over-estimation of risk. 

We now proceed to perform similar backtests on the other shorter periods of financial turmoil, which contain a very small number of prediction periods. Note that there are more periods of recessions in our data sample, but we decided to begin in 1990 to have a reasonable training window size. As mentioned earlier, for each of these shorter crisis periods, we fit the Kalman filter once starting a few months before the beginning of the interval to overcome the burn-in period. For the diffusion mapping, throughout all backtests (including the already discussed June 2004-December 2016), we made the empirical choice of $\ell=39$ for the diffusion coordinate embedding dimension and set $\veps$ to the median $d(x(t_i),x(t_j))$, based on the decay of the spectrum of the $\mathbf{P}$ matrix, displayed in Figure \ref{fig:macro_spectrum} for different choices of $\veps$. The rate of spectrum decay generally determines how well the manifold structure has been captured, slower decay requires more eigenvectors to represent the data, while faster decay usually suggests fewer dimensions are sufficient. We also display in Figure \ref{fig:macro_mse} the mean squared reconstruction error (using the lifting operator $\mathbf{H}^x$) for the diffusion maps computed on the macroeconomic factors for different choices of $\ell$ and $\veps$, our choice of $\veps$ balances a lower error than that for smaller $\veps$ with a better capture of the manifold structure than that for larger $\veps$, as evidenced by Figure \ref{fig:macro_spectrum}, where the larger $\veps$ capture mostly noise. These choices of $\ell$ and $\veps$ are fixed throughout the historical backtests and was obtained by inspecting the point where we observe the largest pairwise consecutive drop in the ordered eigenvalues obtained from applying the diffusion map to the entire time series of macroeconomic factors.

\begin{figure}
\centering

\subfloat[Diffusion Map Spectrum: Eigenvalues of the Markov \newline Matrix $\mathbf{P}$ for different choices of $\veps$ as a function of the \newline diffusion coordinate embedding dimension $l$
]{\includegraphics[width=0.5\linewidth, trim=0cm 0cm 0cm .77cm, clip]{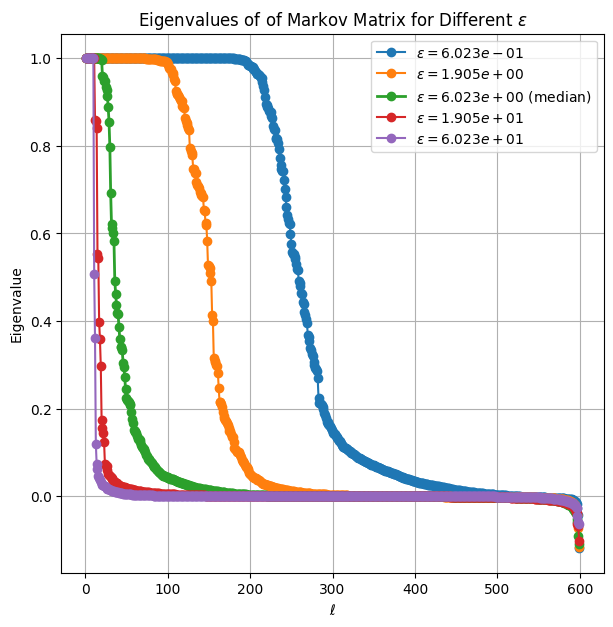}\label{fig:macro_spectrum}}
      \subfloat[Mean Squared Reconstruction Error:  Mean Squared Error from reconstructing diffusion coordinates via the lifting  operator $\mathbf{H}^x$ for different choices of $\veps$ as a function of \newline the  diffusion coordinate embedding dimension $l$]{\includegraphics[width=0.5\linewidth, trim=0cm 0cm 0cm .77cm, clip]{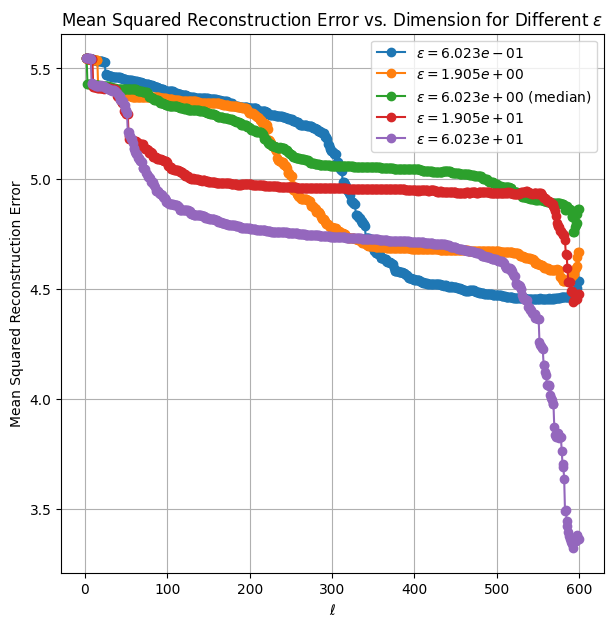}\label{fig:macro_mse}} \\
\caption{Diffusion map diagnostics for selecting kernel bandwidth $\varepsilon$ and 
embedding dimension $\ell$ using the 132 macroeconomic factors. We select a bandwith parameter equal to the median $d(x(t_i),x(t_j))$.}
\label{fig:macro_diagnostics}
\end{figure}

Figure \ref{fig:all_tests} displays the conditional predicted portfolio returns under each model as well as the truth for each of the three shorter crisis periods. Again, we note that for Dec 2007-June 2009, Dynamic PCA yielded outlier predictions during the first few periods, which we removed from the plot in panel \ref{fig:mortgage} for visibility purposes, and marked those periods with a gray bar. Analyzing the panel \ref{fig:oil} for the Oil Price Shock crisis from July 1990-March 1991, it is striking to observe the consistency of the movements between JDKF and the true portfolio return. Focusing on the sharp downturn from September to October 1990, followed by the consistent increase from October 1990 to January 1991, JDKF correctly anticipates the fall and rise, indicating that the dynamic correlation among the common factors is being captured from our sampling procedure. SSA and Static PCA for instance, seem to identify the sudden fall and subsequent steady rise but with a clear delay. Dynamic PCA struggles to capture any consistent trends from the true portfolio returns, this is highly indicative that neither the structure of the common factors nor their dynamics are well explained by linear models (PCA, and VAR(1)). Furthermore, as shown in Table \ref{tab:mae}, JDKF achieves the lowest MAE out of all benchmarks, followed closely by Static PCA. While it is true that Static PCA is a close competitor of JDKF in terms of MAE, Table \ref{tab:accuracy} shows that JDKF achieves higher accuracy than Static PCA 62.5\% of the time, indicating that JDKF is more consistent in capturing month-to-month fluctuations.

\begin{figure}
\centering
\begin{tabular}{@{}c@{}}
\subfloat[\label{fig:oil}]{\includegraphics[width=0.655\linewidth, trim=0cm 0cm 0cm .77cm, clip]{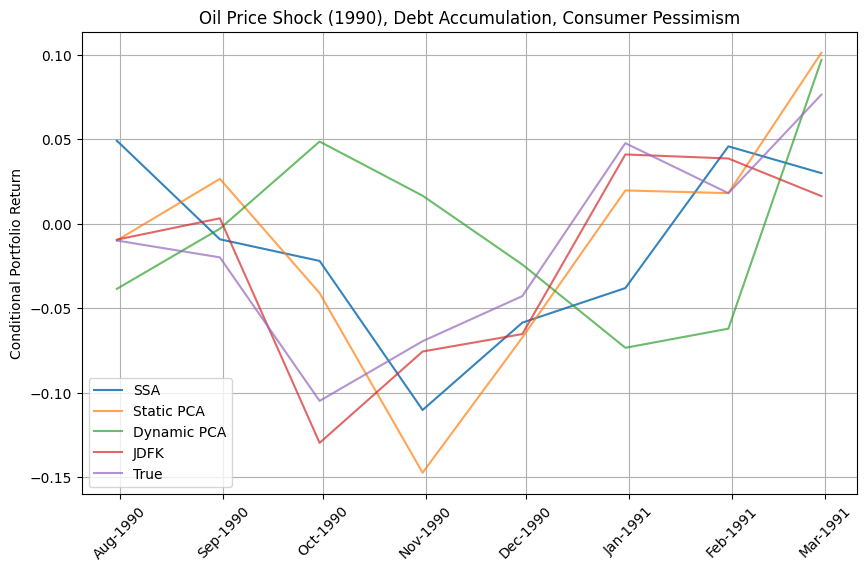}}\\
\subfloat[\label{fig:dotcom}]{\includegraphics[width=0.655\linewidth, trim=0cm 0cm 0cm .77cm, clip]{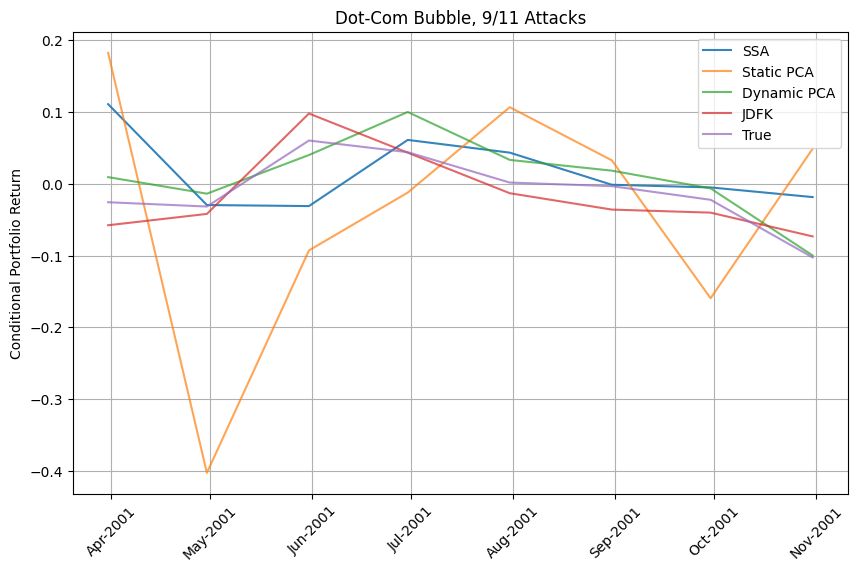}} \\
\subfloat[\label{fig:mortgage}]{\includegraphics[width=0.655\linewidth, trim=0cm 0cm 0cm .77cm, clip]{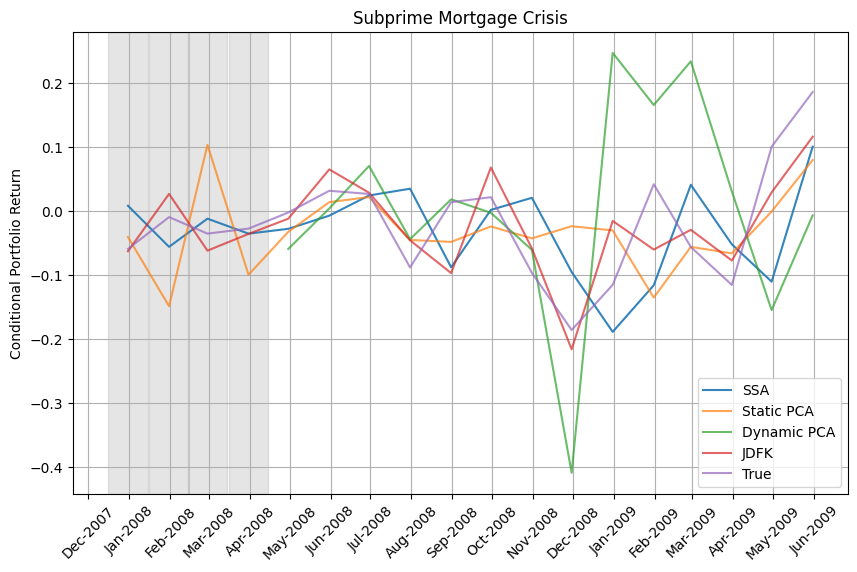}}
\end{tabular}\hfill
\begin{minipage}[t]{0.31\linewidth}
\caption{Historical backtests during major financial crises. \protect\subref{fig:oil} July 1990 - March 1991: True and Conditionally Predicted Portfolio Returns Spanning the Oil Price Shock of 1990, Debt Accumulation, Consumer Pessimism. \protect\subref{fig:dotcom} March 2001 - November 2001: True and Conditionally Predicted Portfolio Returns Spanning the Dot-Com Bubble, 9/11 Attacks \protect\subref{fig:mortgage} December 2007 - June 2009: True and Conditionally Predicted Portfolio Returns Spanning the Subprime Mortgage Crisis.}
\label{fig:all_tests}
\end{minipage}
\end{figure}

For the period encapsulating the Dot-Com Bubble and the 9/11 attacks, displayed in panel \ref{fig:dotcom}, it is notable that from April to July 2001, the time series for the true portfolio return follows a zig-zag pattern followed by a slight downward trend thereafter, which is captured by JDKF and notably also by Dynamic PCA. Indeed JDKF outperforms all metrics in terms of MAE as per Table \ref{tab:mae}, but Dynamic PCA is a close competitor with respect to this metric. Even still, JDKF achieves approximately a 12.75\% reduction in MAE over Dynamic PCA. Another notable fact from this backtest is that, unlike in the Oil Price shock crisis, Dynamic PCA is performing significantly better than Static PCA in terms of MAE and also time series trend consistency. This could be suggestive of dynamics during this period that are not highly nonlinear and are perhaps also well captured by linear models. 

Finally, it is during the Mortgage crisis period from December 2007 to June 2009, displayed in panel \ref{fig:mortgage}, where JDKF significantly outperforms all benchmarks by a wide margin in terms of both performance metrics. A visually striking indication of JDKF's superiority can be observed during the months of October 2008 to March 2009, where there is a sharp spike down in the true portfolio return, which then roughly recovers towards the end of February 2009. JDKF is the only model to accurately capture this trend, again indicating that JDKF is successfully uncovering the joint dynamics of the factors driving the portfolio's return and the conditional sampling procedure is uncovering their dynamic correlations. SSA again appears to predict spikes with a delay and Dynamic PCA completely confuses the temporal trends. Table \ref{tab:mae} shows JDKF achieving a 39.17\% reduction over the MAE achieved by Static PCA, which was the second best performing benchmark in terms of MAE. JDKF, as shown in Table \ref{tab:accuracy} also achieves higher accuracy than Static PCA 72.22\% of the time. The significantly higher MAE achieved by Dynamic PCA, combined with the obvious confusion in capturing the temporal trend of the true portfolio return and the fact that it performed relatively well during the Dot-Com Bubble, is suggestive of more nonlinearities present in the data during the Subprime Mortgage Crisis, which are being successfully picked up by JDKF due to the diffusion map embeddings.

These backtests provide solid evidence that taking the joint dynamics of the common factors into account result in more accurate scenario analysis. Furthermore, JDKF has provided a simple data-driven framework to define high-dimensional scenarios that are fully interpretable and embed them in a dynamic scenario analysis seamlessly without losing much explanatory power. We summarize the complete MAE numbers in the figures above in the following table:
\begin{table}[H]
    \centering
    \begin{tabular}{lcccc}
        \toprule
        \textbf{Time Period} & \textbf{JDFK}  & \textbf{SSA} & \textbf{Dynamic PCA} & \textbf{Static PCA} \\
        \midrule
        July 1990 - Mar 1991 & \textbf{.0206}  & .0462 & .0658 & .0332\\
        Mar 2001 - Nov 2001 & \textbf{.0219} & .0490 & .0251 & .1523\\
        Dec 2007 - June 2009 & \textbf{.0441} & .0750 & 1.017 & .0725\\
        Jun 2004 - Dec 2016 & \textbf{.0304} & .0474 & .2615 & .0502\\
        \bottomrule
    \end{tabular}
\caption{Mean Absolute Error (MAE) of conditional portfolio return predictions across 
four historical periods. JDKF achieves the lowest MAE in all periods, with improvements 
of 35.86\% (2004-2016), 55.43\% (1990-1991), 12.75\% (2001), and 39.17\% (2007-2009) 
relative to the best-performing benchmark. Dynamic PCA performs poorly during  crisis periods (1990-1991, 2007-2009) but remains competitive during the 
Dot-Com bubble (2001). Static PCA 
consistently outperforms SSA, demonstrating the value of capturing factor co-movements 
even without explicit temporal dynamics. All methods estimate factor loadings 
$\tilde{\mathbf{B}}$ via OLS on rolling windows. JDKF and Dynamic PCA additionally 
employ Kalman filtering in their respective embedding spaces (diffusion coordinates 
vs. principal components).}
\label{tab:mae}
    
    \label{tab:mae}
\end{table}

As defined earlier, we also compute the proportion of time the ratio $E_{\mathrm{JDKF},t}/E_{\mathrm{SSA},t}$ is less than 100\%, which indicates the proportion of time JDKF gave a more accurate scenario portfolio return than SSA. And we similarly compute the proportion of time JDKF is more accurate than the PCA-based benchmarks. The proportions for all of the financial crisis periods against the benchmark models are listed in the following table:
\begin{table}[H]
    \centering
    \begin{tabular}{lccc}
        \toprule
        \textbf{Time Period} & \textbf{$E_{\mathrm{JDKF},t}/E_{\mathrm{SSA},t}$} & \textbf{$E_{\mathrm{JDKF},t}/E_{\mathrm{Dynamic PCA},t}$} & \textbf{$E_{\mathrm{JDKF},t}/E_{\mathrm{Static PCA},t}$}\\
        \midrule
        July 1990 - Mar 1991 & 62.5\% & 62.5\% & 62.5\%\\
        Mar 2001 - Nov 2001 & 62.5\% & 50\% & 100\%\\
        Dec 2007 - June 2009 & 72.22\% & 77.78\% & 72.22\% \\
        Jun 2004 - Dec 2016 & 66\% & 66.67\% & 76.67\% \\
        \bottomrule
    \end{tabular}
\caption{Percentage of prediction periods where JDKF achieves lower absolute error 
than each benchmark. Each cell shows $100 \times \#\{t : E_{\text{JDKF},t} < E_{\text{benchmark},t}\}/T$ 
where $E_{t} = |V_t^{\text{true}} - V_t^{\text{pred}}|$ is the absolute prediction error 
at time $t$. JDKF outperforms benchmarks 62.5\%-77.78\% of the time across different 
crises. Perfect performance (100\% vs. Static PCA in 2001) indicates JDKF was closer to 
realized returns in all 8 months of that period.}
\label{tab:accuracy}
    \label{tab:accuracy}
\end{table}

\section{Conclusions}
\label{sec:conclusion}
We have proposed a novel data-driven framework for dynamic factor modeling within a supervised learning setting where one observes both, a dataset of responses $y(t) \in \mathbb{R}^n$, and a high-dimensional set of covariates $x(t) \in \mathbb{R}^m$. Our framework requires no knowledge of the observation dynamics and is versatile enough that it can simultaneously handle the two sets of observations. The key assumption is that the intrinsic state $\theta$ follows a Langevin equation with certain spectral properties, which we demonstrate is not too restrictive, particularly for financial applications. We uncover the joint dynamics of the covariates and responses in a purely data-driven way via lower-dimensional covariate embeddings that retain  explanatory power while exhibiting simple linear dynamics.

We combine anisotropic diffusion maps with Kalman filtering to infer the latent dynamic covariate embeddings, and predict the response variable directly from the diffusion map embedding space. We develop a conditional sampling procedure that accounts for dynamic correlations between covariates when conditioning on fixed values for a subset of them, and note that the procedure is general enough that it allows for conditioning on subsets of responses as well.
Our approach is justified via theoretical results on the convergence of anisotropic diffusion maps as well as the robustness of approximating the learned diffusion coordinates by linear SDEs. The former result applies concentration inequalities for Markov processes to generalize the arguments of \citet{singer_graph_2006,singer2008nonlinear} to the time series case; the latter adapts arguments typically seen in literature on Markov chain Monte Carlo methods.

We study an application of our framework to financial stress testing, where we consider the special case of a linear multi-factor model. Our historical backtests demonstrate the superiority of our methodology in capturing dynamic correlations between the macroeconomic factors driving the return of equity portfolios. We make use of several comparison benchmarks, including the standard approach to scenario analysis and PCA-based models. We demonstrate through our empirical experiments that working on the diffusion coordinate space indeed uncovers the dynamics of the common factors and captures nonlinearities in the data during historical periods of financial stress. 

While the primary application of our framework in this paper has been financial stress testing, it is important to emphasize the rich set of applications where our framework can be employed. 
First, our dynamic factor modeling framework is well-suited for asset pricing problems, where one seeks to construct a sparse set of dynamic factors from a high-dimensional factor zoo that possess high explanatory power for the cross-section of stock returns. Another application of our framework is to the problem of nowcasting, where one seeks to utilize a higher-frequency dataset to improve predictions on lower-frequency data. 
Another application is \emph{nowcasting} (see \citet{giannone_nowcasting_2008}), where one makes use of higher frequency/quality data to improve predictions on lower frequency/quality data. Our approach is well-suited for such a task as it produces lower-dimensional dynamic embeddings which propagate linearly, and also possess explanatory power comparable to the original high-dimensional set of factors. 

%\newpage
\section*{References}
\printbibliography[heading=none]
\newpage

\appendix

\renewcommand{\thesubsection}{\Alph{section}.\arabic{subsection}}
\setcounter{table}{0}
\setcounter{figure}{0}
\renewcommand{\thetable}{\Alph{section}.\arabic{table}}
\renewcommand{\thefigure}{\Alph{section}.\arabic{figure}}

\begin{center}
\bf{\huge Appendices}
\end{center}

\section{Diffusion Kalman Filter}
\label{sec:appa}

Assuming that the gradients in \eqref{eq:phiSDE} are (approximately) constant, \citet{shnitzer2020diffusion}
define a state-space model for the diffusion coordinates using a Kalman filter ($x_t$ is the measurement and $\psi_t$ diffusion coordinates):
\[\psi_t = (I - \Lambda)\psi_{t-1} + {w}_{t-1}\]
\[x_t = \mathbf{H} \psi_t + {v}_t\]
where $\Lambda$ is the transition matrix for the linear dynamics containing the mapped eigenvalues \newline ${-\veps^{-1}\log \kappa_1,\dots -\veps^{-1}\log \kappa_\ell}$ on its diagonal, and $\mathbf{H}$ is the lifting operator. Note that ${w}$ and ${v}$ are Gaussian noises with covariance matrices $\mathbf{Q}$ and $\mathbf{V}$ respectively.
Denote $\hat{\psi}_t = \mathbb{E} [ \psi_t | \mathcal{F}_{t-1}]$, where $\mathcal{F}_{t-1} = \left \{x_0,\ldots,x_{t-1} \right \}$, and  
$P_t = \Cov(\psi_t | \mathcal{F}_{t-1})$

The Diffusion Kalman Filter (DKF) update equations are (with $\mathbf{A}=I - \Lambda$):
\[\hat{\psi}_t = \mathbf{A}\hat{\psi}_{t-1} + K_t \left( x_t - \mathbf{H} \mathbf{A}\hat{\psi}_{t-1} \right)\]
\[P_t = (I - K_t \mathbf{H}) \left( \mathbf{A} P_t \mathbf{A}^T + \mathbf{Q} \right)\]
\[K_t = \left( \mathbf{A} P_{t-1} \mathbf{A}^T + \mathbf{Q} \right) \mathbf{H}^T \nonumber \times \left( \mathbf{H} \mathbf{A} P_{t-1} \mathbf{A}^T \mathbf{H}^T + \mathbf{H} \mathbf{Q} \mathbf{H}^T + \mathbf{V} \right)^{-1}\]

\section{Kalman Filter and EM Algorithm Details}
\label{app:kalman_em_details}

\subsection{State-Space Equations for JDKF}

We provide the complete Kalman filter equations for the Joint Diffusion Kalman Filter state-space model:
\begin{equation}
\begin{aligned}
\text{State equation:} \quad & \psi_{t+1} = \mathbf{A} \psi_t + w_t, \quad w_t \sim \mathcal{N}(0, \mathbf{Q})\\
\text{Observation equation:} \quad & z_t = \mathbf{H}^z \psi_t + v_t, \quad v_t \sim \mathcal{N}(0, \mathbf{R})
\end{aligned}
\end{equation}
where $z_t = [x_t; y_t]$, $\mathbf{H}^z = [\mathbf{H}^x; \mathbf{H}^y]$, and $\mathbf{R} = \begin{bmatrix} \mathbf{R}_x & \mathbf{R}_{xy} \\ \mathbf{R}^\top_{xy} & \mathbf{R}_y \end{bmatrix}$.

We begin with the forward pass (Kalman Filter).  Initialize with $\hat{\psi}_{0|0}$ and $P_{0|0}$. For $t = 1, \ldots, T$, we have the followins steps:

\subsection{E-Step: Kalman Filter with Joint Observations}
\begin{enumerate}
    \item {Prediction Step:}
    \begin{align*}
    \hat{\psi}_{t|t-1} &= \mathbf{A} \hat{\psi}_{t-1|t-1}\\
    P_{t|t-1} &= \mathbf{A} P_{t-1|t-1} \mathbf{A}^\top + \mathbf{Q}
    \end{align*}
    
    \item Innovation:
    \begin{align*}
    e_t &= z_t - \mathbf{H}^z \hat{\psi}_{t|t-1}\\
    S_t &= \mathbf{H}^z P_{t|t-1} {\mathbf{H}^z}^\top + \mathbf{R}
    \end{align*}
    
    \item {Kalman Gain:}
    \begin{align*}
    K_t = P_{t|t-1} {\mathbf{H}^z}^\top S_t^{-1}
    \end{align*}
    
    \item {Update Step:}
    \begin{align*}
    \hat{\psi}_{t|t} &= \hat{\psi}_{t|t-1} + K_t e_t\\
    P_{t|t} &= (I - K_t \mathbf{H}^z) P_{t|t-1}
    \end{align*}
    
    \item {Log-Likelihood Contribution:}
    \begin{align*}
    \mathcal{L}^{(t)} = -\frac{1}{2} \left( \log |S_t| + e_t^\top S_t^{-1} e_t + (m+n) \log 2\pi \right)
    \end{align*}
\end{enumerate}

Next, we proceed with the backward pass (RTS Smoother). Initialize at $T$ with $\hat{\psi}_{T|T}$ and $P_{T|T}$ from the forward pass. For $t = T-1, \ldots, 1$, we have

\begin{enumerate}
    \item {Smoother Gain:}
    \begin{align*}
    J_t = P_{t|t} \mathbf{A}^\top P_{t+1|t}^{-1}
    \end{align*}
    
    \item {Smoothed State Estimate:}
    \begin{align*}
    \hat{\psi}_{t|T} = \hat{\psi}_{t|t} + J_t \left( \hat{\psi}_{t+1|T} - \hat{\psi}_{t+1|t} \right)
    \end{align*}
    
    \item {Smoothed Covariance:}
    \begin{align*}
    P_{t|T} = P_{t|t} + J_t (P_{t+1|T} - P_{t+1|t}) J_t^\top
    \end{align*}
\end{enumerate}

\subsection{M-Step: Parameter Updates}

Given smoothed state estimates $\{\hat{\psi}_{t|T}\}_{t=1}^T$ from the E-step, the M-step updates are as follows:

\begin{enumerate}
    \item {Update $\mathbf{R}_x$:}
    \begin{align*}
    \mathbf{R}_x = \frac{1}{T} \sum_{t=1}^{T} \left( x_t - \mathbf{H}^x \hat{\psi}_{t|T} \right) \left( x_t - \mathbf{H}^x \hat{\psi}_{t|T} \right)^\top
    \end{align*}
    
    \item {Update $\mathbf{R}_y$:}
    \begin{align*}
    \mathbf{R}_y = \frac{1}{T} \sum_{t=1}^{T} \left( y_t - \mathbf{H}^y \hat{\psi}_{t|T} \right) \left( y_t - \mathbf{H}^y \hat{\psi}_{t|T} \right)^\top
    \end{align*}
    
    \item {Update $\mathbf{R}_{xy}$:}
    \begin{align*}
    \mathbf{R}_{xy} = \frac{1}{T} \sum_{t=1}^{T} \left( x_t - \mathbf{H}^x \hat{\psi}_{t|T} \right) \left( y_t - \mathbf{H}^y \hat{\psi}_{t|T} \right)^\top
    \end{align*}
\end{enumerate}

Consider the linear Multi-Factor Case. In this case,  $\mathbf{H}^y = \mathbf{B}\mathbf{H}^x$, and we additionally update:
\begin{align*}
\mathbf{B} = \left( \sum_{t=1}^{T} y_t (\mathbf{H}^x \hat{\psi}_{t|T})^\top \right) \left( \sum_{t=1}^T \mathbf{H}^x \hat{\psi}_{t|T} (\mathbf{H}^x \hat{\psi}_{t|T})^\top \right)^{-1}
\end{align*}

\section{Joint Diffusion Kalman Filter Explanatory Power}
\label{sec:appb}

\setcounter{table}{0}
\setcounter{figure}{0}
\renewcommand{\thetable}{C.\arabic{table}}
\renewcommand{\thefigure}{C.\arabic{figure}}
The figures below show respectively the individual $R^2$ from regressing each time series of stock returns against the full and reduced sets of factors. Note the significant loss in explanatory power from using the DKF diffusion coordinates, this is to be expected since their estimation process does not jointly consider the variation in the dependent variables.

\begin{figure}[H]
    \centering
    \includegraphics[width=.8\linewidth, trim=0cm 0cm 0cm .8cm, clip]{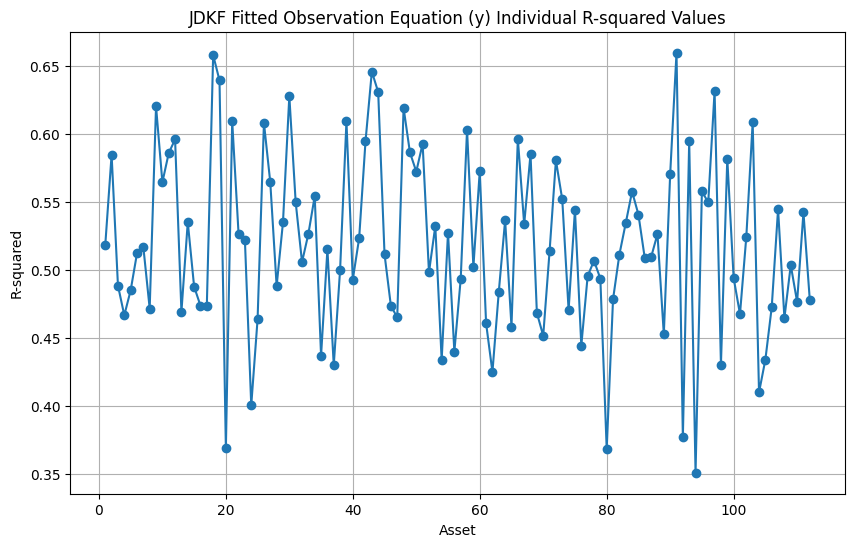}
    \caption{$R^2$ from regressing 132 individual stock returns on JDKF diffusion coordinates ($\ell=39$) via the fitted equation $y_t = \mathbf{B}\mathbf{H}^x\psi_t$. Median $R^2 = 0.52$, demonstrating that dimensionality reduction from 132 macroeconomic factors to 39 coordinates retains substantial predictive information. This contrasts with standard DKF (Figures B.2--B.4), which loses explanatory power by not jointly modeling covariates and responses.}
    \label{fig:regular_figure}
\end{figure}

\begin{landscape}
\begin{figure}[H]
    \centering
    \includegraphics[width=1.3\textwidth, trim=0cm 0cm 0cm 1.7cm, clip]{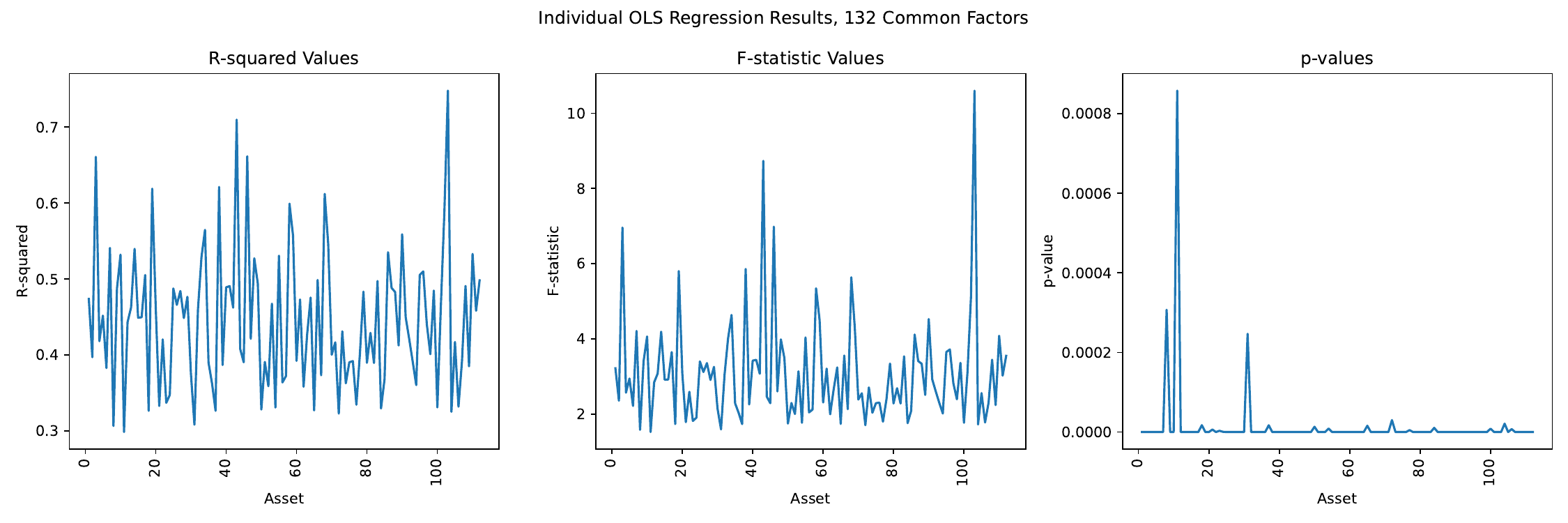} \\ 
    \caption{Plot of the $R^2$ from the OLS Regressions of Individual Stocks vs Original Set of 132 Macro Variables}
\medskip
    \includegraphics[width=1.3\textwidth, trim=0cm 0cm 0cm 1.7cm, clip]{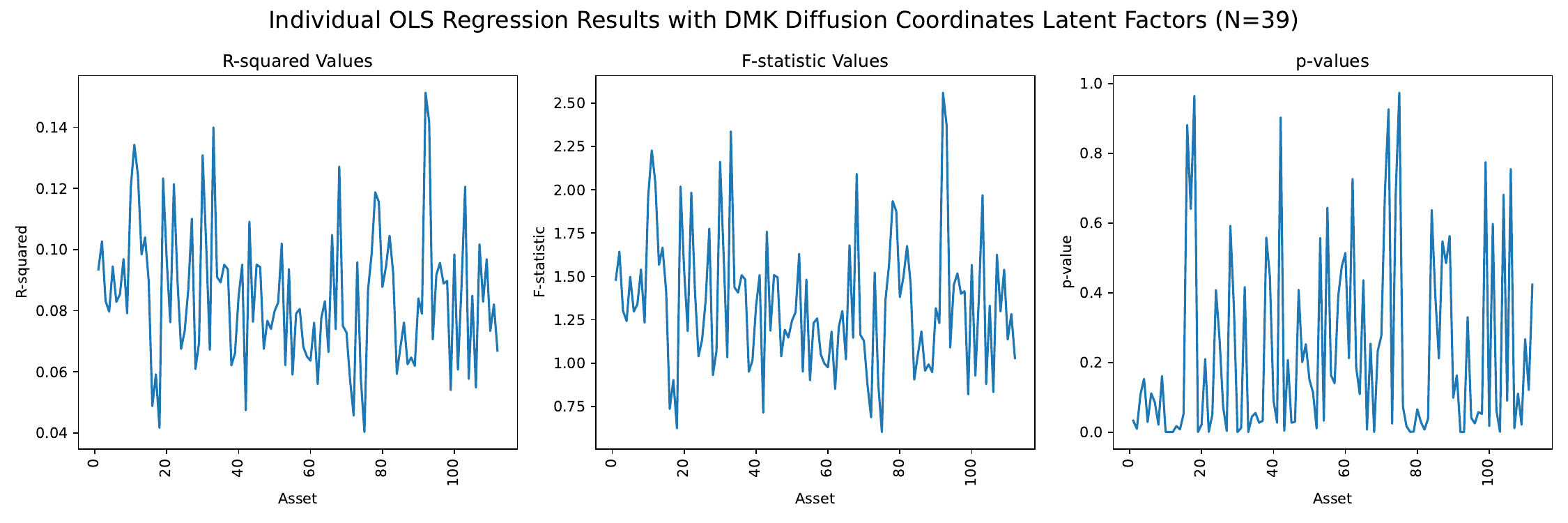} \\ 
    \caption{Plot of the $R^2$ from OLS Regressions of Individual Stocks vs Reduced Set of 39 DKF Diffusion Coordinates}
\end{figure}
\newpage

\begin{figure}   
    \centering    \includegraphics[width=1.3\textwidth, trim=0cm 0cm 0cm 1.5cm, clip]{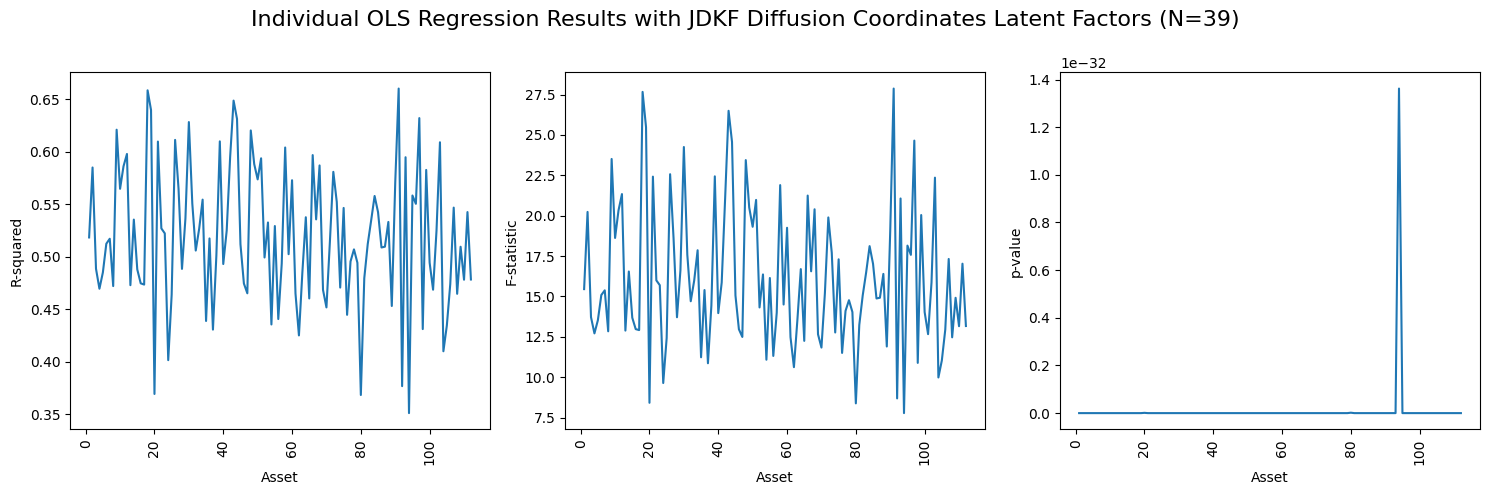} \\ 
    \caption{Results from regressing the individual stock returns against 39 JDKF diffusion coordinates: $y_{j,t} = \beta_j^\top \hat{\psi}_{t, \textrm{JDKF}} + \varepsilon_{j,t}$. {Left}: $R^2$ values range from 0.35 to 0.65, establishing the baseline explanatory power before dimensionality reduction. {Center}: F-statistics are uniformly high, indicating strong joint significance of the factor set. {Right}: All p-values are near zero, confirming statistical significance. These  results provide a benchmark for evaluating how much explanatory power is retained by reduced-dimension embeddings in Figures B.2--B.4. 
    }
\end{figure}
\end{landscape}

\section{Second Moment Computation}
\label{sec:appc}

\setcounter{table}{0}
\setcounter{figure}{0}
\renewcommand{\thetable}{D.\arabic{table}}
\renewcommand{\thefigure}{D.\arabic{figure}}
We provide the full computation for the second moment of $V$ from the proof of Theorem \ref{thm:concentration}. A Taylor approximation following \citet[Equation (28)]{singer2008nonlinear} yields
\begin{align*}
&\E_\mu [u]=(2\pi \veps)^{d/2} \{p(w) f(w)+\frac{\veps}{2}[E(w) p(w) f(w)+ \Delta (p(w) f(w))]+O(\veps^2) \}\\
&\E_\mu [v]=(2\pi \veps)^{d/2} \{p(w)+\frac{\veps}{2}[E(w) p(w) + \Delta p(w)]+O(\veps^2) \}\\
&\E_\mu [u^2]=(\pi \veps)^{d/2} \{p(w) f^2(w)+\frac{\veps}{4}[E(w) p(w) f(w)^2+ \Delta (p(w) f^2(w))]+O(\veps^2) \}\\
&\E_\mu [v^2]=(\pi \veps)^{d/2} \{p(w)+\frac{\veps}{4}[E(w) p(w) + \Delta p(w)]+O(\veps^2) \}\\
&\E_\mu [u v]=(\pi \veps)^{d/2} \{p(w) f(w) +\frac{\veps}{4}[E(w) p(w) f(w)  + \Delta (p(w) f(w))]+O(\veps^2) \}
\end{align*}
where $E$ is a scalar function (see \citet[Paragraph following Equation (28)]{singer2008nonlinear}) which will cancel out in the following computation. In what follows, we will suppress dependence on $w$, and all expectations will be taken with respect to the invariant measure $\mu$.

Now, we define the variable $Y_j$ to be the numerator of $V(\theta(t_j))$:
\[Y_j=\E [v]u(\theta(t_j))-\E [u]v(\theta(t_j))+\alpha \E [v](\E [v]-v(\theta(t_j)))\]
The second moment is given by 
\begin{align*}
\E[Y_j^2]&=\E[v]^2\E[u^2]-2\E[v]\E[u]\E[uv]+\E[u]^2\E[v^2]\\
&\quad+2\alpha \E[v]\lb\{\E[u]\E[v^2]- \E[v]\E[uv]\rb\}+\alpha^2\E[v]^2\lb\{\E[v^2]-\E[v]^2\rb\}
\end{align*}
We are interested in the regime where $\alpha \ll\veps$, hence we neglect the terms with $\alpha$ and $\alpha^2$ when we substitute in the expressions for the moments of $u$ and $v$. We then have:
{\small{
\begin{align*}
&\E Y_j^2=\\
&(2\pi \veps)^{d} \{p+\frac{\veps}{2}[E p + \Delta p]+O(\veps^2) \}^2 (\pi \veps)^{d/2} \{p f^2+\frac{\veps}{4}[E p f^2 + \Delta (p f^2)]+O(\veps^2) \}\\&-2(2\pi \veps)^{d}(\pi \veps)^{d/2} \{p f+\frac{\veps}{2}[E p f+ \Delta (p f)]+O(\veps^2) \} \{p+\frac{\veps}{2}[E p + \Delta p]+O(\veps^2) \} \{p f +\frac{\veps}{4}[E p f  + \Delta (p f)]+O(\veps^2) \}\\
&+(2\pi \veps)^{d}(\pi \veps)^{d/2} \{p f+\frac{\veps}{2}[E p f+ \Delta (p f)]+O(\veps^2) \}^2 \{p+\frac{\veps}{4}[E p + \Delta p]+O(\veps^2) \} 
\end{align*}
}}
To shorten notation, set $H(g)=Eg+\Delta g$ for any given $g$. We then have:
\begin{align*}
\E Y_j^2&=2^d (\pi \veps)^{3d/2}\lb(\{p+\frac{\veps}{2} H (p)\}^2 \{p f^2+\frac{\veps}{4} H(p f^2)\}-2 \{p+\frac{\veps}{2}H(p)\} \{pf+\frac{\veps}{2}H(pf)\} \{p f +\frac{\veps}{4}H(p f)\}\rb.\\&\lb.\quad+\{p f+\frac{\veps}{2}H(p f)\}^2 \{p+\frac{\veps}{4}H(p)\}\rb)+\textrm{h.o.t.}\\&=2^d (\pi \veps)^{3d/2}\lb(\{p^2\frac{\veps}{4} H (pf^2)+2p^2f^2\frac{\veps}{2} H(p)\}-2\{p^2f^2+\frac{\veps}{2} H(p)+p^2f\frac{\veps}{2}H(pf)+p^2f\frac{\veps}{4}H(pf)\}\rb.\\&\lb.\quad+ \{2p^2 f\frac{\veps}{2}H(p f)+p^2f^2\frac{\veps}{4}H(p)\}\rb)+\textrm{h.o.t.}\\&=2^d (\pi \veps)^{3d/2}\lb(\frac{\veps}{4}p^2H (pf^2)+(\veps-\veps+\frac{\veps}{4})p^2f^2H(p)+(-\veps-\frac{\veps}{2}+\veps)p^2fH(pf)\rb)+\textrm{h.o.t.}\\&=2^d (\pi \veps)^{3d/2}\frac{\veps}{4}\lb(p^2H (pf^2)+p^2f^2H(p)-2p^2fH(pf)\rb)+\textrm{h.o.t.}
\end{align*}
Looking at the term in parenthesis and using the identity $\Delta(gh)=g\Delta h+h\Delta g+2\nabla g\cdot\nabla h$ gives
\begin{align*}
    &p^2H (pf^2)+p^2f^2H(p)-2p^2fH(pf)\\
    &=p^2(E pf^2+\Delta (pf^2)+E p f^2+f^2 \Delta p-2E p f^2-2f\Delta (p f))\\&=p^2(\Delta (pf^2)-2f\Delta (pf)+f^2\Delta p)\\&=p^2(p\Delta f^2+f^2 \Delta p +2 \nabla p \cdot \nabla f^2+f^2 \Delta p-2f(p\Delta f + f \Delta p+2 \nabla p \cdot \nabla f))\\&=p^2(p\Delta f^2-2fp\Delta f + 2\nabla p \cdot \nabla f^2-4f \nabla p \cdot \nabla f)\\&=p^2(\Delta (p f^2)-f^2 \Delta p-2f(\Delta(pf)-f\Delta p)).
\end{align*}
Hence,
\begin{align*}
    &\E[Y_j^2]=2^d (\pi \veps)^{3d/2}\frac{\veps}{4}p^2(\Delta (p f^2)-2f \Delta (pf)+f^2\Delta p)+\textrm{h.o.t.}
\end{align*}
Noting now that $\Delta f^2=2f \Delta f+2 \nabla f \cdot \nabla f$ and $\nabla f^2=2f \nabla f$, we have that:
\begin{align*}
\Delta (p f^2)=p \Delta f^2+f^2 \Delta p+2 \nabla\ p \cdot \nabla f^2=2p f \Delta f+2 p \nabla f \cdot\nabla f +f^2 \Delta p+4f \nabla p \cdot\nabla f
\end{align*}
and 
\[2f\Delta (p f)=2f^2 \Delta p+2 f p \Delta f +4 f \nabla p\cdot \nabla f\]
Hence,
\begin{align*}
\E[Y_j^2]=2^d (\pi \veps)^{3d/2}\lb(\frac{\veps}{2}p^3 \|\nabla f \|^2+O(\veps^2)\rb)
\end{align*}
Finally, we divide by $\E[v]^4$ to obtain
\begin{align*}
    &\E_{\theta_i\sim\mu} [V(\theta_i)^2]=\frac{2^d (\pi \veps)^{3d/2}\lb(\frac{\veps}{2}p^3\|\nabla f\|^2+O(\veps^2)\rb)}{(2\pi \veps)^{2d}p^4}=\frac{2^{-d} \pi^{-d/2} \veps^{-d/2}\lb(\veps\|\nabla f(w)\|^2+O(\veps^2)\rb)}{p(w)}
\end{align*}

\section{Dynamic PCA Kalman Filter Equations}
\label{sec:appd}
\setcounter{table}{0}
\setcounter{figure}{0}
\renewcommand{\thetable}{D.\arabic{table}}
\renewcommand{\thefigure}{D.\arabic{figure}}
We provide the standard EM algorithm procedure and Kalman filter equations for the Dynamic PCA benchmark outlined in Section \ref{sec:stress_benchmarks}.

\subsection{E-Step: Kalman Filter with Joint Observations}

The Kalman filter is used to estimate the latent state $\mathbf{Z}_t$ and the corresponding state covariance $P_{t|t}$. Note that for Dynamic PCA we utilize the standard Kalman filter and EM algorithm formulations where we only observe a single set of measurements $\mathbf{X_t}$ . The transition matrix $\mathbf{A}$ is known and estimated by fitting a VAR(1) process to the empirical PC components computed from the data, and the state covariance matrix $\mathbf{Q}$ is computed as the empirical covariance matrix of these PC components. The observation equation is:
\[\mathbf{X}_t = \mathbf{\Gamma} \mathbf{Z}_t + \veps_t \sim \mathcal{N}(0, \mathbf{R})\]
and the state transition equation is:
\[\quad \mathbf{Z}_t = \mathbf{A} \mathbf{Z}_{t-1} + \eta_t, \quad \eta_t \sim \mathcal{N}(0, \mathbf{Q})\]

The steps of the Kalman filter are as follows:
\begin{enumerate}
    \item Prediction Step:
    \[\hat{\psi}_{t|t-1} = \mathbf{A} \hat{\mathbf{Z}}_{t-1|t-1}, \quad P_{t|t-1} = \mathbf{A} P_{t-1|t-1} \mathbf{A}^\top + \mathbf{Q}\]
    \item Innovation (Residual):
    \[\text{innovation}_t = \mathbf{X}_t - \mathbf{\Gamma} \hat{\mathbf{Z}}_{t|t-1}\]

\item Innovation Covariance:
\[S_t = \mathbf{\Gamma} P_{t|t-1} \mathbf{\Gamma}^\top + \mathbf{R}\]
\item Kalman Gain:
\[K_t = P_{t|t-1} \mathbf{\Gamma}^\top S_t^{-1}\]
\item Update Step:
\[\hat{\mathbf{Z}}_{t|t} = \hat{\mathbf{Z}}_{t|t-1} + K_t \text{innovation}_t, \quad P_{t|t} = (I - K_t \mathbf{\Gamma}) P_{t|t-1}\]
\item Log-Likelihood Contribution:
\[\mathcal{L}_t = -\frac{1}{2} \left( \log |S_t| + \text{innovation}_t^\top S_t^{-1} \text{innovation}_t + p \log 2\pi \right)\]
where $p$ is the dimension of $\mathbf{X}_t$
\end{enumerate}

After running the Kalman filter, we apply the Rauch-Tung-Striebel (RTS) smoother:
\begin{enumerate}
    \item Initialize at \( T \):
    \[
    \hat{\mathbf{Z}}_{T|T} = \hat{\mathbf{Z}}_{T|T}, \quad P_{T|T} = P_{T|T}
    \]
    \item Backward pass for \( t = T-1, ..., 1 \):
    \[
    J_t = P_{t|t} \mathbf{A}^\top P_{t+1|t}^{-1}
    \]
    \[
    \hat{\mathbf{Z}}_{t|T} = \hat{\mathbf{Z}}_{t|t} + J_t \left( \hat{\mathbf{Z}}_{t+1|T} - \hat{\mathbf{Z}}_{t+1|t} \right)
    \]
    \[
    P_{t|T} = P_{t|t} + J_t (P_{t+1|T} - P_{t+1|t}) J_t^\top
    \]
\end{enumerate}

\subsection{M-Step: Parameter Updates}

After running the Kalman filter, we update the parameter $\mathbf{R}$ based on the current estimates of $\mathbf{Z}_t$.

\begin{enumerate}
    \item Update $\mathbf{R}$ using the residuals from the smoothed estimates $\hat{\mathbf{Z}}_{t|T}$:
    \[\mathbf{R} = \frac{1}{T} \sum_{t=1}^{T} \left( \mathbf{X}_t - \mathbf{\Gamma} \hat{\mathbf{Z}}_t \right) \left( \mathbf{X}_t - \mathbf{\Gamma} \hat{\mathbf{Z}}_t \right)^\top
\]

\end{enumerate}

\section{Implementation Details for Benchmark Methods}
\label{sec:benchmark_impl}

\subsection{Standard Scenario Analysis (SSA)}
\label{app:ssa}
Under the standard multi-factor model for stock returns, if $\mathbf{X_t} \in \mathbb{R}^p$ is the $p$-dimensional vector of common factors, we define $\mathcal{S}$ to be the set containing the indices of the factors in a scenario (the factors we intend to stress), and thus $\mathcal{S}^C$ is the index set of those factors we leave un-stressed and which are not in our scenario. SSA then stresses the components of $\mathbf{X}_{\mathcal{S},t}$ according to a given scenario (e.g. $+20$ on the S\&P index, $-2$ on the CPI index, and $+5$ on the US Dollar/Euro exchange rate), and keeps the components in $\mathbf{X}_{\mathcal{S}^C,t}$ unchanged (i.e. equal to their current value). The new portfolio P\&L or overall return $V_t$ is then computed with $\mathbf{Y}_t$ determined by the scenario and the multi-factor model. By leaving $\mathbf{X}_{\mathcal{S},t}$ unchanged, SSA implicitly assumes that $\mathbb{E}[\mathbf{X}_{\mathcal{S}^C, t+1}|\mathcal{F}_t, \mathbf{X}_{\mathcal{S}, t+1}]= 0$, here $\mathcal{F}_t$ denotes a filtration with information up to time $t$. 
\begin{enumerate}

\item Let $\Delta \mathbf{X}_{\mathcal{S}, t+1}=\mathbf{X}_{\mathcal{S},t+1}-\mathbf{X}_{\mathcal{S}, t}$ denote the $t+1$ scenario stress vector $\in \mathbb{R}^{|\mathcal{S}|}$. 

\item Compute the SSA factor change vector, $\Delta \mathbf{X}^{\mathrm{SSA}}_{t+1}$ as:
\[
\Delta \mathbf{X}^{\mathrm{SSA}}_{i, t+1} =
\begin{cases}
    \Delta \mathbf{X}_{\mathcal{S},t+1}, & \text{for } i \in \mathcal{S}  \\
    0, & \text{for } i \in \mathcal{S}^C 
\end{cases}
\]

\item Obtain the predicted factor vector under SSA 

\[\mathbf{X}_{t+1}^{\mathrm{SSA}}=\mathbf{X}_t+\Delta\mathbf{X}^{\mathrm{SSA}}_{t+1}\]

\item From the fitted matrix of factor loadings $\mathbf{\tilde{B}}$, fit via a multivariate regression, predict post-stress asset returns as:

\[\mathbf{Y}^{\mathrm{SSA-stress}}_{t+1}=\tilde{\mathbf{B}} \mathbf{X}_{t+1}^{\mathrm{SSA}}\]

\end{enumerate}

We summarize the SSA procedure in the following algorithm, which we will then feed into a historical rolling window backtest in the next section. We let $s$ denote the size of the rolling window. Note that the fitted matrix of factor loadings, $\mathbf{\tilde{B}}$ is treated as an input in this algorithmic format. We denote with $x_{t-s:t} \in \mathbb{R}^{s \times p}$ the matrix of common factors from time $t-s$ up to $t$, and with $y_{t-s:t} \in \mathbb{R}^{s \times n}$ the matrix of asset returns from time $t-s$ up to $t$. We take $x_{\mathrm{actual},\mathcal{S}, t+1}$ to be the realized historical scenario.

% Algorithm 4: SSA Simulation
\begin{algorithm}[H]
\caption{SSA Portfolio Simulation} \label{alg:ssa}
\begin{algorithmic}[1]
\Require $x_{t-s:t}, x_{t+1}, \mathbf{\tilde{B}}$ (fitted from $(y_{t-s:t}, x_{t-s:t})$)
\State Set $x_{\mathcal{S},t+1}=x_{\mathrm{actual},\mathcal{S}, t+1}$ and $x_{\mathcal{S}^C,t+1}=x_{\mathcal{S}^C,t}$ to form $x_{t+1, \mathrm{SSA}}$
\State Compute $y_{t+1, \mathrm{SSA}}$ from fitted equation $y_t = \mathbf{\tilde{B}} x_t$ given $x_{t+1, \mathrm{SSA}}$
\State Compute the $1/N$ (or Markowitz) weights $w$
\State Compute portfolio return $\hat{V}_{t+1, \mathrm{SSA}}=w \cdot y_{t+1,\mathrm{SSA}}$
\end{algorithmic}
\end{algorithm}

\subsection{Static PCA}
\label{app:static_pca}
Given a factor matrix \( \mathbf{X} \in \mathbb{R}^{T \times p} \) and a return matrix \( \mathbf{Y} \in \mathbb{R}^{T \times n} \), the static PCA procedure comprises a series of steps. First, we estimate the linear multi-factor regression model, then we project the factor changes onto principal components, apply stress scenarios, and predict the corresponding returns from the fitted factor model.

\begin{enumerate}

    \item \textbf{Preprocessing and Centering.} \\
    First we compute the differenced factor matrix containing the time $t$ factor changes:
    \[
    \Delta \mathbf{X}_t = \mathbf{X}_{t} - \mathbf{X}_{t-1}.
    \]
    We compute the mean of the differenced data, defined as \( \mu_{\Delta \mathbf{X}}=\frac{1}{T} \sum_{t=1}^T \Delta \mathbf{X}_t \), and subtract it to obtain the centered factors upon which we will apply PCA:
    \[
    \tilde{\mathbf{X}}_t = \Delta \mathbf{X}_t - \mu_{\Delta \mathbf{X}}.
    \]

    \item \textbf{Estimation of the Multivariate Regression Model.} \\
    We fit the multivariate regression model to obtain $\mathbf{\tilde{B}}$:
    \[
    \mathbf{Y}_{t} = \mathbf{X}_{t} \mathbf{B} + \mathbf{\veps}_t,
    \]

    \item \textbf{Principal Component Analysis (PCA).} \\
    For each specified PCA dimension \( d \), we compute the PCA transformation of the centered factors on $\tilde{\mathbf{X}}_t$:
    \[
    \mathbf{X}_{\text{PCA}} = U \Sigma V^T.
    \]
     and extract the principal component loading matrix:
    \[
    W = V^T.
    \]

    \item \textbf{Applying Stressed Scenarios in PCA Space.} \\
    We define the unstressed scenario by setting the un-stressed factors equal to their current value at time $t$ (meaning no change as in SSA):
    \[
    \mathbf{X}_{t+1, i}^{\text{scenario}} = \mathbf{X}_{t, i}, \quad \forall i \not \in \mathcal{S}.
    \]
    For the specified set of stressed scenario variables \( \mathcal{S} \), set the $t+1$ values equal to their realized historical level, meaning we consider the true historical stress:
    \[
    \mathbf{X}_{t+1, i}^{\text{scenario}} = \mathbf{X}_{t+1, i}, \quad \forall i \in \mathcal{S}.
    \]
    Compute the perturbation vector (whose elements corresponding to un-stressed factors will be zero and those corresponding to the stressed factors will be equal to the true historical change):
    \[
    \Delta \mathbf{X} = \mathbf{X}_{t+1}^{\text{scenario}} - \mathbf{X}_t.
    \]
    Center the perturbation vector:
    \[
    {\Delta \tilde{\mathbf{X}}} = \Delta \mathbf{X} - \mu_{\Delta \mathbf{X}}.
    \]
    Project the stress perturbation onto the PCA component space:
    \[
    \text{proj}_{\text{PCA}} = W (W^T {\Delta \tilde{\mathbf{X}}}).
    \]
    Obtain the perturbed scenario factor vector:
    \[
    \mathbf{X}_{t+1}^{\text{PCA-stress}} = \mathbf{X}_t + \text{proj}_{\text{PCA}} + \mu_{\Delta \mathbf{X}}.
    \]

    \item \textbf{Prediction of Stressed Returns.} \\
    Using the estimated regression factor loading matrix $\tilde{\mathbf{B}}$, we predict the asset returns under the stressed scenario:
    \[
    \mathbf{Y}_{t+1}^{\text{PCA-stress}} = \mathbf{X}_{t+1}^{\text{PCA-stress}} \mathbf{\tilde{B}}.
    \]

\end{enumerate}

We summarize the static PCA procedure in the following algorithm, which we will feed into our historical rolling window backtest in the next section. We let $s$ denote the size of the rolling window. Note that the fitted matrix of factor loadings, $\mathbf{\tilde{B}}$ is treated as an input in this algorithmic format. We denote with $x_{t-s:t} \in \mathbb{R}^{s \times p}$ the matrix of common factors from time $t-s$ up to $t$, and with $y_{t-s:t} \in \mathbb{R}^{s \times n}$ the matrix of asset returns from time $t-s$ up to $t$. We take $x_{\mathrm{actual},\mathcal{S}, t+1}$ to be the realized historical scenario. For notational brevity, in the algorithm below we also assume that the mean of the factor changes, $\mu_{\Delta \mathbf{X}}$,  is already added back into $\text{proj}_{\text{PCA}}$

\begin{algorithm}[H]
\caption{Static PCA Portfolio Simulation} \label{alg:static_pca}
\begin{algorithmic}[1]
\Require $x_{t-s:t}, x_{t+1}, \mathbf{\tilde{B}}$ (fitted from $(y_{t-s:t}, x_{t-s:t})$)
\State Apply PCA on the (pre-processed) differenced data $\Delta x_{t-s:t}$
\State Set $x_{\mathcal{S},t+1}=x_{\mathrm{actual},\mathcal{S}, t+1}$ and $x_{\mathcal{S}^C,t+1}=x_{\mathcal{S}^C,t}$ to form $\tilde{x}_{t+1}$
\State Obtain the change vector $\Delta x_t=\tilde{x}_{t+1}-x_t$
\State Compute the projection of the centered $\Delta x_t$ onto the PCA component space: $\text{proj}_{\text{PCA}}$
\State Obtain the perturbed scenario vector $x_{t+1,\text{Static PCA}}=x_t+\text{proj}_{\text{PCA}}$
\State Compute $y_{t+1, \text{Static PCA}}$ from fitted equation $y_t = \mathbf{\tilde{B}} x_{t+1, \text{Static PCA}}$
\State Compute the $1/N$ (or Markowitz) weights $w$
\State Compute portfolio return $\hat{V}_{t+1, \text{Static PCA}}=w \cdot y_{t+1,\text{Static PCA}}$
\end{algorithmic}
\end{algorithm}

\subsection{Dynamic PCA}  
\label{app:dynamic_pca}
The procedure consists of the following steps:

\begin{enumerate}

\item We estimate the transition matrix $\mathbf{A}$ by fitting a VAR(1) process to the PC components computed from the data. 

\item From the state space representation we then implement the standard Kalman filter (see Appendix \ref{sec:appd}) to obtain our dynamic estimates $\mathbf{Z}_t$ of the PC embeddings. In our empirical experiments we select the dimension of the PC embeddings as $d=5$ such that $>99\%$ of the variance in our data is explained by them. 

\item Once we have obtained our estimates for $\mathbf{Z}_t$, we perform conditional sampling given a scenario to obtain the stressed factor vector. The conditional sampling procedure closely follows that outlined in Section \ref{sec:sample} with the exception that we replace the lifting operator matrix with $\mathbf{\Gamma}$ and the diffusion coordinates with $\mathbf{Z}_t$:

\[\mu_{\text{cond, PCA}} = \mathbf{\Gamma}_{\mathcal{S}^C}^{+} \mu_{\mathrm{cond}}\]
\[\Sigma_{\text{cond, PCA}} = \mathbf{\Gamma}_{\mathcal{S}^C}^{+} \Sigma_{\mathrm{cond}} (\mathbf{\Gamma}_{\mathcal{S}^C}^{+})^T\]

we then sample from the conditional multivariate normal distribution to get $\mathbf{Z}_{t+1}$ given the scenario in the original space:
\[\mathbf{Z}_{t+1} \mid \mathbf{Z}_t, \mathbf{X}_{\mathcal{S}, t} \sim \mathcal{N}(\mu_{\text{cond, PCA}}, \Sigma_{\text{cond, PCA}})\]

We obtain $K$ Monte Carlo conditional samples:
\[ \{\mathbf{Z}^{(k)}_{t+1} \mid \mathbf{Z}_t, \mathbf{X}_{\mathcal{S}, t}\}_{k=1}^K\]

\item For each conditional Monte Carlo sample, predict post-stress asset returns under the dynamic PCA stress using the fitted matrix of factor loadings $\mathbf{\tilde{B}}$:
\[\mathbf{Y}_{t+1, \text{Dynamic PCA}}^{(k)}=\mathbf{\tilde{B}}\mathbf{\Gamma}^T \mathbf{Z}_{t+1}^{(k)}\]

\item Obtain the dynamic PCA post-stress returns by taking the average across all Monte Carlo predictions:
\[\mathbf{Y}_{t+1, \text{Dynamic PCA}}=\frac{1}{K} \sum_{k=1}^K \mathbf{Y}_{t+1, \text{Dynamic PCA}}^{(k)}\]
\end{enumerate}

We summarize the dynamic PCA procedure post-Kalman filter fitting, in the following algorithm, which we will feed into a historical rolling window backtest in the next section. We let $s$ denote the size of the rolling window. Note that the fitted matrix of factor loadings, $\mathbf{\tilde{B}}$, the matrix of eigenvectors $\mathbf{\Gamma}$, and the time $t$ Kalman filter estimate $z_t$ for the latent state are all treated as inputs in this algorithmic format. We denote with $x_{t-s:t} \in \mathbb{R}^{s \times p}$ the matrix of common factors from time $t-s$ up to $t$, and $y_{t-s:t} \in \mathbb{R}^{s \times n}$ the matrix of asset returns from time $t-s$ up to $t$. 

\begin{algorithm}[H]
\caption{Dynamic PCA Portfolio Simulation} \label{alg:dynamic_pca}
\begin{algorithmic}[1]
\Require $K, z_t, x_{t+1}, \mathbf{\Gamma}, \mathbf{\tilde{B}}$ (fitted from $(y_{t-s:t}, x_{t-s:t})$)
\For{$k \gets 1$ to $K$} 
    \State Generate $z_{t+1}^{(k)} \mid (z_t, x_{\mathcal{S},t+1})$ 
    \State Compute $y_{t+1, \text{Dynamic PCA}}^{(k)}=\mathbf{\tilde{B}} \mathbf{\Gamma}^T z_{t+1}^{(k)}$
    \State Compute the $1/N$ (or Markowitz) weights $w^{(k)}$
    \State Compute portfolio return $V_{t+1, \text{Dynamic PCA}}^{(k)}=w^{(k)} \cdot y_{t+1,\text{Dynamic PCA}}^{(k)}$
\EndFor
\State Compute $\hat{V}_{t+1, \text{Dynamic PCA}}=K^{-1}\sum_{i=1}^K V_{t+1,\text{Dynamic PCA}}^{(k)}$
\end{algorithmic}
\end{algorithm}

\subsection{Joint Diffusion Kalman Filter (JDKF)}
\label{app:jdkf_implementation}
We describe the implementation details of the JDKF algorithm:

\begin{enumerate}
    \item We first apply diffusion maps on the common factors data matrix $\mathbf{X} \in \mathbb R^{T \times p}$ and obtain the matrix of diffusion coordinates $\mathbf{\psi} \in \mathbb{R}^{T \times \ell}$. We then obtain the lifting operator $\mathbf{H}_X=\mathbf{X}^T \mathbf{\psi} \in \mathbb{R}^{p \times \ell}$ and the transition matrix $\mathbf{A}=(I+\Lambda) \in \mathbb{R}^{\ell \times \ell}$. Recall that $\ell$ denotes the dimension of the diffusion coordinate space, which we choose from observing the largest eigenvalue drop when the diffusion maps algorithm is applied on the factor data. 
    \item From the state space equations, we then estimate the JDKF parameters, which include $\mathbf{B}$, $\mathbf{R}_X$, $\mathbf{R}_Y$, and $\mathbf{R}_{XY}$ by fitting the joint Kalman filter as outlined in Section \ref{sec:jdkf_model} to obtain filtered dynamic diffusion coordinate estimates $\psi_{t}^{\mathrm{JDKF}}$, from which we also obtain $\mathbf{X}_{t}^{\mathrm{JDKF}}=\mathbf{H}_X \psi_{t}^{\mathrm{JDKF}}$
    \item Using the estimated diffusion coordinates from the JDKF $\psi_t^{\mathrm{JDKF}}$, we then perform conditional sampling given the scenario in the factor observation space. Recall that we split our time $t$ factor vector $\mathbf{X}_{t}$ into two parts, $\mathbf{X}_{\mathcal{S}, t}$ and $\mathbf{X}_{\mathcal{S}^C, t}$ based on the stressed and un-stressed factors. Following the conditional sampling procedure outlined in Section \ref{sec:sample}:

    \[\mu_{\mathrm{cond, diff}} = \mathbf{H}_{\mathcal{S}^C}^{+} \mu_{\mathrm{cond}}\]
\[\Sigma_{\mathrm{cond, diff}} = \mathbf{H}_{\mathcal{S}^C}^{+} \Sigma_{\mathrm{cond}} (\mathbf{H}_{\mathcal{S}^C}^{+})^T\]

we then sample from the conditional multivariate normal distribution to get $\psi_{t+1}$ given the scenario in the original space:
\[\psi_{t+1} \mid \psi_t, \mathbf{X}_{\mathcal{S}, t} \sim \mathcal{N}(\mu_{\mathrm{cond, diff}}, \Sigma_{\mathrm{cond, diff}})\]

We obtain $K$ Monte Carlo conditional samples:
\[ \{\psi^{(k)}_{t+1} \mid \psi_t, \mathbf{X}_{\mathcal{S}, t}\}_{k=1}^K\]
    
    \item For each conditional Monte Carlo sample, predict post-stress asset returns under the JDKF stress directly from the fitted observation equation:
\[\mathbf{Y}_{t+1, \mathrm{JDKF}}^{(k)}=\mathbf{B}\mathbf{H}_X \psi_{t+1}^{(k)}\]

    \item Obtain the JDKF post-stress returns by taking the average across all Monte Carlo predictions:
\[\mathbf{Y}_{t+1, \mathrm{JDKF}}=\frac{1}{K} \sum_{k=1}^K \mathbf{Y}_{t+1, \mathrm{JDKF}}^{(k)}\]

\end{enumerate}

We summarize the JDKF procedure post-Kalman filter fitting, in the following algorithm, which we will feed into a historical rolling window backtest in the next section. We let $s$ denote the size of the rolling window. Note that the fitted matrix of factor loadings, $\mathbf{B}$, the linear lifting operator $\mathbf{H}_X$, and the time $t$ Kalman filter estimate $\psi_t$ for the latent state are all treated as inputs to the algorithm. We denote with $x_{t-s:t} \in \mathbb{R}^{s \times p}$ the matrix of common factors from time $t-s$ up to $t$, and $y_{t-s:t} \in \mathbb{R}^{s \times n}$ the matrix of asset returns from time $t-s$ up to $t$. 

\begin{algorithm}[H]
\caption{JDKF Portfolio Simulation}\label{alg:jdkf}
\begin{algorithmic}[1]
\Require $K, \hat{\psi}_t, \mathbf{B}, \mathbf{H}_x, x_{t+1}$
\For{$k \gets 1$ to $K$} 
    \State Generate $\hat{\psi}_{t+1}^{(k)} \mid (\hat{\psi}_t, x_{\mathcal{S},t+1})$ 
    \State Compute $y_{t+1, \mathrm{JDKF}}^{(k)}=\mathbf{B} \mathbf{H}_x\hat{\psi}_{t+1}^{(k)}$
    \State Compute the $1/N$ (or Markowitz) weights $w^{(k)}$
    \State Compute portfolio return $V_{t+1, \mathrm{JDKF}}^{(k)}=w^{(k)} \cdot y_{t+1,\mathrm{JDKF}}^{(k)}$
\EndFor
\State Compute $\hat{V}_{t+1, \mathrm{JDKF}}=K^{-1}\sum_{i=1}^K V_{t+1,\mathrm{JDKF}}^{(k)}$
\end{algorithmic}
\end{algorithm}

\subsection{Historical Backtesting Algorithm}
\label{app:backtest_algorithm}

The complete historical rolling backtest procedure with periodic Kalman filter refitting is as follows:

\begin{algorithm}[H]
\caption{Historical Rolling Back-Test SSA vs JDKF with Periodic Re-Fitting} \label{alg:historical_backtest}
\begin{algorithmic}[1]
\Require $s, T, K, R$
\newline
\Comment{$s = \#$ periods in rolling window for model training} 
\newline
\Comment{$T = \#$ periods in the data} 
\newline
\Comment{$K = \#$ Monte-Carlo samples used to estimate factor model-based scenario Portfolio Return}
\newline
\Comment{$R = \#$ months before re-fitting the Kalman filter}
\newline
\Comment{$r$ is the start index for refitting}
\newline
\Comment{$r_{\text{end}}$ is the end index for refitting}
\For{$r \gets  0:R:T - 1$}
    \State $r_{\text{end}} = \min(r + R + s, T)$
    \State Fit JDKF parameters $(\Lambda, \mathbf{H}^x, \mathbf{R}_x, \mathbf{R}_y, \mathbf{R}_{xy}, \mathbf{Q}, \mathbf{B})$ from $(x_{r:r_{\text{end}}}, y_{r:r_{\text{end}}})$
    \State Fit Dynamic PCA Kalman Filter parameters $(\mathbf{A}, \mathbf{\Gamma}, \mathbf{R}, \mathbf{Q})$ from $(x_{r:r_{\text{end}}}, y_{r:r_{\text{end}}})$
    \For{$t \gets \max(s, r + s)$ to $r_{\text{end}}$}
        \State \textbf{Retrieve Precomputed Kalman Filter Estimates}
        \State Retrieve $\hat{\psi}_t$, $\hat{z}_t$ from Kalman filter estimates
        \State \textbf{Estimate Multivariate Regression Model}
        \State Fit $\mathbf{\tilde{B}}$ using $(y_{t-s:t}, x_{t-s:t})$
        \State \textbf{Compute Portfolio Simulations}
        \State $\hat{V}_{t+1, \mathrm{JDKF}} \gets \text{JDKF\_Simulation}(K, \psi_t, \mathbf{B}, \mathbf{H}_x, x_{,t+1})$......(\textbf{Algorithm} \ref{alg:jdkf})
        \State $\hat{V}_{t+1, \text{Dynamic PCA}} \gets \text{DynamicPCA\_Simulation}(K, z_t, \mathbf{\tilde{B}}, \mathbf{\Gamma}, x_{t+1})$......(\textbf{Algorithm} \ref{alg:dynamic_pca})
        \State $\hat{V}_{t+1, \text{Static PCA}} \gets \text{StaticPCA\_Simulation}(x_{t-s:t}, x_{t+1}, \mathbf{\tilde{B}})$......(\textbf{Algorithm} \ref{alg:static_pca})
        \State $\hat{V}_{t+1, \mathrm{SSA}} \gets \text{SSA\_Simulation}(x_t, x_{t+1}, \mathbf{\tilde{B}})$......(\textbf{Algorithm} \ref{alg:ssa})
        \State Compute true portfolio return $V_{t+1,\mathrm{TRUE}}$
        \State \textbf{Compute Errors}
        \State $E_{\mathrm{JDKF},t+1} = |\hat{V}_{t+1, \mathrm{JDKF}} - \hat{V}_{t+1, \mathrm{TRUE}}|$
        \State $E_{\text{Dynamic PCA},t+1} = |\hat{V}_{t+1, \text{Dynamic PCA}} - \hat{V}_{t+1, \mathrm{TRUE}}|$
        \State $E_{\text{Static PCA},t+1} = |\hat{V}_{t+1, \text{Static PCA}} - \hat{V}_{t+1, \mathrm{TRUE}}|$
        \State $E_{\mathrm{SSA},t+1} = |\hat{V}_{t+1, \mathrm{SSA}} - \hat{V}_{t+1, \mathrm{TRUE}}|$
    \EndFor
\EndFor
\end{algorithmic}
\end{algorithm}

In all of our experiments in Section \ref{sec:experiments}, we use $K=10,000$ Monte Carlo simulations of conditional portfolio returns for the probabilistic models JDKF and Dynamic PCA. Since we are working under the special linear multi-factor model case of our JDKF framework, for each refitting period we estimate $\mathbf{B}$ via ordinary least squares (OLS) and $\mathbf{R}$ as the empirical covariance of the OLS residuals; furthermore, we estimate $\mathbf{V}$ as the empirical covariance matrix of the common factors.

\section{Macroeconomic Factors Data}
\label{sec:appe}
\setcounter{table}{0}
\setcounter{figure}{0}
\renewcommand{\thetable}{G.\arabic{table}}
\renewcommand{\thefigure}{G.\arabic{figure}}

We use the macroeconomic factors from FRED-MD by \citet{mccracken2016fred} plus the eight factors found by \citet{goyal2008comprehensive}. We use the version of this data that was curated in \citet{chen2020internet} and the following transformations were applied: 

\begin{table}[htbp]
\centering
\caption{Transformations (tCodes) for Macroeconomic Variables}
\begin{tabular}{|c|l|}
\hline
\textbf{tCode} & \textbf{Transformation Description} \\ \hline
1 & No transformation \\ \hline
2 & \(\Delta x_t\) \\ \hline
3 & \(\Delta^2 x_t\) \\ \hline
4 & \(\log(x_t)\) \\ \hline
5 & \(\Delta \log(x_t)\) \\ \hline
6 & \(\Delta^2 \log(x_t)\) \\ \hline
7 & \(\Delta(x_t/x_{t-1} - 1.0)\) \\ \hline
\end{tabular}
\label{tab:tcode_transformations}
\end{table}

\small{
\begin{longtable}{|l|l|l|l|}
\caption{Macroeconomic Variables (based on table IA.XVI of \citet{chen2020internet})}
\\
\hline
\textbf{Variable Name} & \textbf{Description} & \textbf{Source} & \textbf{tCode} \\ \hline
\endfirsthead

\hline
\textbf{Variable Name} & \textbf{Description} & \textbf{Source} & \textbf{tCode} \\ \hline
\endhead

\hline
\endfoot

\hline
\endlastfoot

RPI & Real Personal Income & Fred-MD & 5 \\ \hline
W875RX1 & Real personal income ex transfer receipts & Fred-MD & 5 \\ \hline
DPCERA3M086SBEA & Real personal consumption expenditures & Fred-MD & 5 \\ \hline
CMRMTSPLx & Real Manu. and Trade Industries Sales & Fred-MD & 5 \\ \hline
RETAILx & Retail and Food Services Sales & Fred-MD & 5 \\ \hline
INDPRO & IP Index & Fred-MD & 5 \\ \hline
IPFPNSS & IP: Final Products and Nonindustrial Supplies & Fred-MD & 5 \\ \hline
IPFINAL & IP: Final Products (Market Group) & Fred-MD & 5 \\ \hline
IPCONGD & IP: Consumer Goods & Fred-MD & 5 \\ \hline
IPDCONGD & IP: Durable Consumer Goods & Fred-MD & 5 \\ \hline
IPNCONGD & IP: Nondurable Consumer Goods & Fred-MD & 5 \\ \hline
IPBUSEQ & IP: Business Equipment & Fred-MD & 5 \\ \hline
IPMAT & IP: Materials & Fred-MD & 5 \\ \hline
IPDMAT & IP: Durable Materials & Fred-MD & 5 \\ \hline
IPNMAT & IP: Nondurable Materials & Fred-MD & 5 \\ \hline
IPMANSICS & IP: Manufacturing (SIC) & Fred-MD & 5 \\ \hline
IPB51222S & IP: Residential Utilities & Fred-MD & 5 \\ \hline
IPFUELS & IP: Fuels & Fred-MD & 5 \\ \hline
CUMFNS & Capacity Utilization: Manufacturing & Fred-MD & 2 \\ \hline
HWI & Help-Wanted Index for United States & Fred-MD & 2 \\ \hline
HWIURATIO & Ratio of Help Wanted/No. Unemployed & Fred-MD & 2 \\ \hline
CLF16OV & Civilian Labor Force & Fred-MD & 5 \\ \hline
CE16OV & Civilian Employment & Fred-MD & 5 \\ \hline
UNRATE & Civilian Unemployment Rate & Fred-MD & 2 \\ \hline
UEMPMEAN & Average Duration of Unemployment (Weeks) & Fred-MD & 2 \\ \hline
UEMPLT5 & Civilians Unemployed - Less Than 5 Weeks & Fred-MD & 5 \\ \hline
UEMP5TO14 & Civilians Unemployed for 5-14 Weeks & Fred-MD & 5 \\ \hline
UEMP15OV & Civilians Unemployed - 15 Weeks \& Over & Fred-MD & 5 \\ \hline
UEMP15T26 & Civilians Unemployed for 15-26 Weeks & Fred-MD & 5 \\ \hline
UEMP27OV & Civilians Unemployed for 27 Weeks and Over & Fred-MD & 5 \\ \hline
CLAIMSx & Initial Claims & Fred-MD & 5 \\ \hline
PAYEMS & All Employees: Total nonfarm & Fred-MD & 5 \\ \hline
USGOOD & All Employees: Goods-Producing Industries & Fred-MD & 5 \\ \hline
CES1021000001 & All Employees: Mining and Logging: Mining & Fred-MD & 5 \\ \hline
USCONS & All Employees: Construction & Fred-MD & 5 \\ \hline
MANEMP & All Employees: Manufacturing & Fred-MD & 5 \\ \hline
DMANEMP & All Employees: Durable goods & Fred-MD & 5 \\ \hline
NDMANEMP & All Employees: Nondurable goods & Fred-MD & 5 \\ \hline
SRVPRD & All Employees: Service-Providing Industries & Fred-MD & 5 \\ \hline
USTPU & All Employees: Trade, Transportation \& Utilities & Fred-MD & 5 \\ \hline
USWTRADE & All Employees: Wholesale Trade & Fred-MD & 5 \\ \hline
USTRADE & All Employees: Retail Trade & Fred-MD & 5 \\ \hline
USFIRE & All Employees: Financial Activities & Fred-MD & 5 \\ \hline
USGOVT & All Employees: Government & Fred-MD & 5 \\ \hline
CES0600000007 & Avg Weekly Hours : Goods-Producing & Fred-MD & 1 \\ \hline
AWOTMAN & Avg Weekly Overtime Hours : Manufacturing & Fred-MD & 2 \\ \hline
AWHMAN & Avg Weekly Hours : Manufacturing & Fred-MD & 1 \\ \hline
HOUST & Housing Starts: Total New Privately Owned & Fred-MD & 4 \\ \hline
HOUSTNE & Housing Starts, Northeast & Fred-MD & 4 \\ \hline
HOUSTMW & Housing Starts, Midwest & Fred-MD & 4 \\ \hline
HOUSTS & Housing Starts, South & Fred-MD & 4 \\ \hline
HOUSTW & Housing Starts, West & Fred-MD & 4 \\ \hline
PERMIT & New Private Housing Permits (SAAR) & Fred-MD & 4 \\ \hline
PERMITNE & New Private Housing Permits, Northeast (SAAR) & Fred-MD & 4 \\ \hline
PERMITMW & New Private Housing Permits, Midwest (SAAR) & Fred-MD & 4 \\ \hline
PERMITS & New Private Housing Permits, South (SAAR) & Fred-MD & 4 \\ \hline
PERMITW & New Private Housing Permits, West (SAAR) & Fred-MD & 4 \\ \hline
AMDMNOx & New Orders for Durable Goods & Fred-MD & 5 \\ \hline
AMDMUOx & Unfilled Orders for Durable Goods & Fred-MD & 5 \\ \hline
BUSINVx & Total Business Inventories & Fred-MD & 5 \\ \hline
ISRATIOx & Total Business: Inventories to Sales Ratio & Fred-MD & 2 \\ \hline
M1SL & M1 Money Stock & Fred-MD & 6 \\ \hline
M2SL & M2 Money Stock & Fred-MD & 6 \\ \hline
M2REAL & Real M2 Money Stock & Fred-MD & 5 \\ \hline
AMBSL & St. Louis Adjusted Monetary Base & Fred-MD & 6 \\ \hline
TOTRESNS & Total Reserves of Depository Institutions & Fred-MD & 6 \\ \hline
NONBORRES & Reserves Of Depository Institutions & Fred-MD & 7 \\ \hline
BUSLOANS & Commercial and Industrial Loans & Fred-MD & 6 \\ \hline
REALLN & Real Estate Loans at All Commercial Banks & Fred-MD & 6 \\ \hline
NONREVSL & Total Nonrevolving Credit & Fred-MD & 6 \\ \hline
CONSPI & Nonrevolving consumer credit to Personal Income & Fred-MD & 2 \\ \hline
S\&P 500 & S\&P’s Common Stock Price Index: Composite & Fred-MD & 5 \\ \hline
S\&P: indust & S\&P’s Common Stock Price Index: Industrials & Fred-MD & 5 \\ \hline
S\&P div yield & S\&P’s Composite Common Stock: Dividend Yield & Fred-MD & 2 \\ \hline
S\&P PE ratio & S\&P’s Composite Common Stock: Price-Earnings Ratio & Fred-MD & 5 \\ \hline
FEDFUNDS & Effective Federal Funds Rate & Fred-MD & 2 \\ \hline
CP3Mx & 3-Month AA Financial Commercial Paper Rate & Fred-MD & 2 \\ \hline
TB3MS & 3-Month Treasury Bill & Fred-MD & 2 \\ \hline
TB6MS & 6-Month Treasury Bill & Fred-MD & 2 \\ \hline
GS1 & 1-Year Treasury Rate & Fred-MD & 2 \\ \hline
GS5 & 5-Year Treasury Rate & Fred-MD & 2 \\ \hline
GS10 & 10-Year Treasury Rate & Fred-MD & 2 \\ \hline
AAA & Moody’s Seasoned Aaa Corporate Bond Yield & Fred-MD & 2 \\ \hline
BAA & Moody’s Seasoned Baa Corporate Bond Yield & Fred-MD & 2 \\ \hline
COMPAPFFx & 3-Month Commercial Paper Minus FEDFUNDS & Fred-MD & 1 \\ \hline
TB3SMFFM & 3-Month Treasury C Minus FEDFUNDS & Fred-MD & 1 \\ \hline
TB6SMFFM & 6-Month Treasury C Minus FEDFUNDS & Fred-MD & 1 \\ \hline
T1YFFM & 1-Year Treasury C Minus FEDFUNDS & Fred-MD & 1 \\ \hline
T5YFFM & 5-Year Treasury C Minus FEDFUNDS & Fred-MD & 1 \\ \hline
T10YFFM & 10-Year Treasury C Minus FEDFUNDS & Fred-MD & 1 \\ \hline
AAAFFM & Moody’s Aaa Corporate Bond Minus FEDFUNDS & Fred-MD & 1 \\ \hline
BAAFFM & Moody’s Baa Corporate Bond Minus FEDFUNDS & Fred-MD & 1 \\ \hline
EXSZUSx & Switzerland / U.S. Foreign Exchange Rate & Fred-MD & 5 \\ \hline
EXJPUSx & Japan / U.S. Foreign Exchange Rate & Fred-MD & 5 \\ \hline
EXUSUKx & U.S. / U.K. Foreign Exchange Rate & Fred-MD & 5 \\ \hline
EXCAUSx & Canada / U.S. Foreign Exchange Rate & Fred-MD & 5 \\ \hline
WPSFD49207 & PPI: Finished Goods & Fred-MD & 6 \\ \hline
WPSFD49502 & PPI: Finished Consumer Goods & Fred-MD & 6 \\ \hline
WPSID61 & PPI: Intermediate Materials & Fred-MD & 6 \\ \hline
WPSID62 & PPI: Crude Materials & Fred-MD & 6 \\ \hline
OILPRICEx & Crude Oil, spliced WTI and Cushing & Fred-MD & 6 \\ \hline
PPICMM & PPI: Metals and metal products & Fred-MD & 6 \\ \hline
CPIAUCSL & CPI : All Items & Fred-MD & 6 \\ \hline
CPIAPPSL & CPI : Apparel & Fred-MD & 6 \\ \hline
CPITRNSL & CPI : Transportation & Fred-MD & 6 \\ \hline
CPIMEDSL & CPI : Medical Care & Fred-MD & 6 \\ \hline
CUSR0000SAC & CPI : Commodities & Fred-MD & 6 \\ \hline
CUSR0000SAD & CPI : Durables & Fred-MD & 6 \\ \hline
CUSR0000SAS & CPI : Services & Fred-MD & 6 \\ \hline
CPIULFSL & CPI : All Items Less Food & Fred-MD & 6 \\ \hline
CUSR0000SA0L2 & CPI : All items less shelter & Fred-MD & 6 \\ \hline
CUSR0000SA0L5 & CPI : All items less medical care & Fred-MD & 6 \\ \hline
PCEPI & Personal Cons. Expend.: Chain Index & Fred-MD & 6 \\ \hline
DDURRG3M086SBEA & Personal Cons. Exp: Durable goods & Fred-MD & 6 \\ \hline
DNDGRG3M086SBEA & Personal Cons. Exp: Nondurable goods & Fred-MD & 6 \\ \hline
DSERRG3M086SBEA & Personal Cons. Exp: Services & Fred-MD & 6 \\ \hline
CES0600000008 & Avg Hourly Earnings : Goods-Producing & Fred-MD & 6 \\ \hline
CES2000000008 & Avg Hourly Earnings : Construction & Fred-MD & 6 \\ \hline
CES3000000008 & Avg Hourly Earnings : Manufacturing & Fred-MD & 6 \\ \hline
MZMSL & MZM Money Stock & Fred-MD & 6 \\ \hline
DTCOLNVHFNM & Consumer Motor Vehicle Loans Outstanding & Fred-MD & 6 \\ \hline
DTCTHFNM & Total Consumer Loans and Leases Outstanding & Fred-MD & 6 \\ \hline
INVEST & Securities in Bank Credit at All Commercial Banks & Fred-MD & 6 \\ \hline
VXOCLSx & CBOE S\&P 100 Volatility Index: VXO & Fred-MD & 1 \\ \hline

dp & Dividend-price ratio & W \& G & 2 \\ \hline
ep & Earnings-price ratio & W \& G & 2 \\ \hline
bm & Book-to-market ratio & W \& G  & 5 \\ \hline
ntis & Net equity expansion & W \& G & 2 \\ \hline
tbl & Treasury-bill rate & W \& G  & 2 \\ \hline
tms & Term spread & W \& G  & 1 \\ \hline
dfy & Default spread & W \& G  & 2 \\ \hline
svar & Stock variance & W \& G & 5 \\ \hline
\end{longtable}
}

\end{document}